\documentclass[sigconf]{acmart}
\usepackage{amsthm,color}
\usepackage{multirow, tabularx}
\usepackage{subcaption}
\usepackage{graphicx}

\newcommand{\nettt}{\textsc{NeT$^3$}}

\AtBeginDocument{%
  \providecommand\BibTeX{{%
    \normalfont B\kern-0.5em{\scshape i\kern-0.25em b}\kern-0.8em\TeX}}}

\setcopyright{iw3c2w3}
\copyrightyear{2021}
\acmYear{2021}
\acmDOI{10.1145/XXXXXXX.XXXXXXX}

\acmConference[WWW '21]{Proceedings of the Web Conference 2021}{April 19--23, 2021}{Ljubljana, Slovenia}
\acmBooktitle{Proceedings of the Web Conference 2021 (WWW '21), April 19--23, 2021,
Ljubljana, Slovenia}
\acmPrice{}
\acmISBN{978-1-4503-8312-7/21/04}

\copyrightyear{2021}
\acmYear{2021} 
\setcopyright{iw3c2w3}
\acmConference[WWW '21]{Proceedings of the Web Conference 2021}{April 19--23, 2021}{Ljubljana, Slovenia} 
\acmBooktitle{Proceedings of the Web Conference 2021 (WWW '21), April 19--23, 2021, Ljubljana, Slovenia}
\acmPrice{}
\acmDOI{10.1145/3442381.3449969}
\acmISBN{978-1-4503-8312-7/21/04}
\begin{document}

\title{Network of Tensor Time Series}


\author{Baoyu Jing, Hanghang Tong}
% \email{}
\affiliation{%
  \{baoyuj2, htong\}@illinois.edu\\
  \institution{University of Illinois at Urbana-Champaign}
  \state{IL}
  \country{USA}
}

% \author{Hanghang Tong}
% \email{htong@illinois.edu}
% \affiliation{%
%   \institution{University of Illinois at Urana-Champaign}
%   \state{IL}

\author{Yada Zhu}
\email{yzhu@us.ibm.com}
\affiliation{%
  \institution{IBM Research}
  \state{NY}
  \country{USA}
}


\begin{abstract}
Co-evolving time series appears in a multitude of applications such as environmental monitoring, financial analysis, and smart transportation.
This paper aims to address the following challenges, including
(C1) how to incorporate \textit{explicit relationship networks} of the time series;
(C2) how to model the \textit{implicit relationship} of the temporal dynamics.
We propose a novel model called Network of Tensor Time Series (\nettt), which is comprised of two modules, including Tensor Graph Convolutional Network (TGCN) and Tensor Recurrent Neural Network (TRNN).
TGCN tackles the first challenge by generalizing Graph Convolutional Network (GCN) for flat graphs to tensor graphs, which captures the synergy between multiple graphs associated with the tensors.
TRNN leverages tensor decomposition to model the implicit relationships among co-evolving time series.
The experimental results on five real-world datasets demonstrate the efficacy of the proposed method.
\end{abstract}




\maketitle

\section{Introduction}
% \begin{figure}[t]
%   \centering
%   \begin{minipage}{\linewidth}
%     \centering
%     \subcaptionbox{A tensor time series consists of three modes: location, data type and time.\label{fig:example_tensor_time_series}}
%     {\includegraphics[width=.48\linewidth]{figs/example_a.png}}\quad
%     \subcaptionbox{A temporal snapshot. Rows and columns present locations and data types, which are constrained by their networks.\label{fig:example_snapshot}}
%     {\includegraphics[width=.48\linewidth]{figs/example_d.png}}\quad
%     \subcaptionbox{A slice along one data type: co-evolving time series of one data type at different locations.\label{fig:example_location}}
%     {\includegraphics[width=.48\linewidth]{figs/example_b.png}}\quad
%     \subcaptionbox{A slice along one location: co-evolving time series of different data types at the same location.\label{fig:example_type}}
%     {\includegraphics[width=.48\linewidth]{figs/example_c.png}}\quad
%     \caption{An exemplary tensor time series and three slices along different dimensions. (Best viewed in color.)}\label{fig:example}
%   \end{minipage}
% \end{figure}

\begin{figure}[t!]
    \centering
    \begin{subfigure}[b]{.22\textwidth}
        \includegraphics[width=\textwidth]{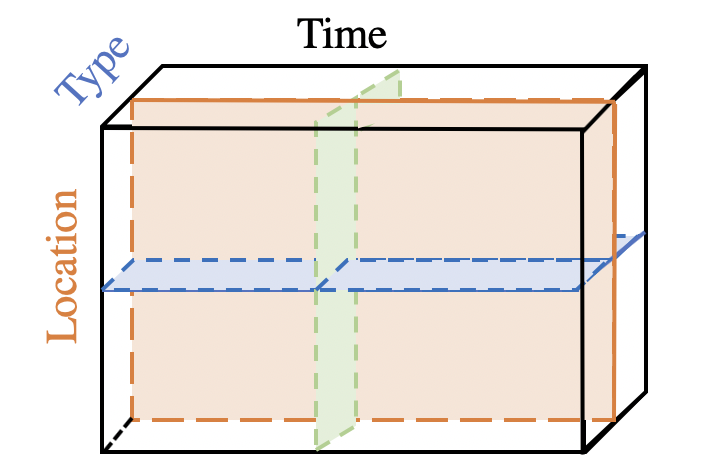}
        \caption{An example of tensor time series, which is comprised of three modes: location, data type and time.}
        \label{fig:example_tensor_time_series}
    \end{subfigure}
    \quad
    \begin{subfigure}[b]{.22\textwidth}
        \includegraphics[width=\textwidth]{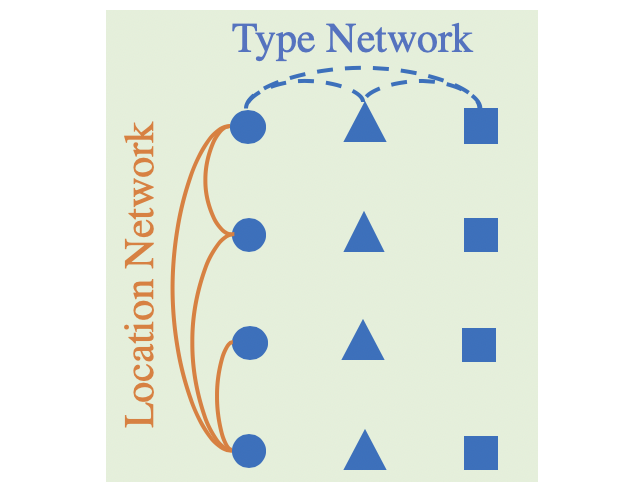}
        \caption{A temporal snapshot. Rows and columns present locations and data types, which are constrained by their networks.}
        \label{fig:example_snapshot}
    \end{subfigure}
    \begin{subfigure}[b]{.22\textwidth}
        \centering
        \includegraphics[width=\textwidth]{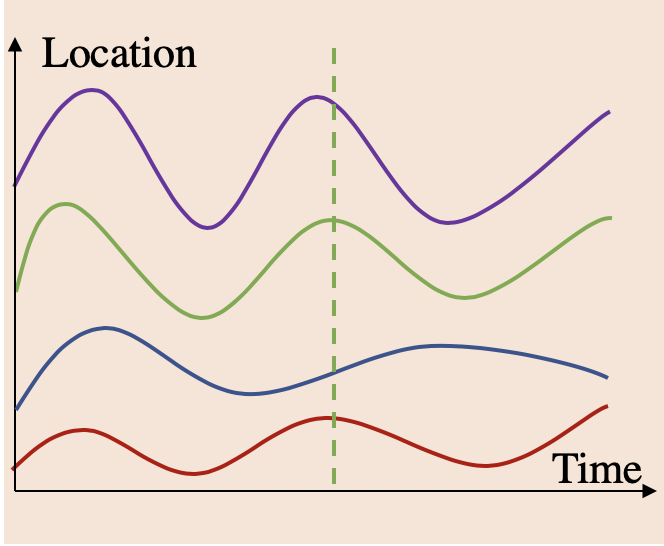}
        \caption{A slice along one data type: co-evolving time series of the same data type at different locations.}
        \label{fig:example_location}
    \end{subfigure}
    \quad
    \begin{subfigure}[b]{.22\textwidth}\label{fig:visualize_eeij}
        \centering
        \includegraphics[width=\textwidth]{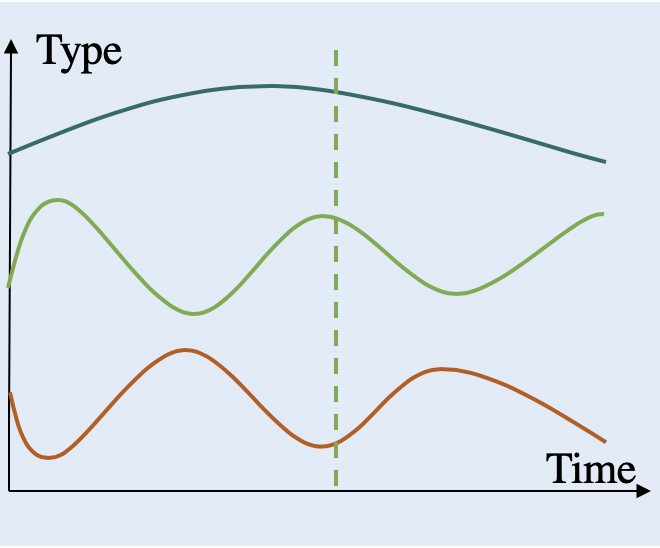}
        \caption{A slice along one location: co-evolving time series of different data types at the same location.}
        \label{fig:example_type}
    \end{subfigure}
    
    \caption{An exemplary tensor time series and three slices along different dimensions. (Best viewed in color.)}
    \label{fig:example}
\end{figure}

Co-evolving time series naturally arises in numerous applications, ranging from environmental monitoring \cite{banzon2016long, srivastava2018comparative}, financial analysis \cite{tsay2014financial} to smart transportation \cite{li2017diffusion, yu2017spatio, li2019predicting}.
As shown in Figure \ref{fig:example_tensor_time_series} and \ref{fig:example_snapshot}, each temporal snapshot of the co-evolving time series naturally forms a multi-dimensional array, i.e., a \textit{multi-mode tensor} \cite{rogers2013multilinear}.
For example, the spatial-temporal monitoring data of atmosphere is a time series of an $N_1\times N_2\times N_3 \times N_4$ tensor, where $N_1$, $N_2$, $N_3$ and $N_4$ denote latitude, longitude, elevation and air conditions respectively (e.g. temperature, pressure and oxygen concentration).
Companies' financial data is a time series of an $N_1 \times N_2 \times N_3$ tensor, where $N_1$, $N_2$ and $N_3$ denote the companies, the types of financial data (e.g. revenue, expenditure) and the statistics of them respectively.
% Company revenues is a time series of a $N_1 \times N_3$ tensor, where $N_1$ and $N_3$ are companies and different statistics of revenues.
% Should other financial data (e.g. expenditure) be provided, the two-mode tensor above can be extended to a three-mode tensor $N_1\times N_2 \times N_3$.
% where $N_2$ represents different types of financial data.
Nonetheless, the vast majority of the recent deep learning methods for co-evolving time series \cite{li2017diffusion, yu2017spatio, li2019predicting, yan2018spatial, liang2018geoman} have almost exclusively focused on a single mode. 
% \by{The spatial-temporal gcn, it considers the location (or spatial) as one mode, and it doesn't consider longitude, latitude etc.}
% while ignored other modes.

Data points within a tensor are usually related to each other, and different modes are associated with different relationships (Figure \ref{fig:example_snapshot}).
Within the above example of environmental monitoring, along geospatial modes ($N_1$, $N_2$ and $N_3$), we could know the (latitudinal, longitudinal and elevational) location relationship between two data points.
In addition, different data types ($N_4$) are also related with each other.  
As governed by Gay-Lussac's law \cite{barnett1941brief}, given fixed mass and volume, the pressure of a gas is proportional to the Kelvin temperature.
% Revenues of ventilator companies grow together with face mask manufactures with the outbreak of the coronavirus, as these are healthcare related companies.
These relationships can be \textit{explicitly} modeled by \textit{networks} or \textit{graphs} \cite{chakrabarti2006graph, akoglu2015graph}.
Compared with the rich machinery of deep graph convolutional methods for flat graphs \cite{kipf2016semi, defferrard2016convolutional}, multiple graphs associated with a tensor (referred to as \textit{tensor graphs} in this paper) are less studied. 
To fill this gap, we propose a novel Tensor Graph Convolution Network (TGCN) which extends Graph Convolutional Network (GCN) \cite{kipf2016semi}
to tensor graphs based on multi-dimensional convolution.

% Given a time series, the two primary problems of interests are missing value recovery and future value prediction.
% Missing values are commonly observed in time series because of the malfunctioning or connectivity issues of sensors, and successful recovery of the values will provide valuable historical information for studying the time series \cite{kong2013data}.
% A good prediction of future values will prepare us in advance for the crisis in climate \cite{mudelsee2013climate}, economics \cite{lopez2013effect} and epidemic such as COVID-19 \cite{peng2020epidemic}.
% To address these problems, it is indispensable to accurately model the complex \textit{temporal dynamics} behind the time series.
Another key challenge for modeling the temporal dynamics behind co-evolving time series is how to capture the \textit{implicit relationship} of different time series.
As shown in Figure \ref{fig:example_location}, the temporal patterns of time series with the same data type (e.g. temperature) are similar.
The relationship of the co-evolving temperature time series can be partially captured by the location network, e.g., two neighboring locations often have similar temporal dynamics.
However, the temperature time series from two locations far apart could also share similar patterns.
Most of the existing studies either use the same temporal model for all time series \cite{li2017diffusion, yu2017spatio, li2019predicting, yan2018spatial}, or use separate Recurrent Neural Networks (RNN) \cite{srivastava2018comparative, zhou2017recover} for different time series. 
Nonetheless, none of them offers a principled way to model the implicit relationship.
To tackle with this challenge, we propose a novel Tensor Recurrent Neural Network (TRNN) based on Multi-Linear Dynamic System (MLDS) \cite{rogers2013multilinear} and Tucker decomposition, which helps reduce the number of model parameters.

Our main contributions are summarized as follows:
\begin{itemize}
    \item We introduce a novel graph convolution for tensor graphs and present a novel TGCN that generalizes GCN \cite{kipf2016semi}. The new architecture can capture the synergy among different graphs by simultaneously performing convolution on them.
    \item We introduce a novel TRNN based on MLDS \cite{rogers2013multilinear} for efficiently modeling the implicit relationship between complex temporal dynamics of tensor time series. 
    \item We present comprehensive evaluations for the proposed methods on a variety of real-world datasets to demonstrate the effectiveness of the proposed method.
\end{itemize}

%\hh{let us add a small paragraph about the organization of the paper -- we have the space and it will make the paper more organized.}
%\by{updated}

The rest of the paper is organized as follows. 
In Section \ref{sec:preliminary}, we briefly introduce relevant definitions about graph convolution and tensor algebra, and formally introduce the definition of network of tensor time series.
In Section \ref{sec:methods}, we present and analyze the proposed TGCN and TRNN.
The experimental results are presented in Section \ref{sec:experiments}.
Related works and conclusion are presented in Section \ref{sec:related_work} and Section \ref{sec:conclusion} respectively.

\section{Preliminaries}\label{sec:preliminary}
In this section, we formally define network of tensor time series (Subsection~\ref{sub:multift}), after we review the preliminaries, including graph convolution on flat graphs (Subsection~\ref{sub:plaingcn}), tensor algebra (Subsection~\ref{sub:tensor}), and multi-dimensional Fourier transformation (Subsection~\ref{sub:multift}) respectively.
We introduce the definitions of the problems in Section \ref{sec:problem_definition}.

\subsection{Graph Convolution on Flat Graphs}\label{sub:plaingcn}

% Spectral graph convolution on a flat graph is defined by analogizing it to the one-dimensional convolution \cite{kipf2016semi, defferrard2016convolutional}.
% Its Chebyshev approximation is thus defined by a linear operation between $\theta_pT_p(\tilde{\mathbf{L}})$ and $\mathbf{x}$, as shown in Definition \ref{def:chebynet}.
% 
% \begin{definition}[Chebyshev Approximation for Flat Graph Convolution]\label{def:chebynet}
% Given an input graph signal $\mathbf{x}\in\mathbb{R}^N$ and its adjacency matrix $\mathbf{A}\in\mathbb{R}^{N\times N}$, where $N$ is the number of nodes, the Chebyshev approximation for graph convolution on a flat graph is defined by \cite{kipf2016semi, defferrard2016convolutional}:
% \begin{equation}
%     \mathbf{g}_\theta \star \mathbf{x} = \sum_{p=0}^P\theta_pT_p(\Tilde{\mathbf{L}})\mathbf{x}
% \end{equation}
% where $\mathbf{L}=\mathbf{I} - \mathbf{D}^{-\frac{1}{2}}\mathbf{A}\mathbf{D}^{-\frac{1}{2}}$ is the graph Laplacian, $\mathbf{D}$ is the degree matrix of $\mathbf{A}$, and $\Tilde{\mathbf{L}} = \frac{2}{\lambda_{max}}\mathbf{L} - \mathbf{I}$ is the normalized graph Laplacian,
% $\lambda_{max}$ is maximum eigenvalue of $\mathbf{L}$;
% $T_p(x)$ is Chebyshev polynomials defined by $T_p(x) = 2xT_{p-1}(x) - T_{p-2}(x)$ with $T_0(x) = 1$ and $T_1(x) = x$, and $p$ denotes the order of polynomials;
% $\mathbf{g}_\theta$ and $\theta_p$ denote the filter vector and the parameter respectively.
% \end{definition} 

Analogous to the one-dimensional Discrete Fourier Transform (Definition \ref{def:dft}), the graph Fourier transform is given by Definition \ref{def:gft}.
Then the spectral graph convolution (Definition \ref{def:graph_conv}) is defined based on one-dimensional convolution and the convolution theorem.
The free parameter of the convolution filter is further replaced by Chebyshev polynomials and thus we have Chebyshev approximation for graph convolution (Definition \ref{def:cheby_graph}).

%\hh{let us move the definition of flat graph to 2.1 and move tensor graph to either 2.2 or 2.5. they are better fit there. plus, right now, we have too many subsections for 'preliminaries'}

% Given an input graph signal $\mathbf{x}\in\mathbb{R}^N$ and its adjacency matrix $\mathbf{A}\in\mathbb{R}^{N\times N}$, where $N$ is the number of nodes, its graph Laplacian matrix $\mathbf{L}=\mathbf{I} - \mathbf{D}^{-\frac{1}{2}}\mathbf{A}\mathbf{D}^{-\frac{1}{2}}$, where $\mathbf{I}\in\mathbb{R}^{N\times N}$ and $\mathbf{D}\in\mathbb{R}^{N\times N}$ denote identity matrix and degree matrix respectively.
% Due to the symmetry of $\mathbf{L}$, it can be decomposed into $\mathbf{L} = \mathbf{\Phi}\mathbf{\Lambda}\mathbf{\Phi}^T$, where $\mathbf{\Lambda}=\text{diag}(\lambda_1, \cdots, \lambda_N)$ is a diagonal matrix of eigenvalues  and $\mathbf{\Phi}\in\mathbb{R}^{N\times N}$ is the matrix of orthonormal eigenvectors.
% Hence, the Chebyshev approximation for spectral graph convolution can be given by Definition \ref{def:cheby_graph}.

\begin{definition}[Flat Graph]\label{def:flat_graph}
A flat graph contains a one-dimentional graph signal $\mathbf{x}\in\mathbb{R}^{N}$ and an adjacency matrix $\mathbf{A}\in\mathbb{R}^{N\times N}$.

\end{definition}

\begin{definition}[Discrete Fourier Transform]\label{def:dft}
Given an one dimensional signal $\mathbf{x}\in\mathbb{R}^N$, where $N$ is the length of the sequence, its Fourier transform is defined by:
\begin{equation}
    \Tilde{\mathbf{x}}[n] = \sum_{k=1}^N\mathbf{x}[k]e^{-\frac{i2\pi}{N}kn}
\end{equation}
where $\mathbf{x}[k]$ is the $k$-th element of $\mathbf{x}$ and $\Tilde{\mathbf{x}}[n]$ is the $n$-th element of the transformed vector $\Tilde{\mathbf{x}}$. 
The above definition can be rewritten as:
\begin{equation}
    \Tilde{\mathbf{x}} = \mathbf{F}\mathbf{x}
\end{equation}
where $\mathbf{F}\in\mathbb{R}^{N\times N}$ is the filter matrix and $\mathbf{F}[n,k] = e^{-\frac{i2\pi}{N}kn}$.
\end{definition}

\begin{definition}[Graph Fourier Transform \cite{bruna2013spectral}]\label{def:gft}
Given a graph signal $\mathbf{x}\in\mathbb{R}^{N}$, along with its adjacency matrix $\mathbf{A}\in\mathbb{R}^{N\times N}$, where $N$ is the number of nodes, the graph Fourier transform is defined by:
\begin{equation}
    \Tilde{\mathbf{x}} = \mathbf{\Phi}^T\mathbf{x}
\end{equation}
where $\mathbf{\Phi}$ is the eigenvector matrix of the graph Laplacian matrix $\mathbf{L}=\mathbf{I} - \mathbf{D}^{-\frac{1}{2}}\mathbf{A}\mathbf{D}^{-\frac{1}{2}}=\mathbf{\Phi}\mathbf{\Lambda}\mathbf{\Phi}^T$, $\mathbf{I}\in\mathbb{R}^{N\times N}$, $\mathbf{D}\in\mathbb{R}^{N\times N}$ denote the identity matrix and the degree matrix, and $\mathbf{\Lambda}$ is a diagonal matrix whose diagonal elements are eigenvalues. %\hh{$\Lambda$ is not defined}
%\by{updated}
\end{definition}

% \begin{definition}[Discrete Convolution]\label{def:conv}
% Given two vectors $\mathbf{x}\in\mathbb{R}^{N}$ and $\mathbf{y}\in\mathbb{R}^{N}$, the discrete convolution is defined as:
% \begin{equation}
%     \mathbf{y}\star\mathbf{x} = 
% \end{equation}
% \end{definition}

\begin{definition}[Spectral Graph Convolution \cite{bruna2013spectral}]\label{def:graph_conv}
Given a signal $\mathbf{x}\in\mathbb{R}^N$ and a filter $\mathbf{g}\in\mathbb{R}^N$, the spectral graph convolution is defined in the Fourier domain according to the convolution theorem:
\begin{align}
    \mathbf{\Phi}^T(\mathbf{g}\star\mathbf{x}) &= (\mathbf{\Phi}^T\mathbf{g})\odot(\mathbf{\Phi}^T\mathbf{x})\\
    \mathbf{g}\star\mathbf{x} &= \mathbf{\Phi}(\mathbf{\Phi}^T\mathbf{g})\odot(\mathbf{\Phi}^T\mathbf{x}) = \mathbf{\Phi}\text{diag}(\Tilde{\mathbf{g}}) \mathbf{\Phi}^T\mathbf{x}\label{eq:spectral_conv}
\end{align}
where $\star$ and $\odot$ denote convolution operation and Hadamard product; the second equation holds due to the orthonormality.
\end{definition}

\begin{definition}[Chebyshev Approximation for Spectral Graph Convolution \cite{defferrard2016convolutional}]\label{def:cheby_graph}
Given an input graph signal $\mathbf{x}\in\mathbb{R}^N$ and its adjacency matrix $\mathbf{A}\in\mathbb{R}^{N\times N}$, the Chebyshev approximation for graph convolution on a flat graph is given by \cite{kipf2016semi, defferrard2016convolutional}:
\begin{equation}
\mathbf{g}_\theta \star \mathbf{x} =\mathbf{\Phi} (\sum_{p=0}^P\theta_pT_p(\Tilde{\mathbf{\Lambda}}))\mathbf{\Phi}^T\mathbf{x} = \sum_{p=0}^P\theta_pT_p(\Tilde{\mathbf{L}})\mathbf{x}
\end{equation}
where $\Tilde{\mathbf{\Lambda}} = \frac{2}{\lambda_{max}}\mathbf{\Lambda} - \mathbf{I}$ is the normalized eigenvalues,
$\lambda_{max}$ is maximum eigenvalue of the matrix $\mathbf{\Lambda}$;
$\Tilde{\mathbf{L}} = \frac{2}{\lambda_{max}}\mathbf{L} - \mathbf{I}$;
$T_p(x)$ is Chebyshev polynomials defined by $T_p(x) = 2xT_{p-1}(x) - T_{p-2}(x)$ with $T_0(x) = 1$ and $T_1(x) = x$, and $p$ denotes the order of polynomials;
$\mathbf{g}_\theta$ and $\theta_p$ denote the filter vector and the parameter respectively.
\end{definition} 

\subsection{Tensor Algebra}\label{sub:tensor}
% \begin{definition}[Matricization]\label{def:matricization}
% Matricization ``flattens'' a tensor into a matrix.
% Given a tensor $\mathcal{X}\in\mathbb{R}^{N_1\times\cdots\times N_M \times N_1'\times\cdots\times N_M'}$, its matricization 
% % \hh{use textrm when representing a function name (mat). otherwise, it indicates it is a variable name}
% $\textrm{mat}(\mathcal{X})\in\mathbb{R}^{N_1\cdots N_M\times N_1'\cdots N_M'}$ is defined as:
% \begin{equation}
% \textrm{mat}(\mathcal{X})[i,j] = \mathcal{X}[n_1, \cdots, n_M, n_1',\cdots, n_M']
% \end{equation}
% $i = 1+\sum_{m=1}^M\prod_{l=1}^{m-1}N_l(n_m-1)$, $j = 1+\sum_{m=1}^M\prod_{l=1}^{m-1}N_l'(n_m-1)$.
% \end{definition}

\begin{definition}[Mode-m Product] \label{def:mode_product}
The mode-m product generalizes matrix-matrix product to tensor-matrix product.
Given a matrix $\mathbf{U}\in\mathbb{R}^{N_m\times N'}$, and a tensor $\mathcal{X}\in\mathbb{R}^{N_1\times\cdots N_{m-1}\times N_{m}\times N_{m+1}\cdots\times N_M}$, then $\mathcal{X}\times_{m}\mathbf{U}\in\mathbb{R}^{N_1\times\cdots N_{m-1}\times N'\times N_{m+1} \cdots \times N_M}$ is its mode-m product.
% the mode-$m$ product of $\mathcal{X}$ and $\mathbf{U}$ is denoted by $\mathcal{X}\times_{m}\mathbf{U}\in\mathbb{R}^{N_1\times\cdots N_{m-1}\times N'\times N_{m+1} \cdots \times N_M}$.
Its element $[n_1, \cdots, n_{m-1}, n', n_{m+1}, \cdots, n_M]$ is defined as:

\begin{equation}
    \begin{split}
        & (\mathcal{X} \times_{m} \mathbf{U})[n_1, \cdots, n_{m-1}, n', n_{m+1}, \cdots, n_M]\\
        =& \sum_{n_m=1}^{N_m}\mathcal{X}[n_1, \cdots, n_{m-1}, n_m, n_{m+1}, \cdots, n_M]\mathbf{U}[n_m, n']
    \end{split}
\end{equation}

\end{definition}

% \begin{definition}[Contracted Product]\label{def:multi_linear}
% The contracted product (or multi-linear product) is a generalization of matrix multiplication for tensors.
% Given two tensors $\mathcal{Z}\in\mathbb{R}^{N1'\times\cdots\times N_M' \times N_1\times\cdots\times N_M}$ and   $\mathcal{U}\in\mathbb{R}^{N_1\times \cdots \times N_M \times N_1' \times\cdots\times N_M'}$, 
% the element $[n_1, \dots, n_M]$ of their multi-linear product $\mathcal{U}\circledast\mathcal{Z}$ is:
% \begin{equation}
%     \begin{split}
%     (\mathcal{U}\circledast\mathcal{Z})[n_1, \dots, n_M] = & \sum_{n_1', \dots, n_M'} \mathcal{U}[n_1, \dots, n_M, n_1', \dots, n_M']\\
%     &\times \mathcal{Z}[n_1', \dots, n_M', n_1, \dots, n_M]
% \end{split}
% \end{equation}
% \end{definition}

\begin{definition}[Tucker Decomposition]\label{def:tucker}
The Tucker decomposition can be viewed as a form of high-order principal component analysis \cite{kolda2009tensor}. 
A tensor $\mathcal{X}\in\mathbb{R}^{N_1\times\cdots\times N_M}$ can be decomposed into a smaller core tensor $\mathcal{Z}\in\mathbb{R}^{N'_1\times\cdots\times N'_M}$ by $M$ orthonormal matrices $\mathbf{U}_m\in\mathbb{R}^{N'_m\times N_m}$ ($N'_m < N_m$):
\begin{equation}
    \mathcal{X} = \mathcal{Z}\prod_{m=1}^M\times_m\mathbf{U}_m
\end{equation}
% where $\textrm{mat}(\mathcal{U})=\mathbf{U}_M\otimes_k\cdots\otimes_k\mathbf{U}_1$ and $\otimes_k$ denotes the Kronecker product.
The matrix $\mathbf{U}_m$ is comprised of principal components for the $m$-\textrm{th} mode and the core tensor $\mathcal{Z}$ indicates the interactions among the components.
Due to the orthonormality of $\mathbf{U}_m$, we have:
\begin{equation}
    \mathcal{Z} = \mathcal{X}\prod_{m=1}^M\times_m\mathbf{U}_m^T
\end{equation}
\end{definition}

\begin{figure*}[t!]
    \centering
    \includegraphics[width=0.9\linewidth]{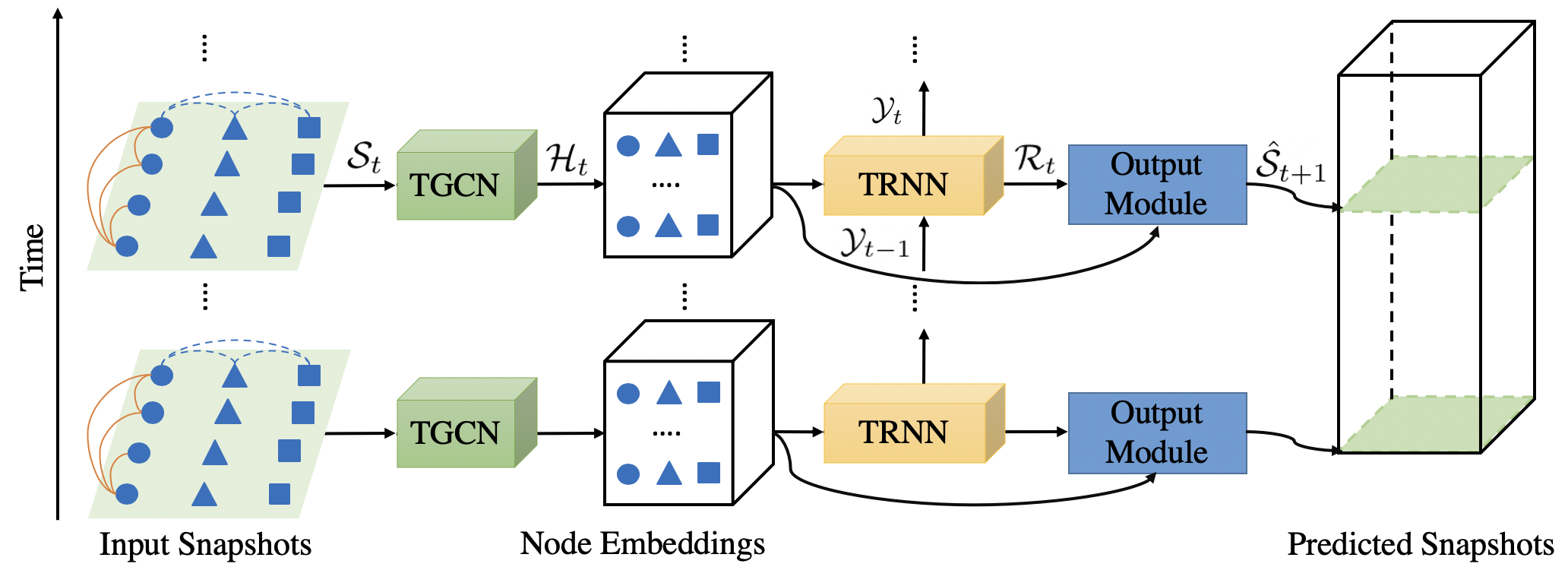}
    \caption{The framework of the proposed model \nettt. At each time step $t$, the model takes a snapshot $\mathcal{S}_t$ from the tensor time series $\mathcal{S}$ and extracts its node embedding tensor $\mathcal{H}_t$ via Tensor Graph Convolution Network (TGCN) module.
    $\mathcal{H}_t$ will be fed into the Tensor RNN (TRNN) module to encode the temporal dynamics. 
    Finally, the output module takes both of $\mathcal{H}_t$ and $\mathcal{R}_t$ to predict the snapshot of the next time step $\hat{\mathcal{S}}_{t+1}$.
    Note that $\mathcal{Y}_t$ and $\mathcal{Y}_{t+1}$ are the hidden states of TRNN at time step $t$ and $t+1$ respectively.}
    \label{fig:model}
\end{figure*}

\subsection{Multi-dimensional Fourier Transform}\label{sub:multift}
\begin{definition}[Multi-dimensional Discrete Fourier Transform]\label{def:mdft}
Given a multi-dimensional/mode signal $\mathcal{X}\in\mathbb{R}^{N_1\times\cdots\times N_M}$, the multi-dimensional Fourier transform is defined by:
\begin{equation}
\begin{split}
    \Tilde{\mathcal{X}}[n_1, \cdots, n_M] = \prod_{m=1}^M\sum_{k_m=1}^{N_m}e^{-\frac{i2\pi}{N_m}k_mn_m} \mathcal{X}[k_1, \cdots, k_M]
\end{split}
\end{equation}
Similar to the one-dimensional Fourier transform (Definition \ref{def:dft}), the above equation can be re-written by a multi-linear form:
\begin{equation}
    \Tilde{\mathcal{X}} = \mathcal{X}\times_1\mathbf{F_1}\cdots\times_M\mathbf{F_M} = \mathcal{X}\prod_{m=1}^M\times_m\mathbf{F}_m 
\end{equation}
where $\times_m$ denotes the mode-m product, $\mathbf{F}_m\in\mathbb{R}^{N_m\times N_m}$ is the filter matrix, and $\mathbf{F}_m[n, k] = e^{-\frac{i2\pi}{N_m}kn}$.
\end{definition}

\begin{definition}[Separable Multi-dimensional Convolution]\label{def:sep_conv}
The separable multi-dimensional convolution is defined based on Definition \ref{def:mdft}.
Given a signal $\mathcal{X}\in\mathbb{R}^{N_1\times\cdots\times N_M}$ and a separable filter $\mathcal{Y}\in\mathbb{R}^{N_1\times\cdots\times N_M}$ such that $\mathcal{Y}[n_1, \cdots, n_m] = \mathbf{y}_1[n_1]\cdots\mathbf{y}_M[n_m]$,
where $\mathbf{y}_m\in\mathbb{R}^{N_m}$ is the filter vector for the $m$-th mode, then the multi-dimensional convolution is the same as iteratively applying one dimensional convolution onto $\mathcal{X}$:
\begin{equation}
\begin{split}
    \mathcal{Y}\star\mathcal{X} = \mathbf{y}_1\star_1\cdots\star_{M-1}\mathbf{y}_M\star_{M}\mathcal{X}
\end{split}
\end{equation}
where $\star_m$ denotes convolution on the $m$-th mode.

Suppose $\mathcal{X}\in\mathbb{R}^{N_1\times N_2}$ and $\mathcal{Y}=\mathbf{y}_1\cdot\mathbf{y}_2^T$, where $\mathbf{y}_1\in\mathbb{R}^{N_1}$ and $\mathbf{y}_2\in\mathbb{R}^{N_2}$.
Then $\mathcal{Y}\star\mathcal{X}$ means applying $\mathbf{y}_1$ and $\mathbf{y}_2$ to the rows and columns of $\mathcal{X}$ respectively.
Formally we have:
\begin{equation}
    \mathcal{Y}\star\mathcal{X} = \mathbf{y}_1\star_1\mathbf{y}_2\star_2\mathcal{X} = \mathbf{Y}_1^T\mathcal{X}\mathbf{Y}_2 = \mathcal{X}\prod_{m=1}^2\times_{m}\mathbf{Y}_m
\end{equation}
where $\mathbf{Y}_1\in\mathbb{R}^{N_1\times N_1}$ and $\mathbf{Y}_2\in\mathbb{R}^{N_2\times N_2}$ are the transformation matrix corresponding to $\mathbf{y}_1$ and $\mathbf{y}_2$ respectively.
\end{definition}

\subsection{Network of Tensor Time Series}\label{sub:net3}
\begin{definition}[Tensor Time Series]
A tensor time series is a $(M+1)$-mode tensor $\mathcal{S}\in\mathbb{R}^{N_1\times\cdots\times N_M\times T}$ or $\{\mathcal{S}_t\in\mathbb{R}^{N_1\times\cdots\times N_M}\}_{t=1}^T$, where the $(M+1)$-th mode is the time and its dimension is $T$.
\end{definition}

\begin{definition}[Tensor Graph]\label{def:tensor_graph}
The tensor graph is comprised of a $M$-mode tensor $\mathcal{X}\in\mathbb{R}^{N_1\times\cdots\times N_M}$ and the adjacency matrices for each mode $\mathbf{A}_m\in\mathbb{R}^{N_m\times N_m}$.
Note that if $m$-th mode is not associated with an adjacency matrix, then $\mathbf{A}_m = \mathbf{I}_m$, where $\mathbf{I}_m\in\mathbb{R}^{N_m\times N_m}$ denotes the identity matrix.

\end{definition}

\begin{definition}[Network of Tensor Time Series]
A network of tensor time series is comprised of (1) a tensor time series $\mathcal{S}\in\mathbb{R}^{N_1\times\cdots\times N_M\times T}$ and (2) a set of adjacency matrices $\mathbf{A}_m\in\mathbb{R}^{N_m\times N_m}$ ($m\in[1,\cdots, M]$) for all but the last mode   (i.e., the time mode).
% Note that if the $m$-th mode does not associate with an adjacency matrix, then $\mathbf{A}_m=\mathbf{I}_m$.

\end{definition}

% \subsection{Problem Definition}\label{sec:problem_definition}
% In this paper, we focus on two classic problems for time series, namely, missing value recovery and future value prediction.
% \begin{definition}[Missing Value Recovery]\label{prob:missing}
% Given a network of tensor time series with $\mathcal{S}\in\mathbb{R}^{N_1\times\cdots\times N_M\times T}$ and $\mathbf{A}_m\in\mathbb{R}^{N_m\times N_m}$ ($m\in[1,\cdots, M]$), and an indicator $\mathcal{I}\in\mathbb{R}^{N_1\times\cdots\times N_M\times T}$ indicating the presence $\mathcal{I}[n_1, \cdots, n_M, t]=1$ and the absence $\mathcal{I}[n_1, \cdots, n_M, t]=0$ of the data point $\mathcal{S}[n_1, \cdots, n_M, t]$, the task of missing value recovery is to recover the missing data points of $\mathcal{S}$.
% \end{definition}

% \begin{definition}[Future Value Prediction]\label{prob:future}
% Given a network of tensor time series with $\mathcal{S}\in\mathbb{R}^{N_1\times\cdots\times N_M\times T}$ and $\mathbf{A}_m\in\mathbb{R}^{N_m\times N_m}$ ($m\in[1,\cdots, M]$), and a time step $T'$, the task of the future value prediction is to predict the future values of $\mathcal{S}$ from the time step $T+1$ to $T+T'$.
% \end{definition}

\subsection{Problem Definition}\label{sec:problem_definition}
In this paper, we focus on the representation learning for the network of tensor time series by predicting its future values. 
The model trained by predicting the future values can also be applied to recover the missing values of the time series.

\begin{definition}[Future Value Prediction]\label{prob:future}
Given a network of tensor time series with $\mathcal{S}\in\mathbb{R}^{N_1\times\cdots\times N_M\times T}$ and $\{\mathbf{A}_m\in\mathbb{R}^{N_m\times N_m}\}_{m=1}^M$, and a time step $T'$, the task of the future value prediction is to predict the future values of $\mathcal{S}$ from $T+1$ to $T+T'$.
\end{definition}

\begin{definition}[Missing Value Recovery]\label{prob:missing}
We formulate the task of missing value recovery from the perspective of future value prediction.
Suppose the data point $\mathcal{S}[n_1, \cdots, n_M, T']$ ($T'\leq T$) of $\mathcal{S}\in\mathbb{R}^{N_1\times\cdots\times N_M\times T}$ is missing, then we takes $\omega\leq T'$ historical values of $\mathcal{S}$ prior to the time step $T'$: $\{\mathcal{S}_t\}_{\omega=T'-\omega}^{T'-1}$ as input, and predict the value of the $\hat{\mathcal{S}}[n_1, \cdots, n_M, T']$. 
\end{definition}

\section{Methodology}\label{sec:methods}
An overview of the proposed \nettt\ is presented in Figure \ref{fig:model}, which works as follows.
At each time step $t$, the proposed Tensor Graph Convolutional Network (TGCN) (Section \ref{sec:tensor_graph_conv}) takes as input the $t$-th snapshot $\mathcal{S}_t\in\mathbb{R}^{N_1\times\cdots\times N_M}$ along with its adjacency matrices $\{\mathbf{A}_m\in\mathbb{R}^{N_m\times N_m}\}_{m=1}^M$ 
% from the tensor time series $\mathcal{S}\in\mathbb{R}^{N_1\times\cdots\times N_M\times T}$ 
and extracts its node embedding tensor $\mathcal{H}_t$, which will be fed into the proposed Tensor Recurrent Neural Network (TRNN) (Section \ref{sec:trnn}) to encode temporal dynamics and produce $\mathcal{R}_t$. Finally, the output module (Section \ref{sec:output}) takes both $\mathcal{H}_t$ and $\mathcal{R}_t$ to predict the snapshot of the next time step $\hat{\mathcal{S}}_{t+1}$.
Note that $\mathcal{Y}_{t}$ in Figure \ref{fig:model} denotes the hidden state of TRNN at the time step $t$.

\subsection{Tensor Graph Convolution Network}\label{sec:tensor_graph_conv}
%\hh{(1) let us have a small intro paragraph to summarize the structure of this subsection. -- we have space, and right now, it reads a bit dry and flow is a bit abrupt; (2) for the major things that we propose, explicitly use word 'we propose' instead of 'we introduce' or 'we define'. sec 3.2 does a better job on both points.}
%\by{updated}

In this subsection, we first introduce spectral graph convolution on tensor graphs and its Chebychev approximation in Subsection \ref{subsec:spectral_convo_for_tensor_graph}.
Then we provide a detailed derivation for the layer-wise updating function of the proposed TGCN in Subsection \ref{subsec:tgcl}.

\subsubsection{Spectral Convolution for Tensor Graph}\label{subsec:spectral_convo_for_tensor_graph}
Analogues to the multi-dimensional Fourier transform (Definition \ref{def:mdft}) and the graph Fourier transform on flat graphs (Definition \ref{def:gft}), we first define the Fourier transform on tensor graphs in Definition \ref{def:tensor_graph_ft}.
Then based on the separable multi-dimensional convolution (Definition \ref{def:sep_conv}), and tensor graph Fourier transform (Definition \ref{def:tensor_graph_ft}), we propose spectral convolution on tensor graphs in Definition \ref{def:multi_graph_conv}.
Finally, in Definition \ref{def:cheby_conv_tensor}, we propose to use Chebychev approximation in order to parameterize the free parameters in the filters of spectral convolution.

\begin{definition}[Tensor Graph Fourier Transform]\label{def:tensor_graph_ft}
Given a graph signal $\mathcal{X}\in\mathbb{R}^{N_1\times\cdots\times N_M}$, along with its adjacency matrices for each mode $\mathbf{A}_m\in\mathbb{R}^{N_m\times N_m}$ ($m\in[1,\cdots, M]$), the tensor graph Fourier transform is defined by:
\begin{equation}
    \Tilde{\mathcal{X}}=\mathcal{X}\prod_{m=1}^M\times_m\mathbf{\Phi}_m
\end{equation}
where $\mathbf{\Phi}_m$ is the eigenvector matrix of graph Laplacian matrix $\mathbf{L}_m=\mathbf{\Phi}_m\mathbf{\Lambda}_m\mathbf{\Phi}_m^T$ for $\mathbf{A}_m$;
$\times_m$ denotes the mode-m product.
\end{definition}

\begin{definition}[Spectral Convolution for Tensor Graph]\label{def:multi_graph_conv}
Given an input graph signal $\mathcal{X}\in\mathbb{R}^{N_1\times\cdots\times N_M}$, and a multi-dimensional filter $\mathcal{G}\in\mathbb{R}^{N_1\times\cdots\times N_M}$ defined by $\mathcal{G}[n_1,\cdots,n_M]=\mathbf{g}_1[n_1]\cdots\mathbf{g}_M[n_M]$, where $\mathbf{g}_m\in\mathbb{R}^{N_m}$ is the filter vector for the $m$-th mode. 
By analogizing to spectral graph convolution (Definition \ref{def:graph_conv}) and separable multi-dimensional convolution (Definition \ref{def:sep_conv}) , we define spectral convolution for tensor graph as:
\begin{equation}
    \mathcal{G}\star\mathcal{X} = \mathcal{X}\prod_{m=1}^M\times_m\mathbf{\Phi}^T_m\textrm{diag}(\Tilde{\mathbf{g}}_m)\mathbf{\Phi}_m
\end{equation}
where $\Tilde{\mathbf{g}}_m=\mathbf{\Phi}_m^T\mathbf{g}_m$ is the Fourier transformed filter for the $m$-th mode;
$\star$ and $\times_m$ denote the convolution operation and the mode-m product respectively;
$\textrm{diag}(\mathbf{g}_m)$ denotes the diagonal matrix, of which the diagonal elements are the elements in $\mathbf{g}_m$.
\end{definition}

\begin{definition}[Chebyshev Approximation for Spectral Convolution on Tensor Graph]\label{def:cheby_conv_tensor}
Given a tensor graph $\mathcal{X}\in\mathbb{R}^{N_1\times\cdots\times N_M}$, where each mode is associated with an adjacency matrix $\mathbf{A}_m\in\mathbb{R}^{N_m\times N_m}$, the Chebychev approximation for spectral convolution on tensor graphs is given by approximating $\Tilde{\mathbf{g}}_m$ by Chebyshev polynomials:
\begin{equation}\label{eq:cheby_conv_tensor}
\begin{split}
    \mathcal{G}_\theta\star\mathcal{X} &= \mathcal{X}\prod_{m=1}^M\times_m\mathbf{\Phi}^T_m(\sum_{p_m=0}^P\theta_{m,p_m}T_{p_m}(\Tilde{\mathbf{\Lambda}}_m))\mathbf{\Phi}_m\\
    &=\mathcal{X}\prod_{m=1}^M\times_m\sum_{p_m=0}^P\theta_{m,p_m}T_{p_m}(\Tilde{\mathbf{L}}_m)
\end{split}
\end{equation}
where $\mathcal{G}_\theta$ denotes the convolution filter parameterized by $\theta$;
$\mathbf{\Lambda}_m\in\mathbb{R}^{N_m\times N_m}$ is the
matrix of eigenvalues for the graph Laplacian matrix $\mathbf{L}_m=\mathbf{I}_m - \mathbf{D}_m^{-\frac{1}{2}}\mathbf{A}_m\mathbf{D}_m^{-\frac{1}{2}}=\mathbf{\Phi}_m\mathbf{\Lambda}_m\mathbf{\Phi}_m^T$;
$\Tilde{\mathbf{\Lambda}}_m = \frac{2}{\lambda_{m,max}}\mathbf{\Lambda}_m - \mathbf{I}_m$ is the normalized eigenvalues,
$\lambda_{m,max}$ is maximum eigenvalue in the matrix $\mathbf{\Lambda}_m$;
$\Tilde{\mathbf{L}}_m = \frac{2}{\lambda_{m,max}}\mathbf{L}_m - \mathbf{I}_m$;
$T_{p_m}(x)$ is Chebyshev polynomials defined by $T_{p_m}(x) = 2xT_{p_m-1}(x) - T_{p_m-2}(x)$ with $T_0(x) = 1$ and $T_1(x) = x$, and $p_m$ denotes the order of polynomials;
$\theta_{m,p_m}$ denote the co-efficient of $T_{p_m}(x)$.
For clarity, we use the same polynomial degree $P$ for all modes. 
\end{definition}

\subsubsection{Tensor Graph Convolutional Layer}\label{subsec:tgcl}
Due to the linearity of mode-m product, Equation \eqref{eq:cheby_conv_tensor} can be re-formulated as:
\begin{equation}\label{eq:cheby_conv_tensor_2}
\begin{split}
    \mathcal{G}_\theta\star\mathcal{X} &= \sum_{p_1, \cdots, p_M=0}^P\mathcal{X}\prod_{m=1}^M\times_m\theta_{m,p_m} T_{p_m}(\Tilde{\mathbf{L}}_m)\\
    &=\sum_{p_1, \cdots, p_M=0}^P\prod_{m=1}^M\theta_{m,p_m}  \mathcal{X}\prod_{m=1}^M\times_mT_{p_m}(\Tilde{\mathbf{L}}_m)
\end{split}
\end{equation}

We follow \cite{kipf2016semi} to simplify Equation \eqref{eq:cheby_conv_tensor_2}.
Firstly, let $\lambda_{m,max}=2$ and we have:
\begin{equation}\label{eq:L_tmp}
\begin{split}
    \Tilde{\mathbf{L}}_m &=\frac{2}{\lambda_{m,max}}\mathbf{L}_m - \mathbf{I}_m\\
    &=\mathbf{I}_m - \mathbf{D}_m^{-\frac{1}{2}}\mathbf{A}_m\mathbf{D}_m^{-\frac{1}{2}}- \mathbf{I}_m\\
    &=-\mathbf{D}_m^{-\frac{1}{2}}\mathbf{A}_m\mathbf{D}_m^{-\frac{1}{2}}
\end{split}
\end{equation}
For clarity, we use $\Tilde{\mathbf{A}}_m$ to represent $\mathbf{D}_m^{-\frac{1}{2}}\mathbf{A}_m\mathbf{D}_m^{-\frac{1}{2}}$.
Then we fix $P=1$ and drop the negative sign in Equation \eqref{eq:L_tmp} by absorbing it to parameter $\theta_{m,p_m}$.
Therefore, we have
\begin{equation}\label{eq:tmp2}
    \sum_{p=0}^P\theta_{m,p_m}T_p(\Tilde{\mathbf{L}}_m) = \theta_{m,0} + \theta_{m,1}\Tilde{\mathbf{A}}_m
\end{equation}
Furthermore, by plugging Equation \eqref{eq:tmp2} back into Equation \eqref{eq:cheby_conv_tensor_2} and replacing the product of parameters $\prod_{m=1}^M\theta_{m,p_m}$ by a single parameter $\theta_{p_1,\cdots, p_M}$, we will obtain:
\begin{equation}\label{eq:cheby_tensor_graph_final}
\begin{split}
    \mathcal{G}_\theta\star\mathcal{X} = \sum_{\exists p_m=1}\theta_{p_1,\cdots, p_M}\mathcal{X}\prod_{p_m=1}\times_m\Tilde{\mathbf{A}}_m + \theta_{0,\cdots, 0}\mathcal{X}
\end{split}
\end{equation}
We can observe from the above equation that $p_m$ works as an indicator for whether applying the convolution filter $\Tilde{\mathbf{A}}_m$ to $\mathcal{X}$ or not.
If $p_m=1$, then $\Tilde{\mathbf{A}}_m$ will be applied to $\mathcal{X}$, otherwise, $\mathbf{I}_m$ will be applied.
When $p_m=0$ for $\forall m\in[1, \cdots, M]$, we will have $\theta_{0,\cdots, 0}\mathcal{X}$. 
To better understand how the above approximation works on tensor graphs, let us assume $M=2$.
Then we have:
\begin{equation}\label{eq:cheby_tensor_graph_example}
    \mathcal{G}_\theta\star\mathcal{X} = \theta_{1,1}\mathcal{X}\times_1\Tilde{\mathbf{A}}_1\times_2\Tilde{\mathbf{A}}_2 + \theta_{1,0}\mathcal{X}\times_1\Tilde{\mathbf{A}}_1 + \theta_{0,1}\mathcal{X}\times_2\Tilde{\mathbf{A}}_2 + \theta_{0,0}\mathcal{X}
\end{equation}

Given the approximation in Equation \eqref{eq:cheby_tensor_graph_final}, we propose the tensor graph convolution layer in Definition \ref{def:tgcl}.

\begin{definition}[Tensor Graph Convolution Layer]\label{def:tgcl}
Given an input tensor $\mathcal{X}\in\mathbb{R}^{N_1\times\cdots\times N_M\times d}$, where $d$ is the number of channels, along with its adjacency matrices $\{\mathbf{A}_m\}_{m=1}^M$, the Tensor Graph Convolution Layer (TGCL) with $d'$ output channels is defined by:
\begin{equation}\label{eq:tgcl}
\begin{split}
    & \text{TGCL}(\mathcal{X}, \{\mathbf{A}_m\}_{m=1}^M)\\
    = & \sigma(\sum_{\exists p_m=1}\mathcal{X}\prod_{p_m=1}\times_m\Tilde{\mathbf{A}}_m\times_{M+1}\mathbf{\Theta}_{p_1,\cdots, p_M} + \mathcal{X}\times_{M+1}\mathbf{\Theta}_{0})
\end{split}
\end{equation}
where $\mathbf{\Theta}\in\mathbb{R}^{d\times d'}$ is parameter matrix;
$\sigma(\cdot)$ is activation function. 
\end{definition}

In the \nettt model (Figure \ref{fig:model}), given a snapshot $\mathcal{S}_t\in\mathbb{R}^{N_1\times\cdots\times N_M}$ along with its adjacency matrices $\{\mathbf{A}_m\}_{m=1}^M$, we use a one layer TGCL to obtain the node embeddings $\mathcal{H}_t\in\mathbb{R}^{N_1\times\cdots\times N_M\times d}$, where $d$ is the dimension of the node embeddings:
\begin{equation}\label{eq:tgcn2h}
    \mathcal{H}_t = \textrm{TGCN}(\mathcal{S}_t)
\end{equation}
% \hh{i think we still need to say sth like the output of tgcl is Ht, or put Ht in the equation, e.g., Ht=TGCL ....otherwise, it reads a bit dis-connected with 3.2}
% \by{updated}

\subsubsection{Synergy Analysis}
The proposed TGCL effectively models tensor graphs and captures the synergy among different adjacency matrices.
The vector $\mathbf{p}=[p_1, \cdots, p_M]\in[0,1]^M$ represents a combination of $M$ networks, where $p_m=1$ and $p_m=0$ respectively indicate the presence and absence of the $\Tilde{\mathbf{A}}_m$. 
Therefore, each node in $\mathcal{X}$ could collect other nodes' information along the adjacency matrix $\Tilde{\mathbf{A}}_m$ if $p_m=1$.
For example, suppose $M=2$ and $p_1=p_2=1$ (as shown in Figure \ref{fig:synergy} and Equation \eqref{eq:cheby_tensor_graph_example}), then node $\mathcal{X}[1, 1]$ (node $v$) could reach node $\mathcal{X}[2,2]$ (node $w'$) by passing node $\mathcal{X}[2, 1]$ along the adjacency matrix $\Tilde{\mathbf{A}}_1$ ($\mathcal{X}\times_1\Tilde{\mathbf{A}}_1$) and then arriving at node $\mathcal{X}[2,2]$ via $\Tilde{\mathbf{A}}_2$ ($\mathcal{X}\times_1\Tilde{\mathbf{A}}_1\times_2\Tilde{\mathbf{A}}_2$).
In contrast, with a traditional GCN layer, node $v$ can only gather information of its direct neighbors from a given model (node $v'$ via $\Tilde{\mathbf{A}}_1$ or $w$ via $\Tilde{\mathbf{A}}_2$).

An additional advantage of TGCL lies in that it is robust to missing values in $\mathcal{X}$ since TGCL is able to recover the value of a node from various combination of adjacency matrices.
For example, suppose the value of node $v=0$, then TGCL could recover its value by referencing the value of $v'$ (via $\mathcal{X}\times_1\Tilde{\mathbf{A}}_1$), or the value of $w$ (via $\mathcal{X}\times_2\Tilde{\mathbf{A}}_2$), or the value of $w'$ (via $\mathcal{X}\times_1\Tilde{\mathbf{A}}_1\times_2\Tilde{\mathbf{A}}_2$).
However, a GCN layer could only refer to the node $v'$ via $\Tilde{\mathbf{A}}_1$ or $w$ via $\Tilde{\mathbf{A}}_2$.

\begin{figure}[h!]
    \centering
    \includegraphics[width=.18\textwidth]{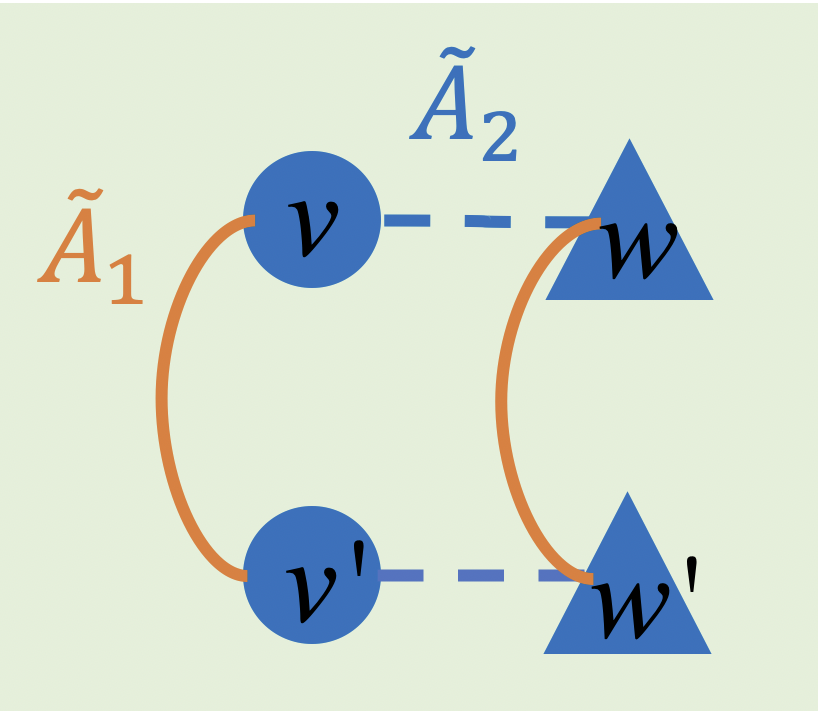}
    \caption{An illustration of synergy analysis of TGCL.}
    \label{fig:synergy}
\end{figure}

\begin{figure*}[t!]
    \centering
    \includegraphics[width=0.99\textwidth]{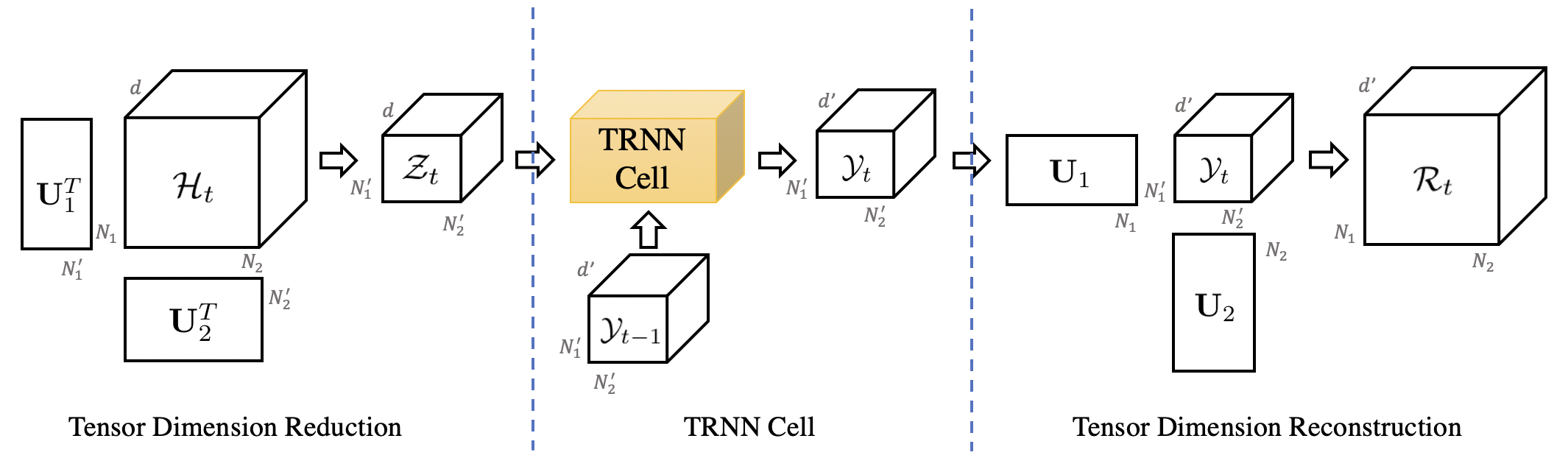}
    \caption{Tensor Recurrent Neural Network (TRNN).}
    \label{fig:trnn}
\end{figure*}

\subsubsection{Complexity Analysis}
% The complexity of the tensor graph convolution operation (Equation \eqref{eq:cheby_tensor_graph_final}) is 
% $O(2^{M-1}\prod_{m=1}^MN_m(2+\sum_{m=1}^MN_m))$. 
For a $M$-mode tensor with $K (1\leq K\leq M)$ networks, the complexity of the tensor graph convolution (Equation \eqref{eq:cheby_tensor_graph_final}) is $O(2^{K-1}\prod_{m=1}^MN_m(2+\sum_{k=1}^KN_k))$.

\subsection{Tensor Recurrent Neural Network}\label{sec:trnn}

Given the output from TGCN: $\mathcal{H}_t\in\mathbb{R}^{N_1\times\cdots\times N_M\times d}$ (Equation \eqref{eq:tgcn2h}), the next step is to incorporate temporal dynamics for $\mathcal{H}_t$.

% Traditional machine learning methods \cite{rogers2013multilinear, cai2015fast,cai2015facets} relying on linear dynamics systems fail to capture the non-linearity of temporal dynamics.
% Albeit recent deep learning methods successfully inject the non-linearity, most of them \cite{li2017diffusion, yu2017spatio, li2019predicting, yan2018spatial} ignored the specificity of each time series by simply using a same function to model temporal dynamics for all time series, which will result in sub-optimal solutions.
% A few methods \cite{srivastava2018comparative, zhou2018recover} ignored the commonality by instantiating separate RNNs for different time series, and these methods do not scale well in practice.
% To solve the drawbacks of the existing methods, 
As shown in Figure \ref{fig:trnn}, we propose a novel Tensor Recurrent Neural Network (TRNN), 
which captures the implicit relation among co-evolving time series by decomposing $\mathcal{H}_t$ into a low dimensional core tensor $\mathcal{Z}_t\in\mathbb{R}^{N_1'\times\cdots\times N_M'\times d}$ ($N_m'<N_m$) via a Tensor Dimension Reduction module (Section \ref{sec:tensor_reduction}).
The Tensor RNN Cell (Section \ref{sec:tensor_rnn_cell}) further introduces non-linear temporal dynamics into $\mathcal{Z}_t$ and produces the hidden state $\mathcal{Y}_{t}\in\mathbb{R}^{N_1'\times\cdots\times N_M'\times d}$.
Finally, the Tensor Dimension Reconstruction module (Section \ref{sec:tensor_reconstruction}) reconstructs $\mathcal{Y}_{t}$ and generates the reconstructed tensor $\mathcal{R}_t\in\mathbb{R}^{N_1\times\cdots\times N_M\times d}$.

\subsubsection{Tensor Dimension Reduction}\label{sec:tensor_reduction}
As shown in the left part of Figure \ref{fig:trnn}, the proposed tensor dimension reduction module will reduce the dimensionality of each mode of $\mathcal{H}_t\in\mathbb{R}^{N_1\times\cdots\times N_M\times d}$, except for the last mode (hidden features), by leveraging Tucker decomposition (Definition \ref{def:tucker}):
\begin{equation}\label{eq:trnn_decompose}
    \mathcal{Z}_t = \mathcal{H}_t\prod_{m=1}^M\times_m\mathbf{U}_m^T
\end{equation}
where $\mathbf{U}_m\in\mathbb{R}^{N'_m\times N_m}$ denotes the orthonormal parameter matrix, which is learnable via backpropagation; 
$\mathcal{Z}_t\in\mathbb{R}^{N_1' \times\cdots\times N_m'\times d}$ is the core tensor of $\mathcal{H}_t$.

\subsubsection{Tensor RNN Cell}\label{sec:tensor_rnn_cell}
Classic RNN cells, e.g. Long-Short-Term-Memory (LSTM) \cite{hochreiter1997long} 
% and Gated Recurrent Unit (GRU) \cite{cho2014learning}, 
are designed for a single input sequence, and therefore do not directly capture the correlation among co-evolving sequences.
To address this problem, we propose a novel Tensor RNN (TRNN) cell based on tensor algebra.

We first propose a Tensor Linear Layer (TLL):
\begin{equation} \label{eq:ttl}
    \text{TLL}(\mathcal{X}) = \mathcal{X}\prod_{m=1}^{M+1}\times_m\mathbf{W}_m + \mathbf{b}
\end{equation}
where $\mathcal{X}\in\mathbb{R}^{N_1\times\cdots\times N_M\times d}$ is the input tensor,
and $\mathbf{W}_m\in\mathbb{R}^{N_m\times N_m'}$ ($\forall m\in[1,\cdots, M]$) and $\mathbf{W}_{M+1}\in\mathbb{R}^{d\times d'}$ are the linear transition parameter matrices;
$\mathbf{b}\in\mathbb{R}^{d'}$ denotes the bias vector.
% $\mathbf{W}_m$ ($\forall m\in[1,\cdots, M]$) captures the correlation of the $m$-th mode and $\mathbf{W}_{M+1}$ is the linear transformation matrix mapping hidden representation from the $\mathbb{R}^{d}$ space to $\mathbb{R}^{d'}$ space.

TRNN can be obtained by replacing the linear functions in any RNN cell with the proposed TLL.
We take LSTM as an example to re-formulate its updating equations.
By replacing the linear functions in the LSTM with the proposed TLL, we have updating functions for Tensor LSTM (TLSTM)\footnote{Bias vectors are omitted for clarity.}:
\begin{align}
    \mathcal{F}_{t} &= \sigma(\textrm{TLL}_{fz}(\mathcal{Z}_t) +  \textrm{TLL}_{fy}(\mathcal{Y}_{t-1})) \\
    \mathcal{I}_{t} &= \sigma(\textrm{TLL}_{iz}(\mathcal{Z}_t) + \textrm{TLL}_{iy}(\mathcal{Y}_{t-1})) \\
    \mathcal{O}_{t} &= \sigma(\textrm{TLL}_{oz}(\mathcal{Z}_t) + \textrm{TLL}_{oy}( \mathcal{Y}_{t-1})) \\
    \Tilde{\mathcal{C}}_{t} &= \tanh(\textrm{TLL}_{cz}(\mathcal{Z}_t) + \textrm{TLL}_{cy}(\mathcal{Y}_{t-1})) \\
    \mathcal{C}_{t} &= \mathcal{F}_{t} \odot \mathcal{C}_{t-1} + \mathcal{I}_{t} \odot \Tilde{\mathcal{C}}_{t} \\
    \mathcal{Y}_{t} &= \mathcal{O}_{t} \odot \sigma(\mathcal{C}_{t})
\end{align}
% \begin{align}
%     \mathcal{F}_{t} &= \sigma(\mathcal{W}_f \circledast \mathcal{Z}_t + \mathcal{V}_f \circledast \mathcal{Y}_{t-1} + \mathbf{b}_f) \\
%     \mathcal{I}_{t} &= \sigma(\mathcal{W}_i \circledast \mathcal{Z}_t + \mathcal{V}_i \circledast \mathcal{Y}_{t-1} + \mathbf{b}_i) \\
%     \mathcal{O}_{t} &= \sigma(\mathcal{W}_o \circledast \mathcal{Z}_t + \mathcal{V}_o \circledast \mathcal{Y}_{t-1} + \mathbf{b}_o) \\
%     \Tilde{\mathcal{C}}_{t} &= \tanh(\mathcal{W}_c \circledast \mathcal{Z}_t + \mathcal{V}_c \circledast \mathcal{Y}_{t-1} + \mathbf{b}_c) \\
%     \mathcal{C}_{t} &= \mathcal{F}_{t} \odot \mathcal{C}_{t} + \mathcal{I}_{t} \odot \Tilde{\mathcal{C}}_{t-1} \\
%     \mathcal{Y}_{t} &= \mathcal{O}_{t} \odot \sigma(\mathcal{C}_{t})
% \end{align}
where $\mathcal{Z}_t\in\mathbb{R}^{N_1'\times\cdots\times N_M'\times d}$ and $\mathcal{Y}_t\in\mathbb{R}^{N_1'\times\cdots\times N_M'\times d'}$ denote the input core tensor and the hidden state tensor at the time step $t$; 
$\mathcal{F}_{t}$, $\mathcal{I}_{t}$, $\mathcal{O}_{t}\in\mathbb{R}^{N_1'\times\cdots\times N_M'\times d'}$ denote the forget gate, the input gate and the output gate, respectively;
$\Tilde{\mathcal{C}}_{t}\in\mathbb{R}^{N_1'\times\cdots\times N_M'\times d'}$ is the tensor for updating the cell memory $\mathcal{C}_t\in\mathbb{R}^{N_1'\times\cdots\times N_M'\times d'}$; 
TLL$_{\ast}(\cdot)$ denotes the tensor linear layer (Equation \eqref{eq:ttl}), and its subscripts in the above equations are used to distinguish different initialization of TLL\footnote{For all TLL related to $\mathcal{Z}_t$: TLL$_{\ast z}(\cdot)$, $\mathbf{W}_m\in\mathbb{R}^{N_m\times N_m'}$ ($\forall m\in[1,\cdots, M]$) and $\mathbf{W}_{M+1}\in\mathbb{R}^{d\times d'}$. For all TLL related to $\mathcal{Y}_{t-1}$: TLL$_{\ast y}(\cdot)$, $\mathbf{W}_m\in\mathbb{R}^{N_m'\times N_m'}$ ($\forall m\in[1,\cdots, M]$) and $\mathbf{W}_{M+1}\in\mathbb{R}^{d'\times d'}$.};
% $\mathbf{b}_\ast$ denotes the bias vectors;
$\sigma(\cdot)$ and $\tanh(\cdot)$ denote the sigmoid activation and tangent activation functions respectively; $\odot$ denotes the Hadamard product.

\subsubsection{Tensor Dimension Reconstruction}\label{sec:tensor_reconstruction}
To predict the values of each time series, we need to reconstruct the dimensionality of each mode. Thanks to  the orthonormality of $\mathbf{U}_m$ ($\forall m\in[1, \cdots, M]$), 
we can naturally reconstruct the dimensionality of $\mathcal{Y}_t\in\mathbb{R}^{N_1'\times\cdots\times N_M'\times d'}$ as follows: 
\begin{equation}
    \mathcal{R}_t = \mathcal{Y}_t\prod_{m=1}^M\times_m\mathbf{U}_m
\end{equation}
where $\mathcal{R}_{t}\in\mathbb{R}^{N_1\times\cdots\times N_M\times d'}$ is the reconstructed tensor.

\subsubsection{Implicit Relationship}
The Tucker decomposition (Definition \ref{def:tucker} and Equation \eqref{eq:trnn_decompose} can be regarded as high-order principal component analysis \cite{kolda2009tensor}.
The matrix $\mathbf{U}_m$ extracts eigenvectors of the $m$-th mode, and each element in $\mathcal{Z}$ indicates the relation between different eigenvectors.
We define $\rho\geq0$ as the indicator of \textit{interaction degree}, such that $N_m' = \rho N_m$ ($\forall m\in[1,\cdots, M]$), to represent to what degree does the TLSTM capture the correlation. 
The ideal range for $\rho$ is $(0,1)$.
When $\rho=0$, the TLSTM does not capture any relations and it is reduced to a single LSTM.
When $\rho=1$, the TLSTM captures the relation for each pair of the eigenvectors.
When $\rho>1$, the $\mathbf{U}_m$ is over-complete and contains redundant information.

Despite the dimentionality reduced by Equation \eqref{eq:trnn_decompose}, it is not guaranteed that the number of parameters in TLSTM will always be less than the number of parameters in multiple separate LSTMs, because of the newly introduced parameters $\mathbf{U}_m$ ($\forall m\in[1, \cdots, M]$).
The following lemma provides an upper-bound for $\rho$ given the dimensions of the input tensor and the hidden dimensions.

\begin{lemma}[Upper-bound for $\rho$]
Let $N_m$ and $N_m'$ be the dimensions of $\mathbf{U}_m$ in Equation \eqref{eq:trnn_decompose}, and let $d\in\mathbb{R}$ and $d'\in\mathbb{R}$ be the hidden dimensions of the inputs and outputs of TLSTM. TLSTM uses less parameters than multiple separate LSTMs, as long as the following condition holds:
\begin{equation}\label{eq:reduction_guarantee}
    \rho \leq \sqrt{\frac{(\prod_{m=1}^MN_m-1)d'(d+d'+1)}{2\sum_{m=1}^MN_m^2}+\frac{1}{256}} - \sqrt{\frac{1}{256}}
\end{equation}
\end{lemma}
\begin{proof}
There are totally $\prod_{m=1}^MN_m$ time series in the tensor time series $\mathcal{S}\in\mathbb{R}^{N_1\times\cdots\times N_M\times T}$, and thus the total number of parameters for $\prod_{m=1}^MN_m$ separate LSTM is:
\begin{equation}
\begin{split}
    N^{(LSTM)} &= \prod_{m=1}^{M}N_m[4(d d' + d' d' + d')] \\
    &=4d'(d+d'+1)\prod_{m=1}^MN_m
\end{split}
\end{equation}

The total number of parameters for the TLSTM is:
\begin{equation}
        N^{(TLSTM)} = 4d'(d+d'+1) + 8\sum_{m=1}^MN_m'^2 + \sum_{m=1}^M N_m'N_m
\end{equation}
where the first two terms on the right side are the numbers of parameters of the TLSTM cell, and the third term is the number of parameters required by $\{\mathbf{U}_m\}_{m=1}^M$ in the Tucker decomposition.

Let $\Delta = N^{(TLSTM)} - N^{(LSTM)}$, and let's replace $N_m'$ by $\rho N_m$, then we have:
\begin{equation}
    \Delta = (8\rho^2 + \rho)\sum_{m=1}^{M}N_m^2 - 4(\prod_{m=1}^MN_m - 1)d'(d+d'+1)
\end{equation}
Obviously, $\Delta$ is a convex function of $\rho$.
Hence, as long as $\rho$ satisfies the condition specified in the following equation, it can be ensured that the number of parameters is reduced.
\begin{equation}\label{eq:reduction_guarantee}
    \rho \leq \sqrt{\frac{(\prod_{m=1}^MN_m-1)d'(d+d'+1)}{2\sum_{m=1}^MN_m^2}+\frac{1}{256}} - \sqrt{\frac{1}{256}}
\end{equation}\end{proof}

\subsection{Output Module}\label{sec:output}
Given the reconstructed hidden representation tensor obtained from the TRNN: $\mathcal{R}_{t+1}\in\mathbb{R}^{N_1\times\cdots\times N_M\times d'}$, which captures the temporal dynamics, and the node embedding of the current snapshot $\mathcal{S}_t$: $\mathcal{H}_t\in\mathbb{R}^{N_1\times\cdots\times N_M\times d}$, the output module is a function mapping $\mathcal{R}_t$ and $\mathcal{H}_t$ to $\mathcal{S}_{t+1}\in\mathbb{R}^{N_1\times N_2 \cdots \times N_M}$.

We use a Multi-Layer Perceptron (MLP) with a linear output activation as the mapping function:
\begin{equation}
    \hat{\mathcal{S}}_{t+1} = \text{MLP}([\mathcal{H}_t, \mathcal{R}_t])
\end{equation}
where $\hat{\mathcal{S}}_{t+1}\in\mathbb{R}^{N_1\times\cdots\times N_M}$ represents the predicted snapshot; $\mathcal{H}_t$ and $\mathcal{R}_t$ are the outputs of TGCN and TRNN respectively; and $[\cdot, \cdot]$ denotes the concatenation operation.

\subsection{Training}
Directly training RNNs over the entire sequence is impractical in general \cite{sutskever2013training}.
A common practice is to partition the long time series data by a certain window size with $\omega$ historical steps and $\tau$ future steps \cite{li2017diffusion, yu2017spatio, li2019predicting}.

Given a time step $t$, let $\{\mathcal{S}_{t'}\}_{t'=t-\omega+1}^t$ and $\{\mathcal{S}_{t'}\}_{t'=t+1}^{t+\tau}$ be the historical and the future slices, the objective function of one window slice is defined as:
\begin{equation}
\begin{split}
    \arg \min_{\mathbf{\Theta}, \mathbf{\mathcal{W}}, \mathbf{\mathcal{B}}} & ||\nettt(\{\mathcal{S}_{t'}\}_{t'=t-\omega+1}^t) -  \{\mathcal{S}_{t'}\}_{t'=t+1}^{t+\tau})||_F^2\\
    + & \mu_1\sum_{t'=t-\omega+1}^{t}||\mathcal{H}_{t'}- \mathcal{Z}_{t'}\prod_{m=1}^M\times_m\mathbf{U}_m||_F^2 \\
    + & \mu_2 \sum_{m=1}^M||\mathbf{U}_m\mathbf{U}^T_m - \mathbf{I}_m||_F^2
\end{split}
\end{equation}
where \nettt denotes the proposed model;
$\mathbf{\Theta}$ and $\mathbf{\mathcal{W}}$  represent the parameters of TGCN and TRNN respectively;
$\mathbf{\mathcal{B}}$ denotes the bias vectors;
% \footnote{Subscripts and superscripts are ignored.} 
the second term denotes the reconstruction error of the Tucker decomposition;
the third term denotes the orthonormality regularization for $\mathbf{U}_m$, and
$\mathbf{I}_m$ denotes identity matrix ($\forall m\in[1,\cdots, M]$);
$||\cdot||_F$ is the Frobenius norm;
$\mu_1$ and $\mu_2$ are coefficients.

\section{Experiments}\label{sec:experiments}
In this section, we present the experimental results for the following questions:
\begin{itemize}
    \item [Q1.] How accurate is the proposed \nettt on recovering missing value and predicting future value?
    \item [Q2.] To what extent does the synergy captured by the proposed TGCN help improve the overall performance of \nettt?
    \item [Q3.] How does the interaction degree $\rho$ impact the performance of \nettt?
    \item [Q4.] How efficient and scalable is the proposed \nettt?
    % \hh{is it just TLSTM or Net3. the text in 4.3 seems to suggest the latter. double check}
    % \by{updated}
\end{itemize}

We first describe the datasets, comparison methods and implementation details in Subsection \ref{sec:exp_setup}, then we provide the results of the effectiveness and efficiency experiments in Subsection \ref{sec:exp_effectiveness} and Subsection \ref{sec:exp_efficiency}, respectively.
\subsection{Experimental Setup}\label{sec:exp_setup}
\subsubsection{Datasets} We evaluate the proposed \nettt\ model on five real-world datasets, whose statistics is summarized in Table~\ref{tab:dataset}.

\paragraph{Motes Dataset}
The \textit{Motes} dataset\footnote{\url{http://db.csail.mit.edu/labdata/labdata.html}} \cite{motes_dataset} is a collection of reading log from 54 sensors deployed in the Intel Berkeley Research Lab.
Each sensor collects 4 types of data, i.e., temperature, humidity, light, and voltage.
% The dataset also provides location information of each sensor and average connectivity probabilities among sensors.
Following \cite{cai2015facets}, we evaluate all the methods on the log of one day, which has 2880 time steps in total, yielding a $54\times 4\times 2880$ tensor time series.
We use the average connectivity of each pair of sensors to construct the network for the first mode (54 sensors).
As for the network of four data types, we use the Pearson correlation coefficient between each pair of them:

\begin{equation}\label{eq:pearson}
    \mathbf{A}[i,j] = \frac{1}{2}(r_{ij} + 1)
\end{equation}
where $r_{ij}\in[-1,1]$ denotes the Pearson correlation coefficient between the sequence $i$ and the sequence $j$.

\paragraph{Soil Dataset}
The \textit{Soil} dataset contains one-year log of water temperature and volumetric water content collected from 42 locations and 5 depth levels in the Cook Agronomy Farm (CAF)\footnote{\url{http://www.cafltar.org/}} near Pullman, Washington, USA, \cite{gasch2017pragmatic} which forms a $42\times 5\times 2\times 365$ tensor time series.
Since the dataset neither provides the specific location information of sensors nor the relation between the water temperature and volumetric water content, we use Pearson correlation, as shown in Equation \eqref{eq:pearson}, to build the adjacency matrices for all the modes.

\begin{table}[t]
    \centering
    \caption{Statistics of the datasets.}
    \begin{tabular}{c|r|r|r}
    \hline
    Dataset & Shape & \# Nodes & Modes with $\mathbf{A}$\\
    \hline
    \textit{Motes} & $54\times4\times2880$ & 216 & 1, 2\\
    \textit{Soil} & $42\times5\times2\times365$ & 420 & 1, 2, 3\\
    \textit{Revenue} & $410\times3\times62$ & 1,230 & 1, 2\\
    \textit{Traffic} & $1000\times2\times1440$ & 2,000 & 1 \\
    \textit{20CR} & $30\times30\times20\times6\times180$ & 108,000 & 1, 2, 3, 4\\
    \hline
    \end{tabular}
    \label{tab:dataset}
\end{table}

\paragraph{Revenue Dataset}
The \textit{Revenue} dataset is comprised of an actual and two estimated quarterly revenues for 410 major companies (e.g. Microsoft Corp.\footnote{\url{https://www.microsoft.com/}}, Facebook Inc.\footnote{\url{https://www.facebook.com/}}) from the first quarter of 2004 to the second quarter of 2019, which yields a $410\times 3\times 62$ tensor time series.
We construct a co-search network \cite{lee2015search} based on log files of the U.S Securities and Exchange Commission (SEC)\footnote{\url{https://www.sec.gov/dera/data/edgar-log-file-data-set.html}} to represent the correlation among different companies, which is used as the adjacency matrix for the first mode.
We also use the Pearson correlation coefficient to construct the adjacency matrix for the three revenues as in Equation \eqref{eq:pearson}. 

\paragraph{Traffic Dataset}
The \textit{Traffic} dataset is collected from Caltrans Performance Measurement System (PeMS).\footnote{\url{https://dot.ca.gov/programs/traffic-operations/mobility-performance-reports}}
Specifically, hourly average speed and occupancy of 1,000 randomly chosen sensor stations in District 7 of California from June 1, 2018, to July 30, 2018, are collected, which yields a $1000\times 2\times 1440$ tensor time series.
The adjacency matrix $\mathbf{A}_1$ for the first mode is constructed by indicating whether two stations are adjacent:
$\mathbf{A}_1[i,j]=1$ represents the stations $i$ and $j$ are next to each other.
As for the second mode, since the Pearson correlation between speed and occupancy is not significant, we use identity matrix $\mathbf{I}$ as the adjacency matrix.

\paragraph{20CR Dataset}
We use the version 3 of the 20th Century Reanalysis data\footnote{20th Century Reanalysis V3 data provided by the NOAA/OAR/ESRL PSL, Boulder, Colorado, USA, from their Web site \url{https://psl.noaa.gov/data/gridded/data.20thC_ReanV3.html}}\footnote{Support for the Twentieth Century Reanalysis Project version 3 dataset is provided by the U.S. Department of Energy, Office of Science Biological and Environmental Research (BER), by the National Oceanic and Atmospheric Administration Climate Program Office, and by the NOAA Physical Sciences Laboratory.} \cite{compo2011twentieth, slivinski2019towards}
collected by the National Oceanic and Atmospheric Administration (NOAA) Physical Sciences Laboratory (PSL).
We use a subset of the full dataset, which covers a $30\times30$ area of north America, ranging from $30^\circ$ N to $60^\circ$ N, $80^\circ$ W to $110^\circ$ W, and it contains 20 atmospheric pressure levels.
For each of the location point, 6 attributes are used, including air temperature, specific humidity, omega, u wind, v wind and geo-potential height.\footnote{For details of the attributes, please refer to the 20th Century Reanalysis project \url{https://psl.noaa.gov/data/20thC_Rean//}}
We use the monthly average data ranging from 2001 to 2015.
Therefore, the shape of the data is $30\times30\times20\times6\times180$.
The adjacency matrix $\mathbf{A}_1$ for the first mode, latitude, is constructed by indicating whether two latitude degrees are next to each other: $\mathbf{A}_1[i,j]=1$ if $i$ and $j$ are adjacent.
The adjacency matrices $\mathbf{A}_2$ and $\mathbf{A}_3$ for the second and the third modes are built in the same way as $\mathbf{A}_1$.
We build $\mathbf{A}_4$ for the 6 attributes based on Equation \eqref{eq:pearson}. 

\subsubsection{Comparison Methods}
We compare our methods with both classic methods (DynaMMo \cite{li2009dynammo}, MLDS \cite{rogers2013multilinear}) and recent deep learning methods (DCRNN \cite{li2017diffusion}, STGCN \cite{yu2017spatio}).
We also compare the proposed full model \nettt with its ablated versions.
To evaluate TGCN, we compare it with MLP, GCN \cite{kipf2016semi} and iTGCN.
Here, iTGCN is an ablated version of TGCN, which ignores the synergy between adjacency matrices. 
The updating function of iTGCN is given by the following equation:
\begin{equation}\label{eq:itgcn}
    \sigma(\sum_{m=1}^M\mathcal{X}\times_m\Tilde{\mathbf{A}}_m\times_{M+1}\mathbf{\Theta}_{m} + \mathcal{X}\times_{M+1}\mathbf{\Theta}_0)
\end{equation}
where $\sigma(\cdot)$ denotes the activation function, $\mathbf{\Theta}$ denotes parameter matrix and $\mathcal{X}\in\mathbb{R}^{N_1\times\cdots\times N_M\times d}$. 
For a fair comparison with GCN and the baseline methods, we construct a flat graph by combining the adjacency matrices:
\begin{equation}
    \mathbf{A}=\mathbf{A}_M\otimes_k\cdots\otimes_k\mathbf{A}_1
\end{equation}
where $\otimes_k$ is Kronecker product, the dimension of $\mathbf{A}$ is $\prod_{m=1}^MN_m$, and $N_m$ is the dimension of $\mathbf{A}_m$.
To evaluate TLSTM, we compare it with multiple separate LSTMs (mLSTM) and a single LSTM.
\begin{figure*}[t]
\centering
\subfloat[Motes-Missing]
{\includegraphics[width=.20\linewidth]{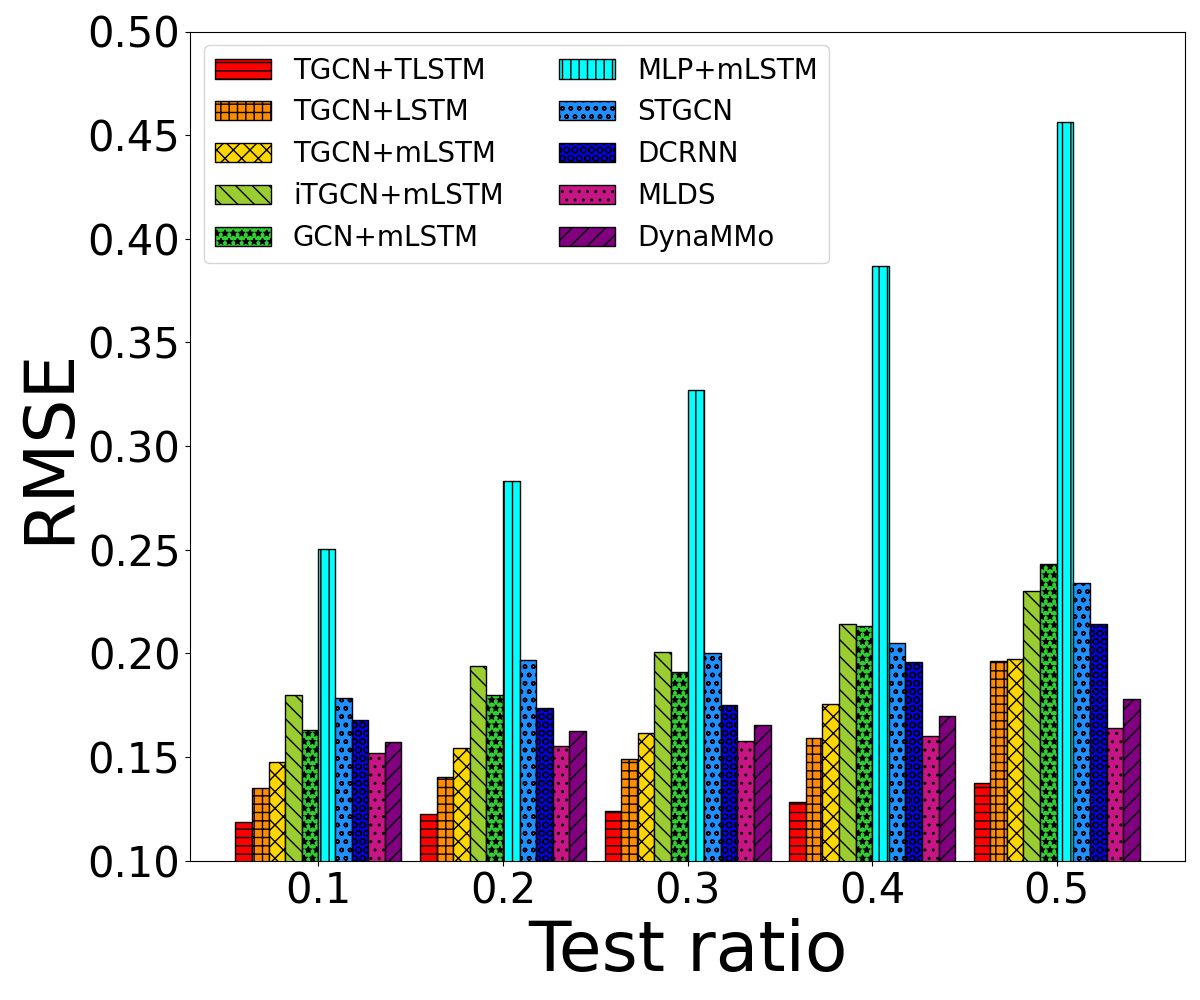}\label{fig:exp_motes_missing}}\,
\subfloat[Soil-Missing]
{\includegraphics[width=.20\linewidth]{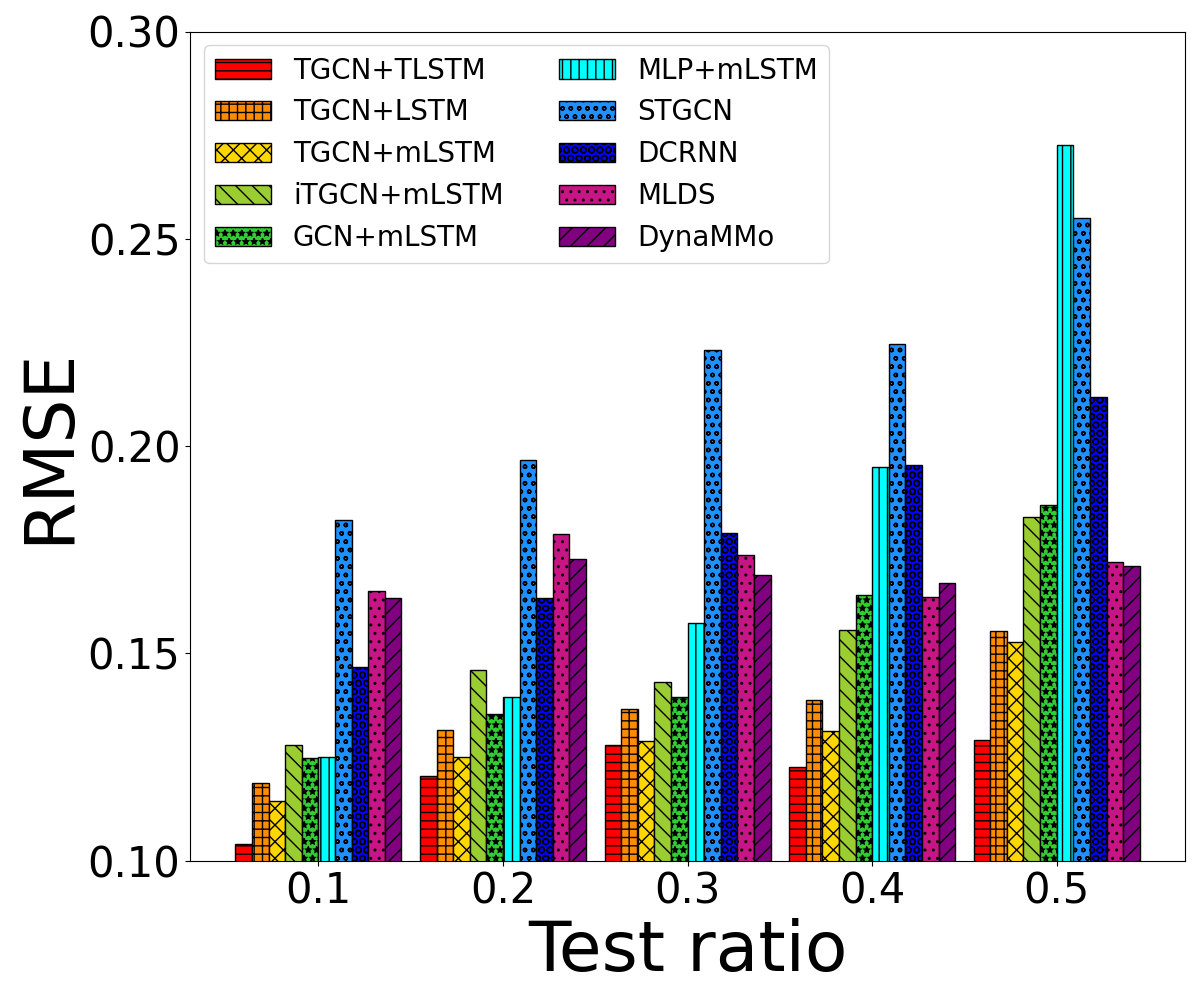}\label{fig:exp_soil_missing}}\,
\subfloat[Revenue-Missing]
{\includegraphics[width=.20\linewidth]{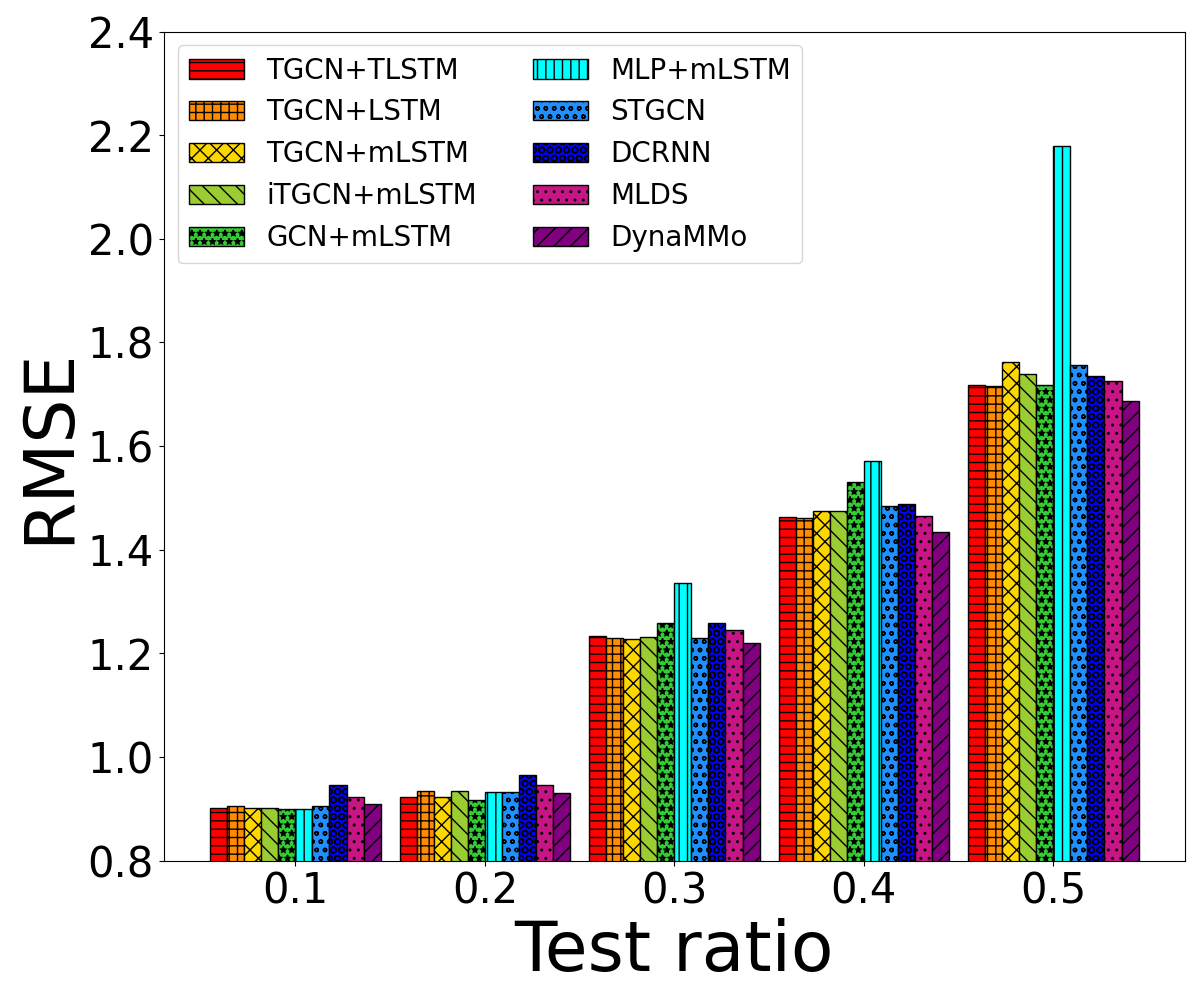}\label{fig:exp_revenue_missing}}\,
\subfloat[Traffic-Missing]
{\includegraphics[width=.20\linewidth]{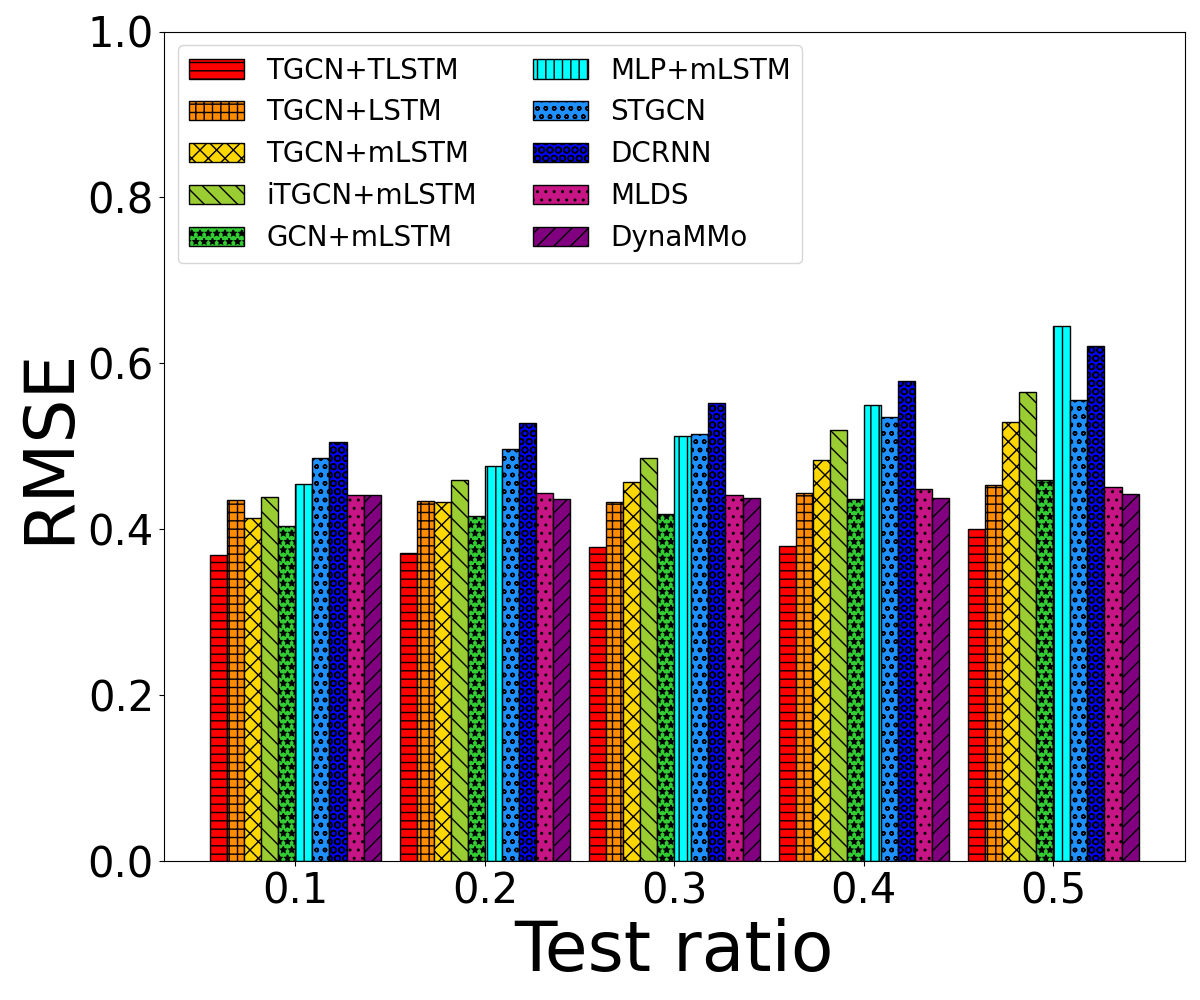}\label{fig:exp_traffic_missing}}\,
\subfloat[Motes-Future]
{\includegraphics[width=.20\linewidth]{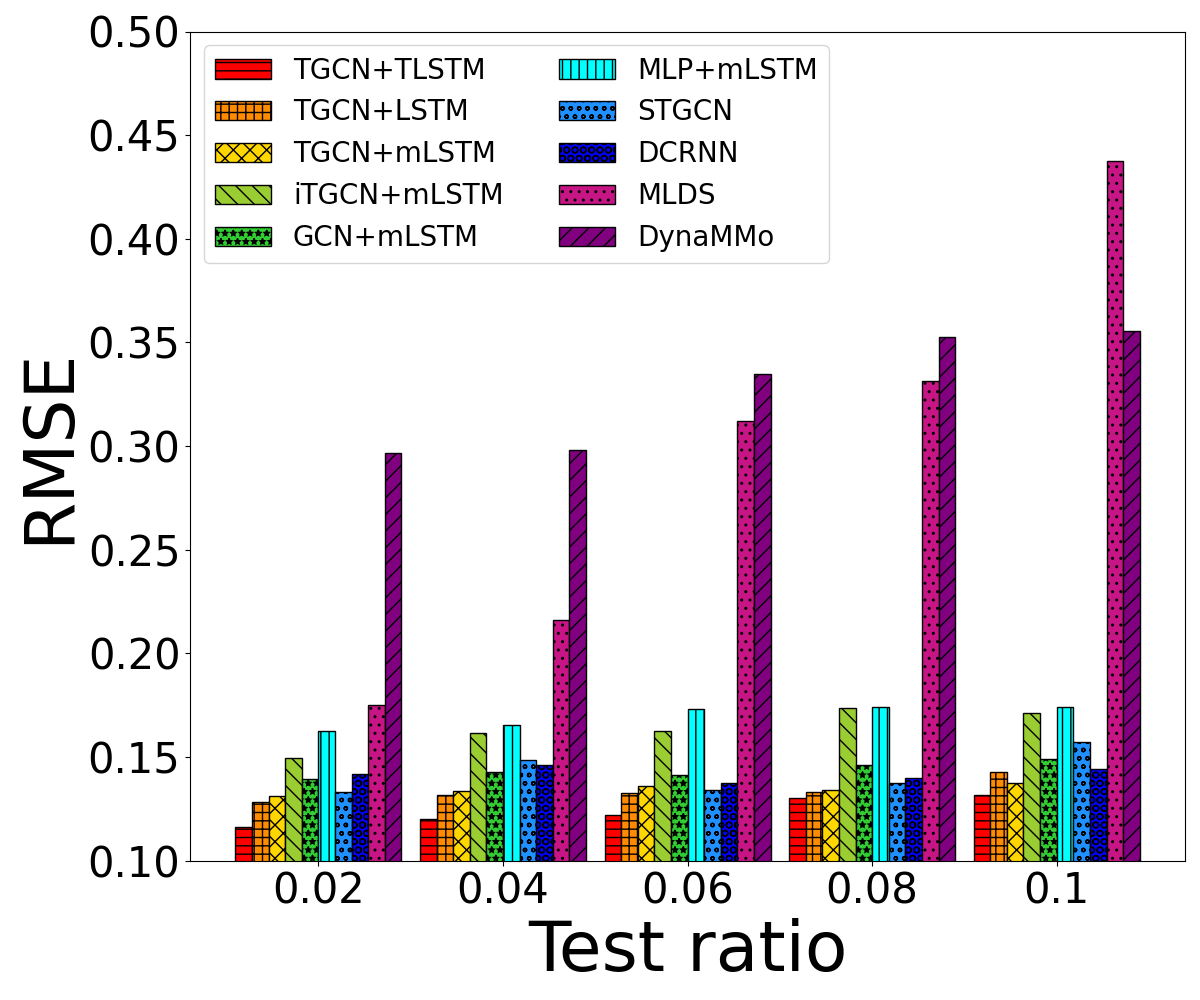}\label{fig:exp_motes_future}}\,
\subfloat[Soil-Future]
{\includegraphics[width=.20\linewidth]{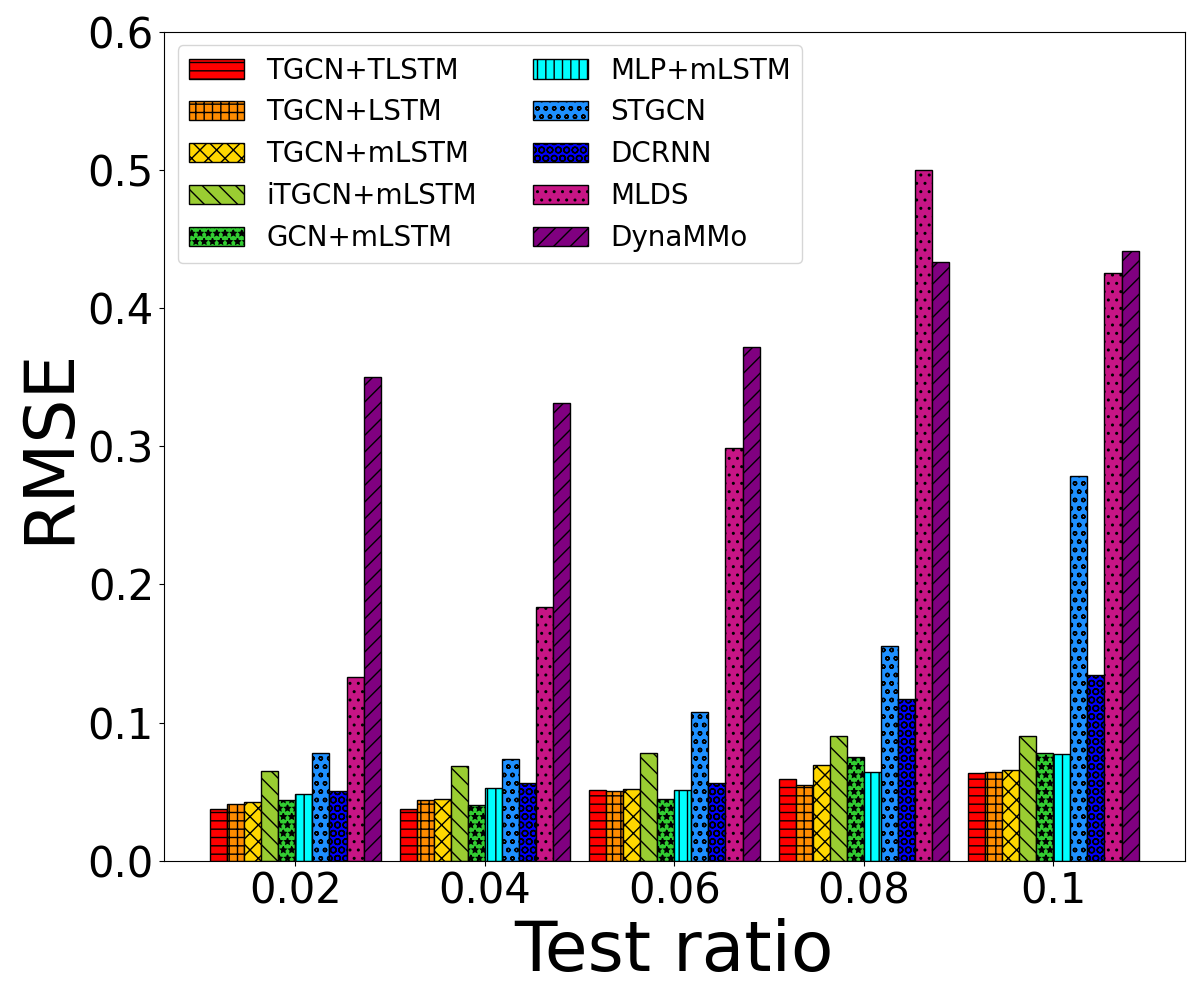}\label{fig:exp_soil_future}}
\subfloat[Revenue-Future]
{\includegraphics[width=.20\linewidth]{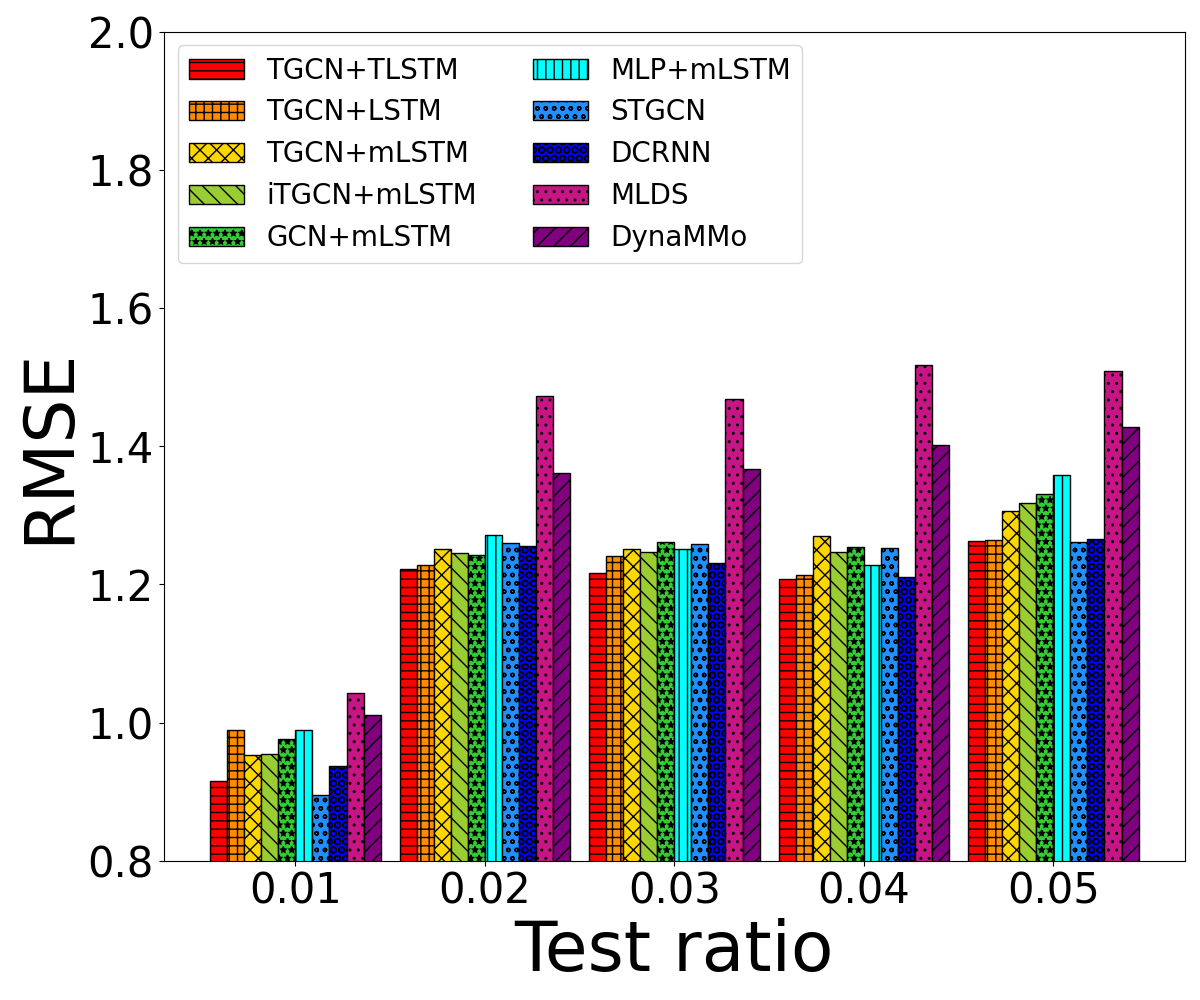}\label{fig:exp_revenue_future}}\,
\subfloat[Traffic-Future]
{\includegraphics[width=.20\linewidth]{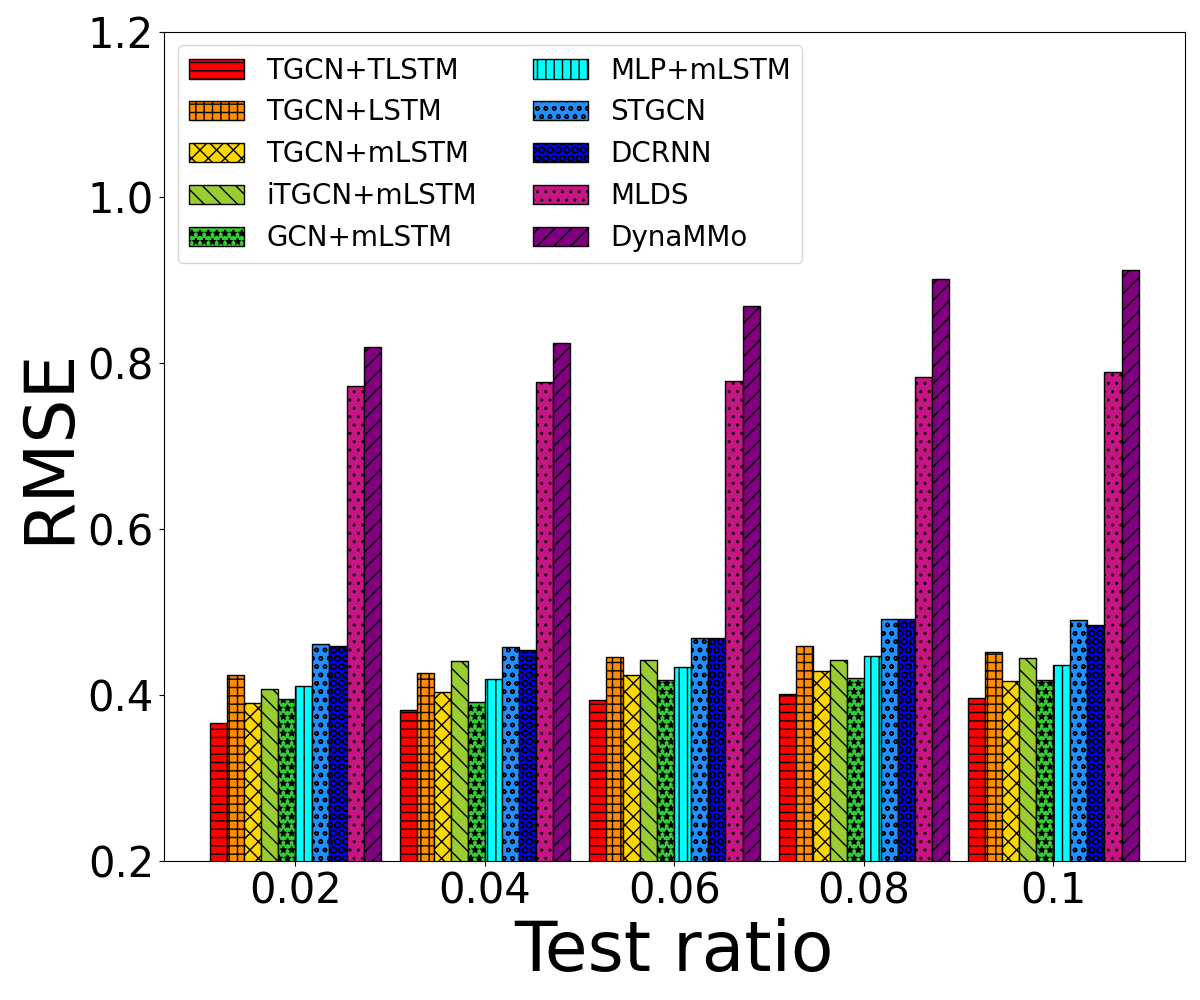}\label{fig:exp_traffic_future}}\,
\caption{RMSE of missing value recovery (upper) and future value prediction (lower).}\label{fig:exp_rmse}
\end{figure*}

\begin{figure*}[t]
\centering
\subfloat[Motes-Missing]
{\includegraphics[width=.20\linewidth]{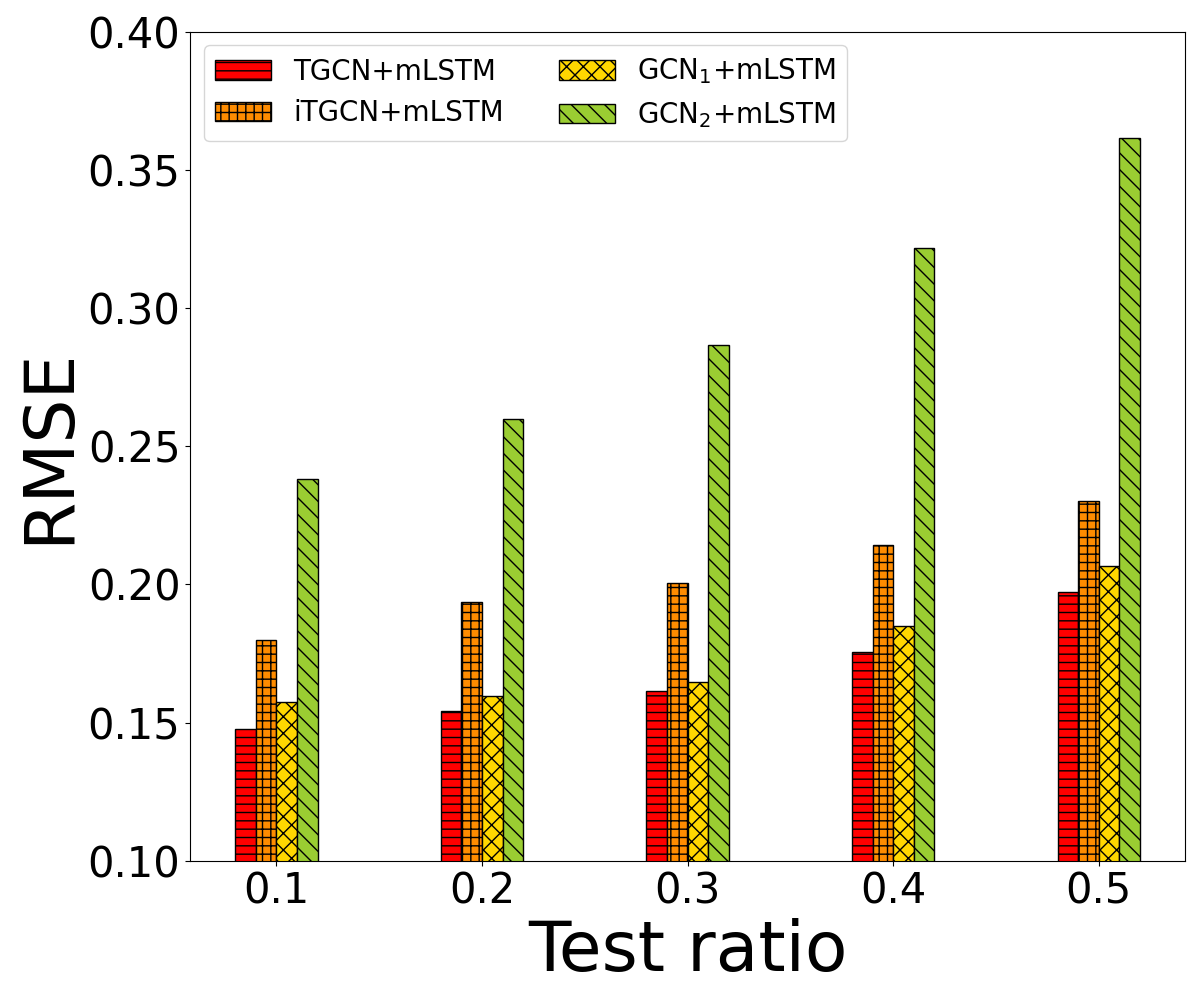}\label{fig:exp_motes_missing_synergy}}\,
\subfloat[Soil-Missing]
{\includegraphics[width=.20\linewidth]{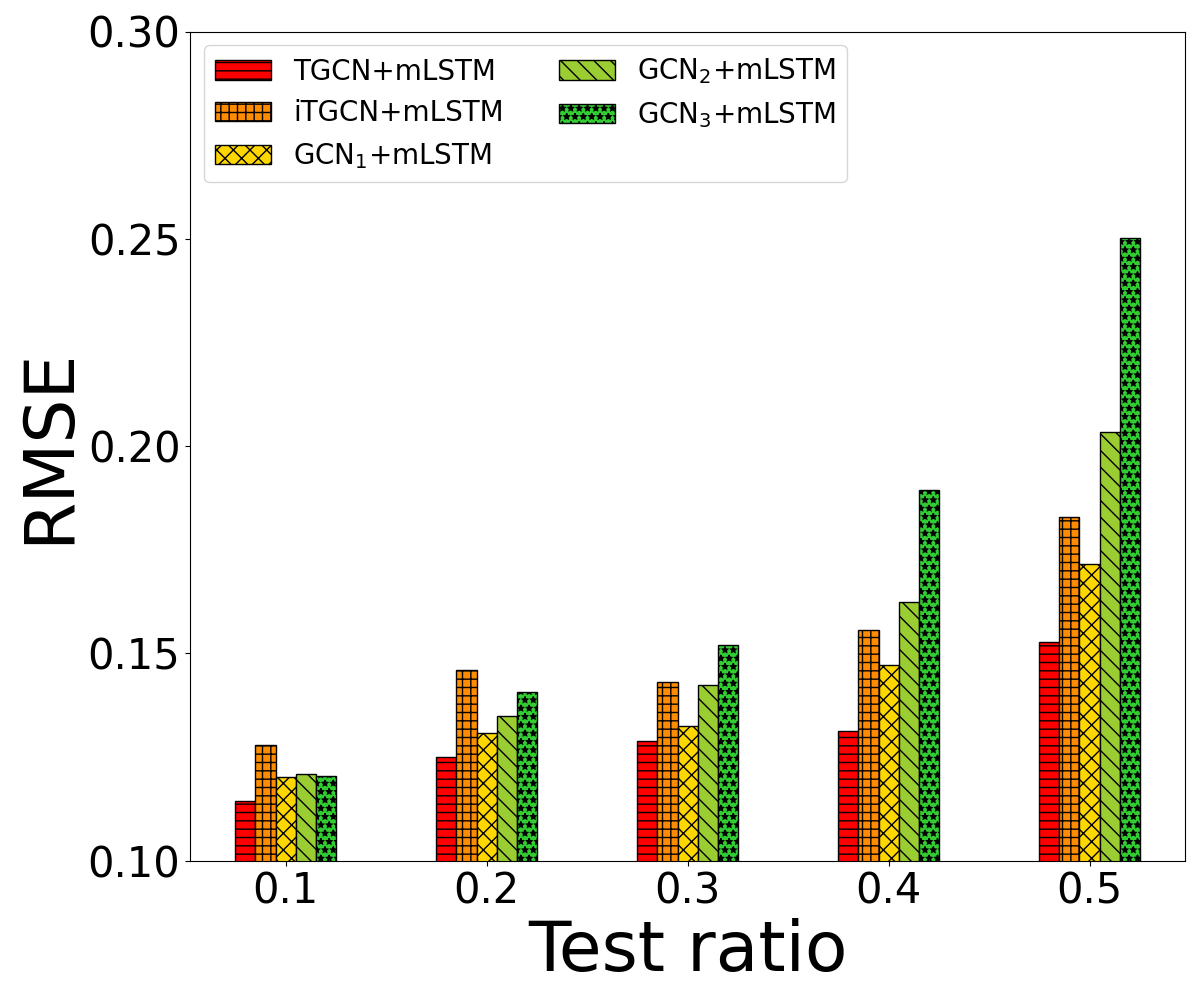}\label{fig:exp_soil_missing_synergy}}\,
\subfloat[Revenue-Missing]
{\includegraphics[width=.20\linewidth]{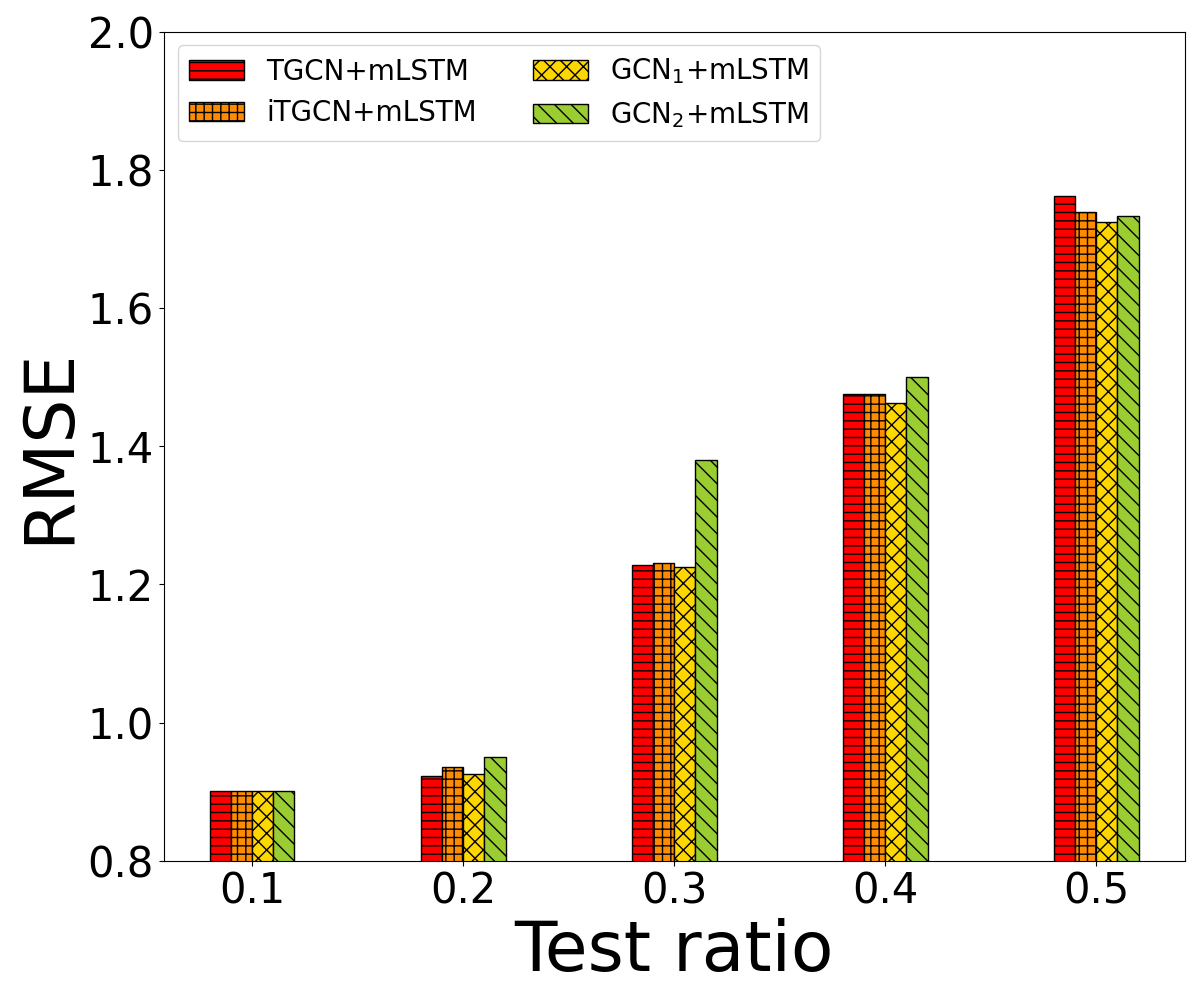}}\,
\subfloat[Traffic-Missing]
{\includegraphics[width=.20\linewidth]{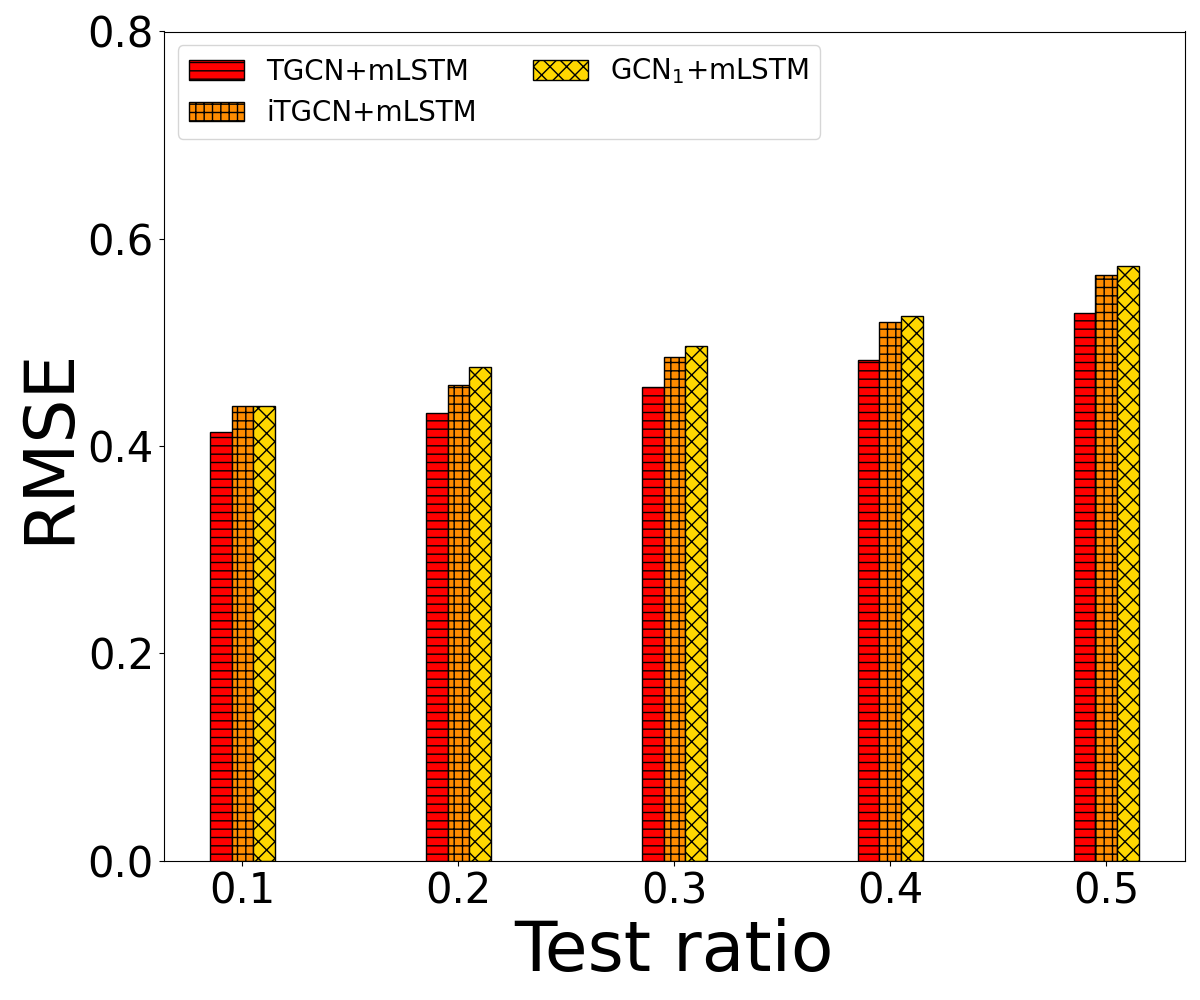}}\,
\subfloat[Motes-Future]
{\includegraphics[width=.20\linewidth]{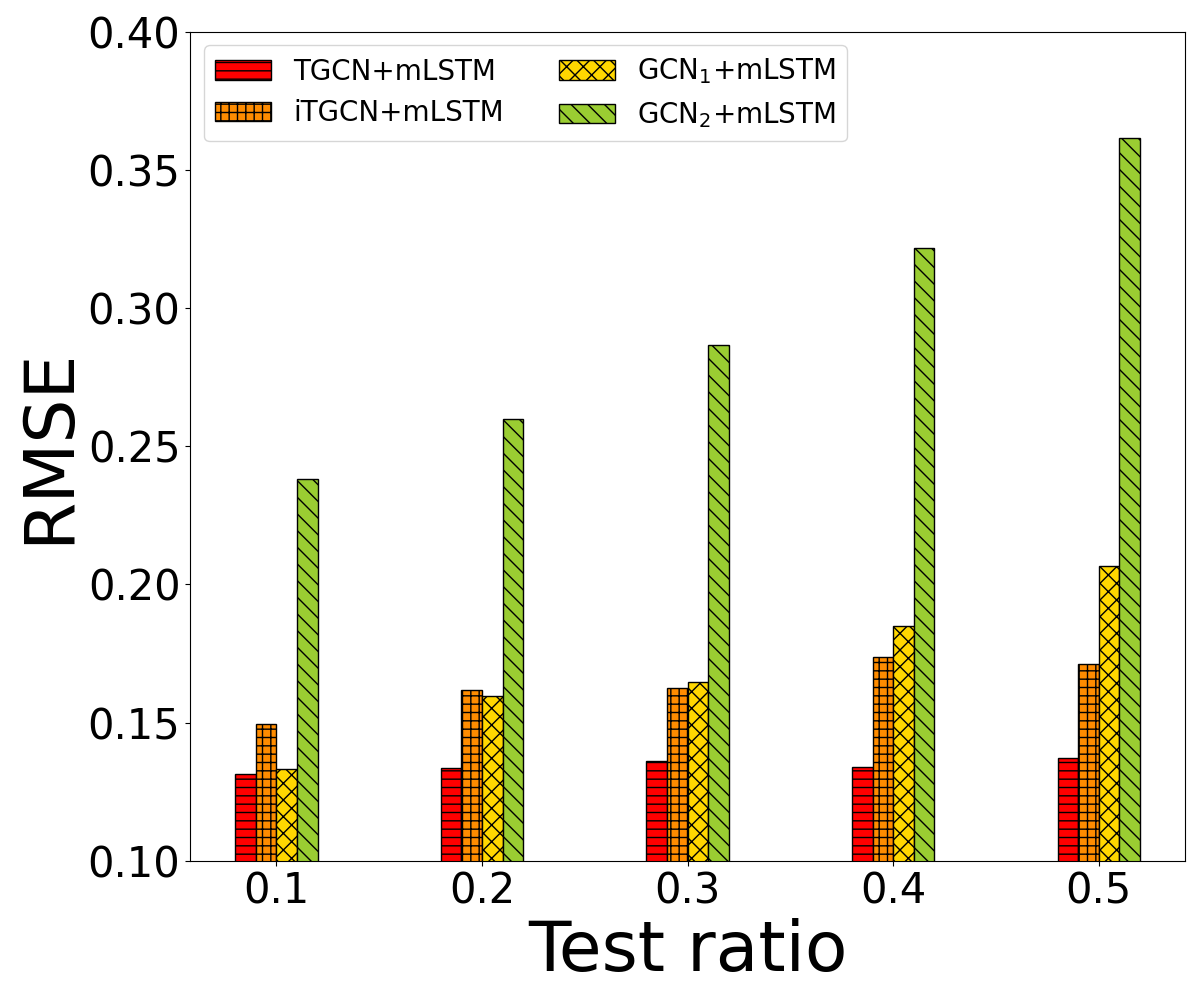}\label{fig:exp_motes_future_synergy}}\,
\subfloat[Soil-Future]
{\includegraphics[width=.20\linewidth]{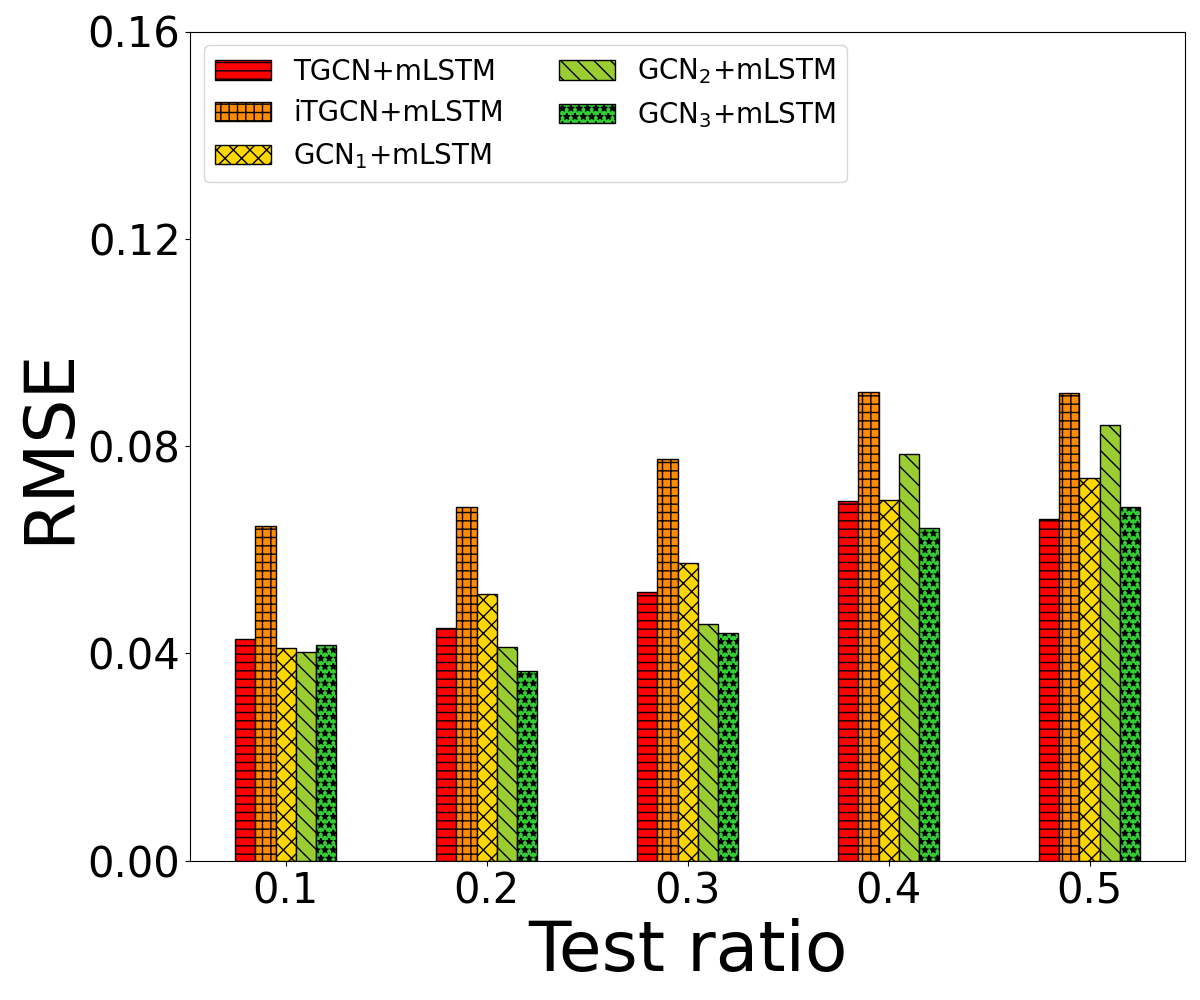}\label{fig:exp_soil_future_synergy}}
\subfloat[Revenue-Future]
{\includegraphics[width=.20\linewidth]{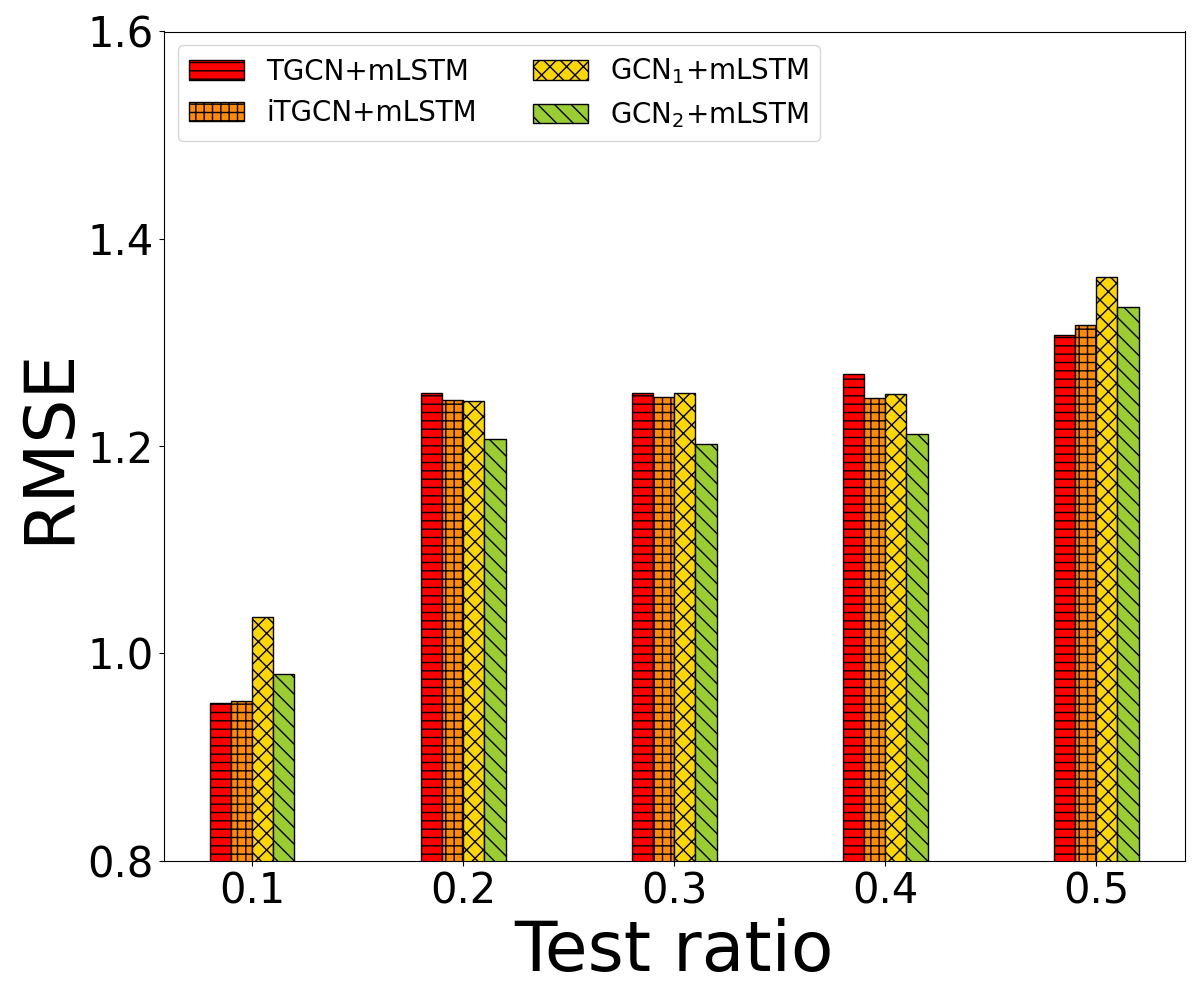}\label{fig:exp_revenue_future_synergy}}\,
\subfloat[Traffic-Future]
{\includegraphics[width=.20\linewidth]{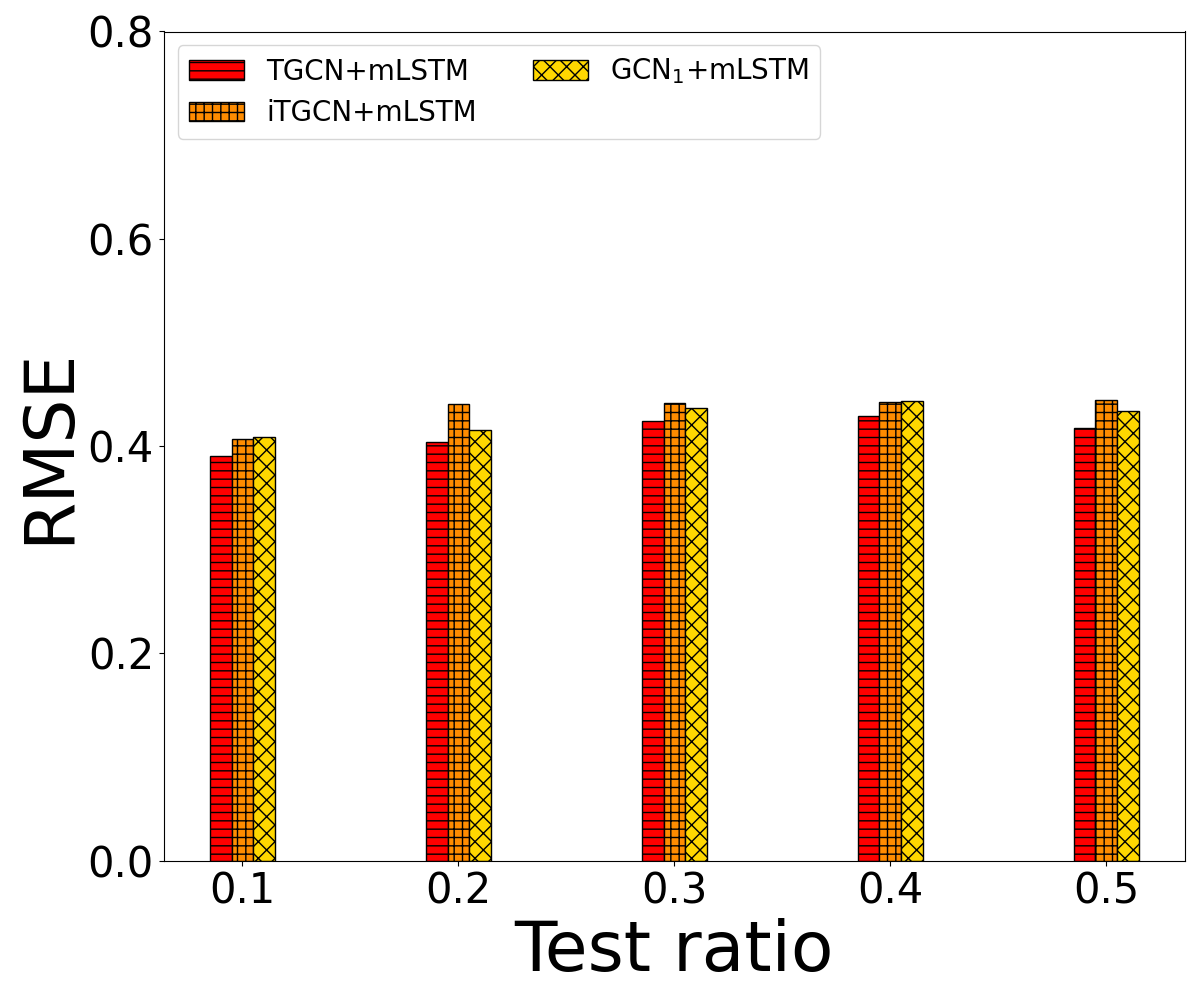}\label{fig:exp_traffic_future_synergy}}\,
\caption{Synergy Analysis: RMSE of missing value recovery (upper) and future value prediction (lower).}\label{fig:exp_rmse_synergy}
\end{figure*}

\begin{figure*}[t]
\centering
\subfloat[Missing Value Recovery]
{\includegraphics[width=.20\linewidth]{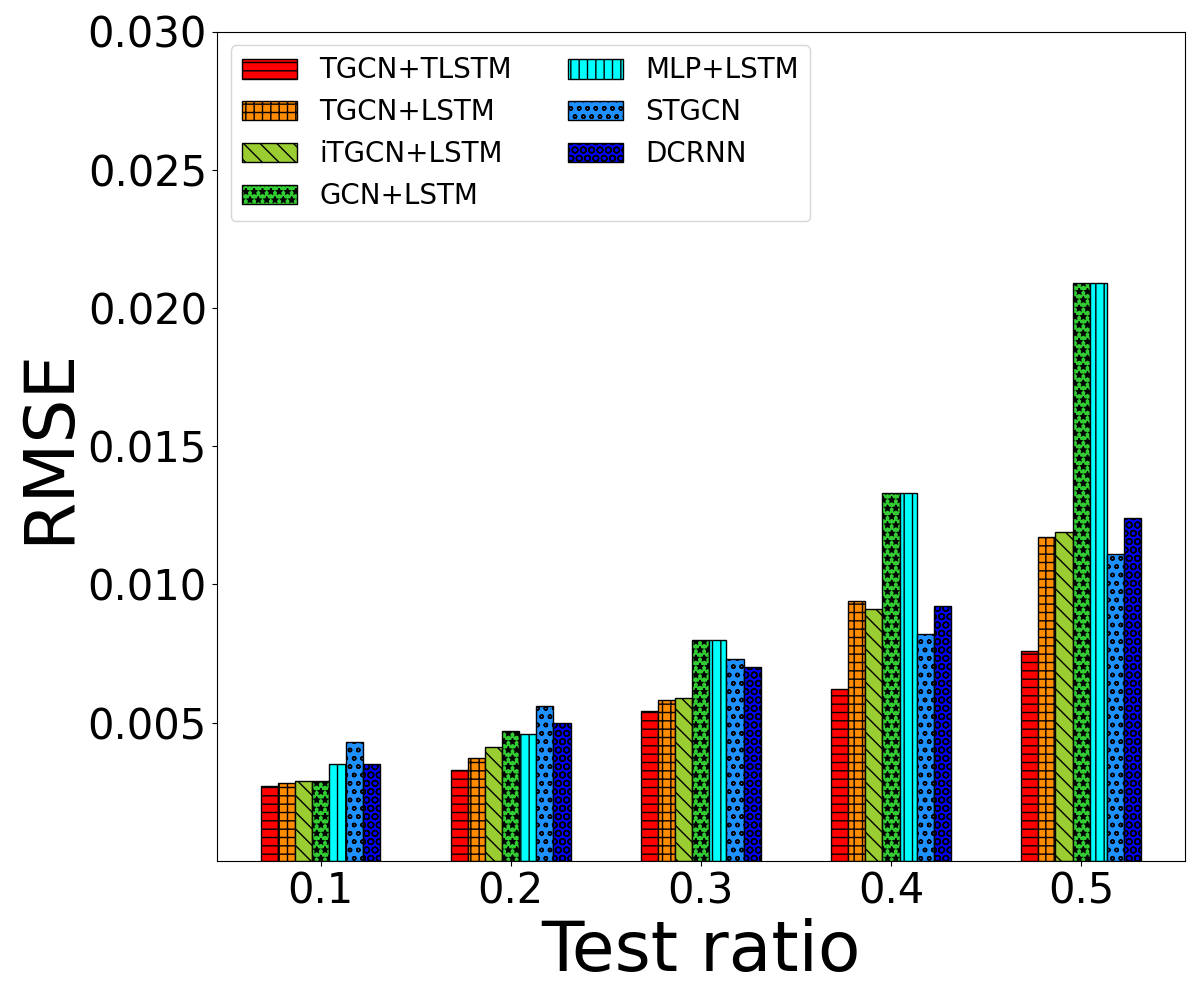}\label{fig:exp_20cr_missing}}\,
\subfloat[Future Value Prediction]
{\includegraphics[width=.20\linewidth]{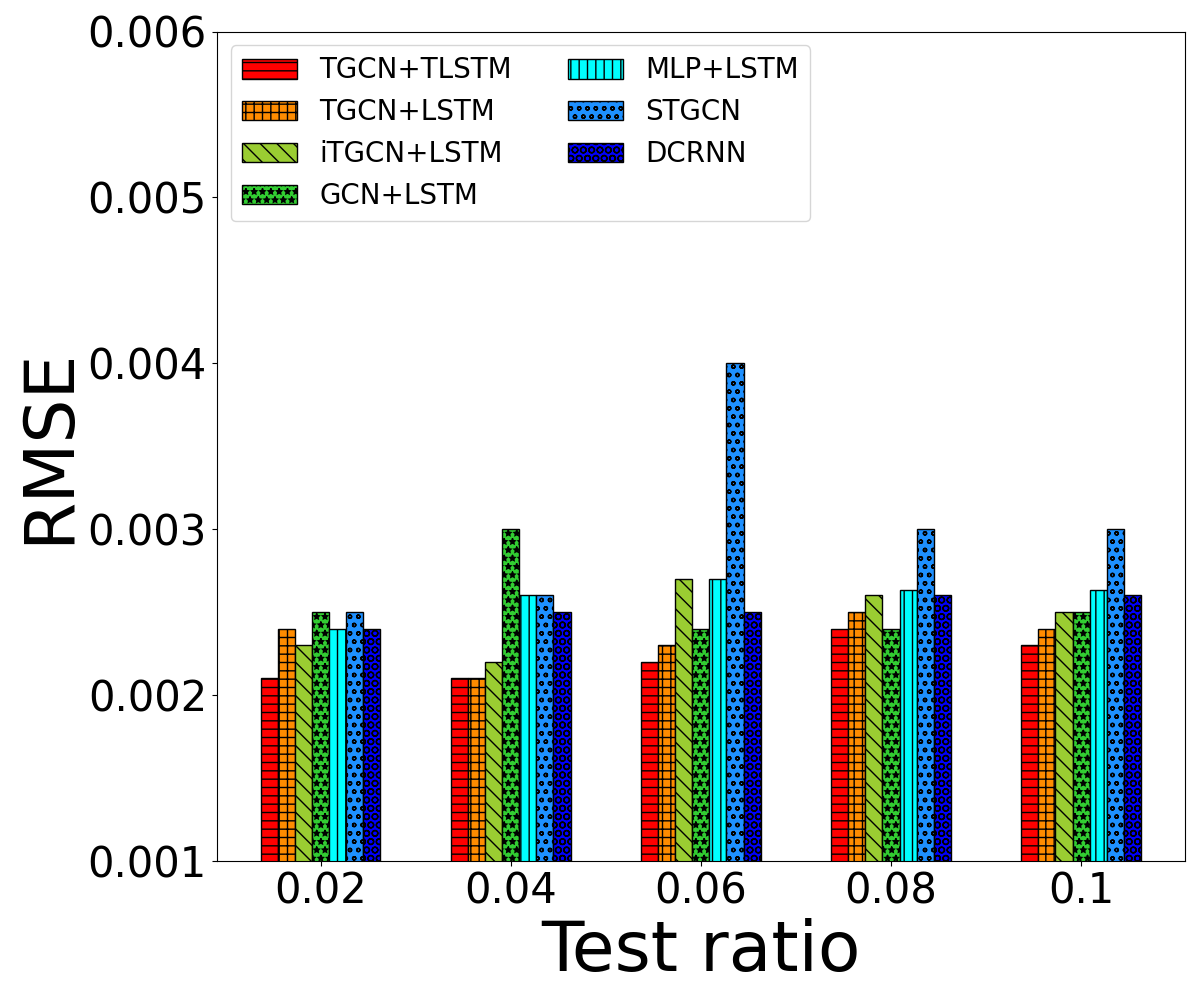}\label{fig:exp_20cr_future}}\,
\subfloat[Synergy-Missing]
{\includegraphics[width=.20\linewidth]{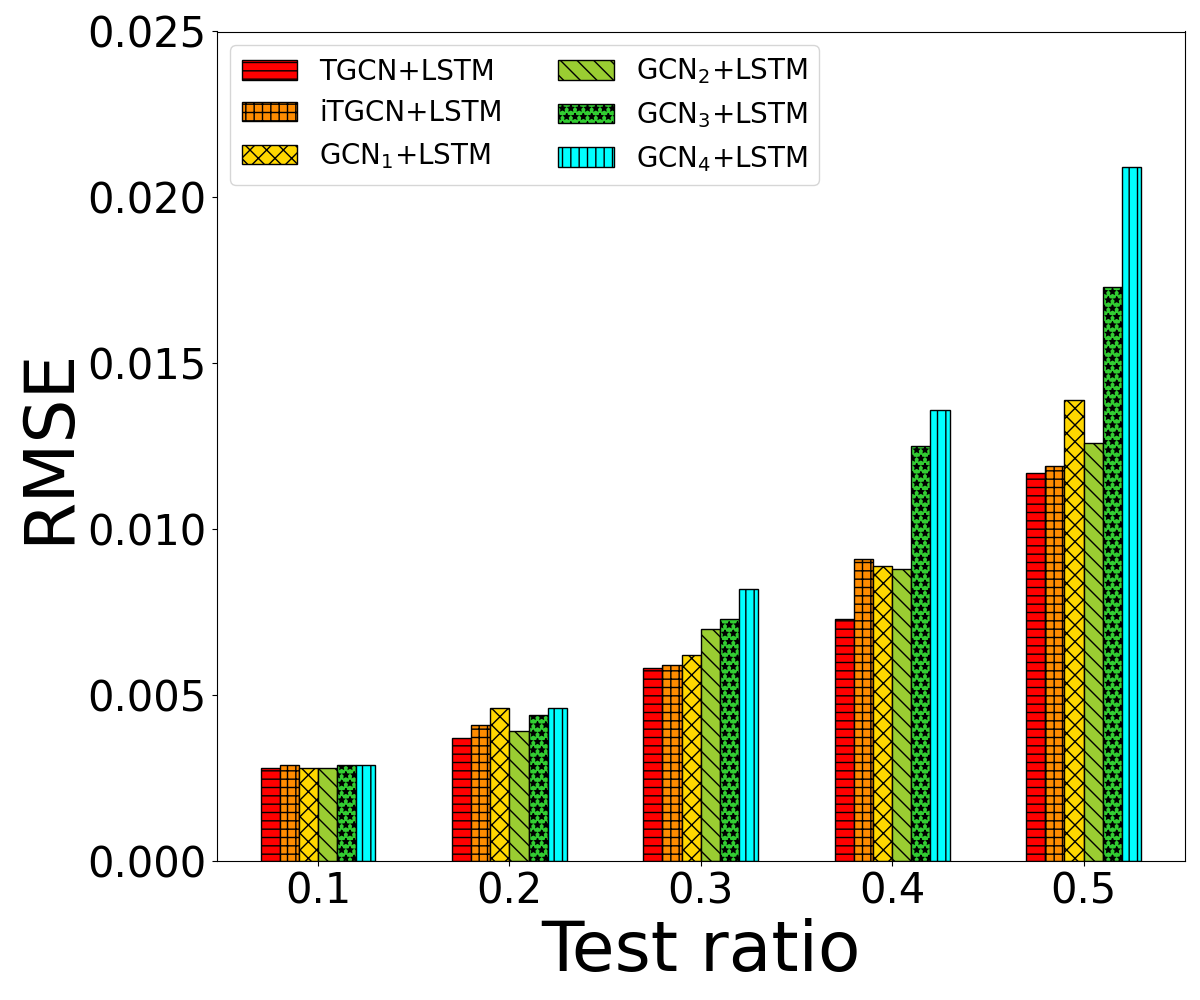}\label{fig:exp_20cr_missing_synergy}}\,
\subfloat[Synergy-Future]
{\includegraphics[width=.20\linewidth]{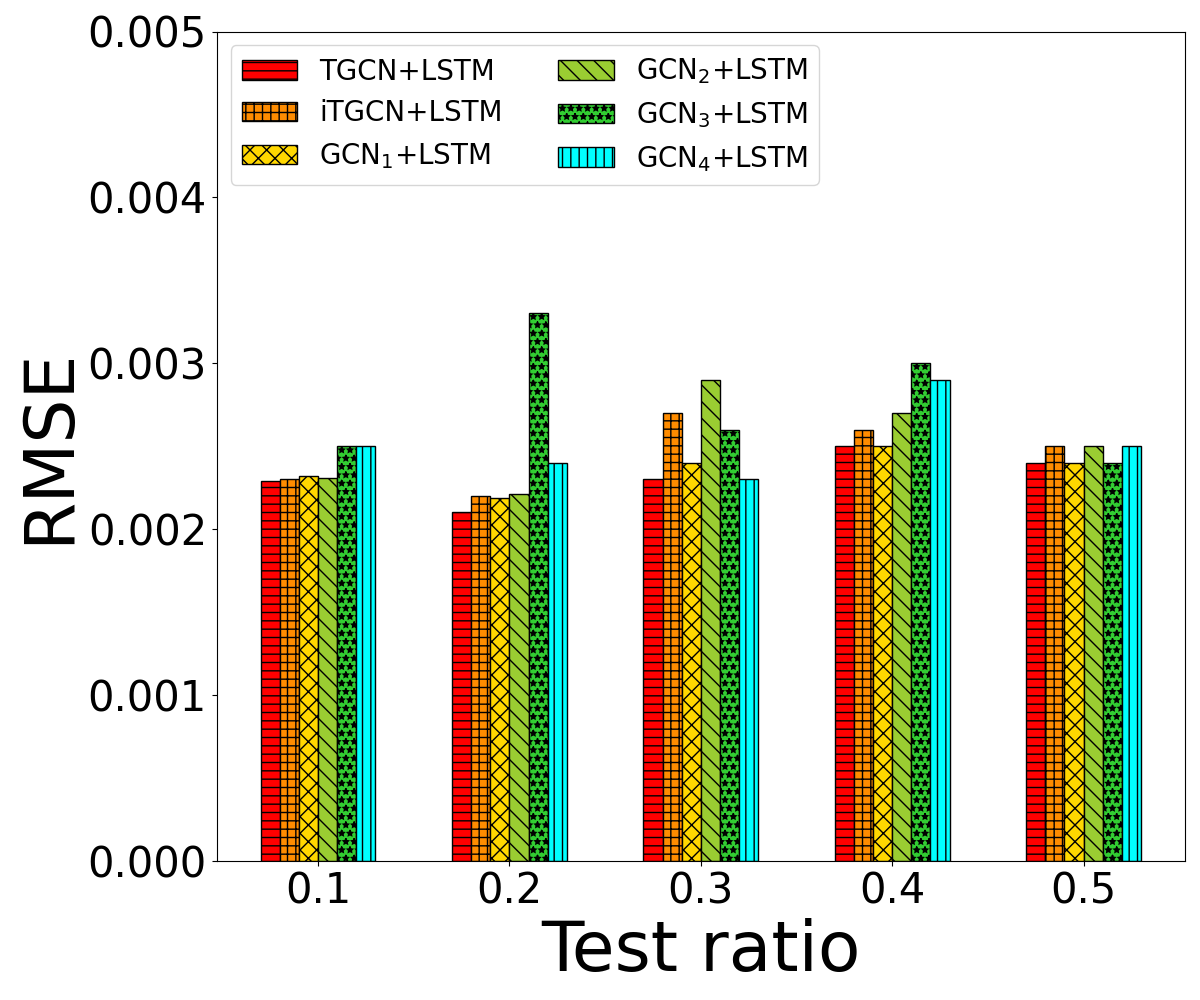}\label{fig:exp_20cr_future_synergy}}
\caption{Experiments on the \textit{20CR} dataset}\label{fig:exp_20cr}
\end{figure*}

\begin{figure*}[h]
\centering
\subfloat
{\includegraphics[width=.32\linewidth]{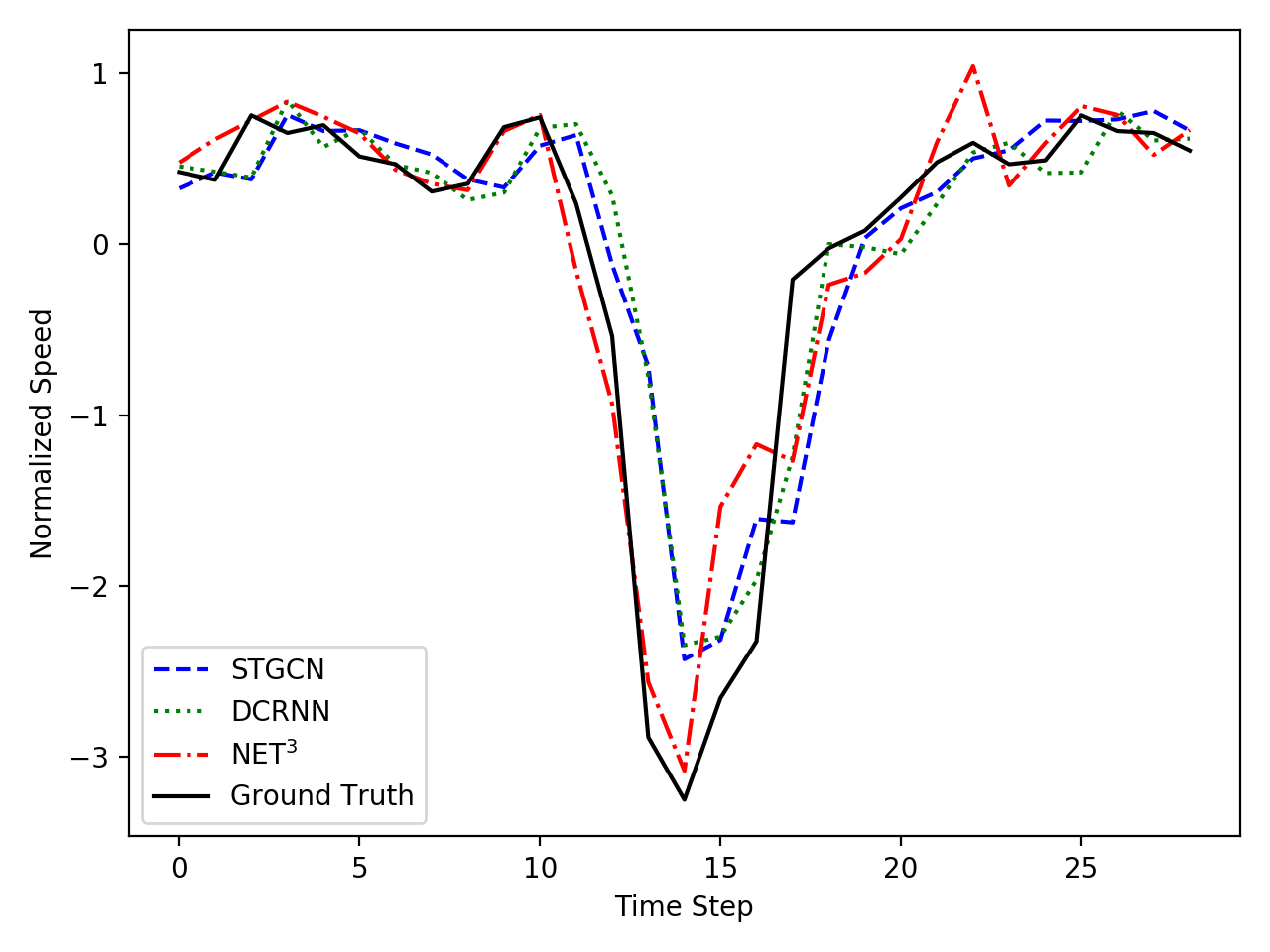}}\,
\subfloat
{\includegraphics[width=.32\linewidth]{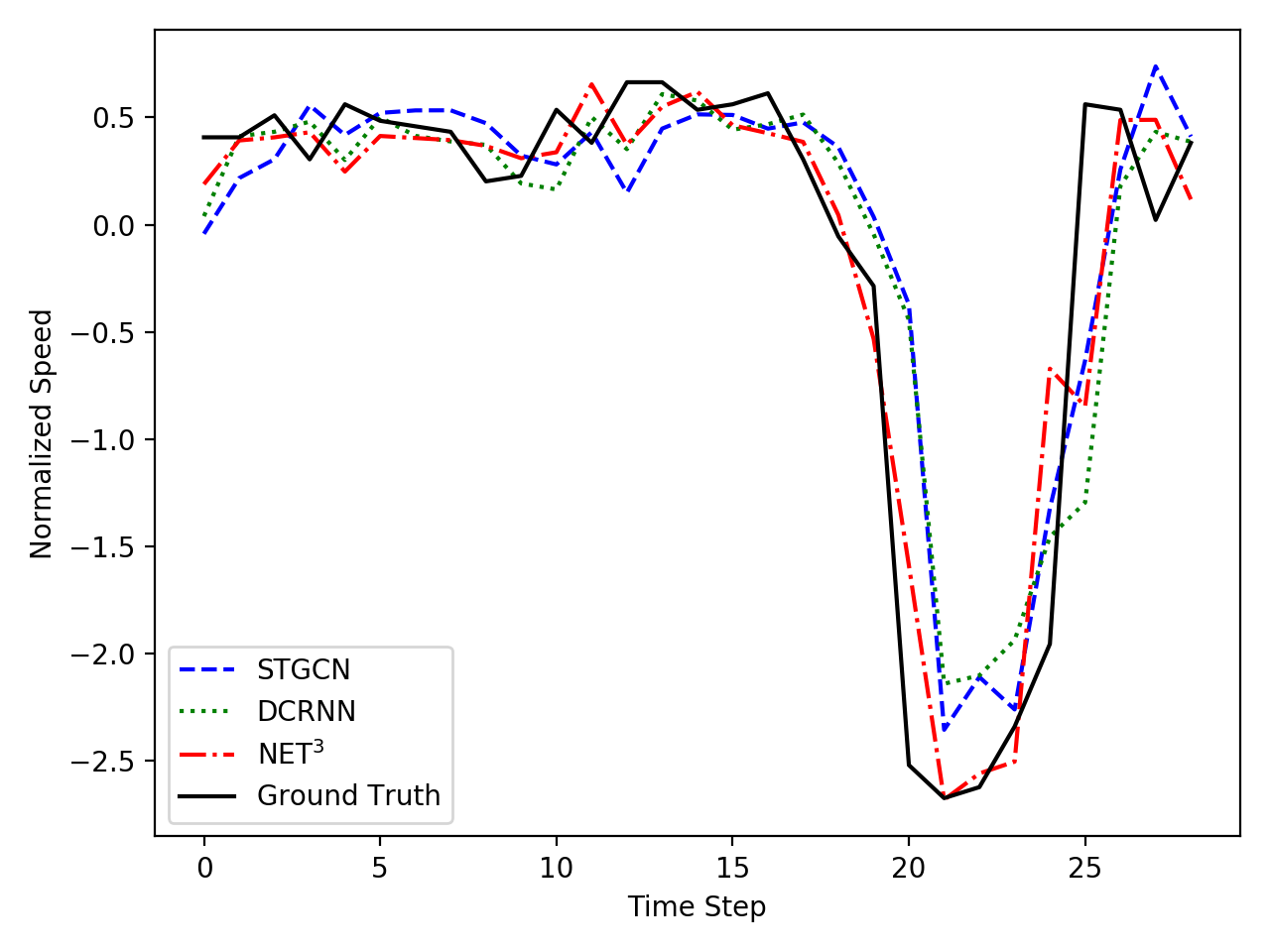}}\,
\subfloat
{\includegraphics[width=.32\linewidth]{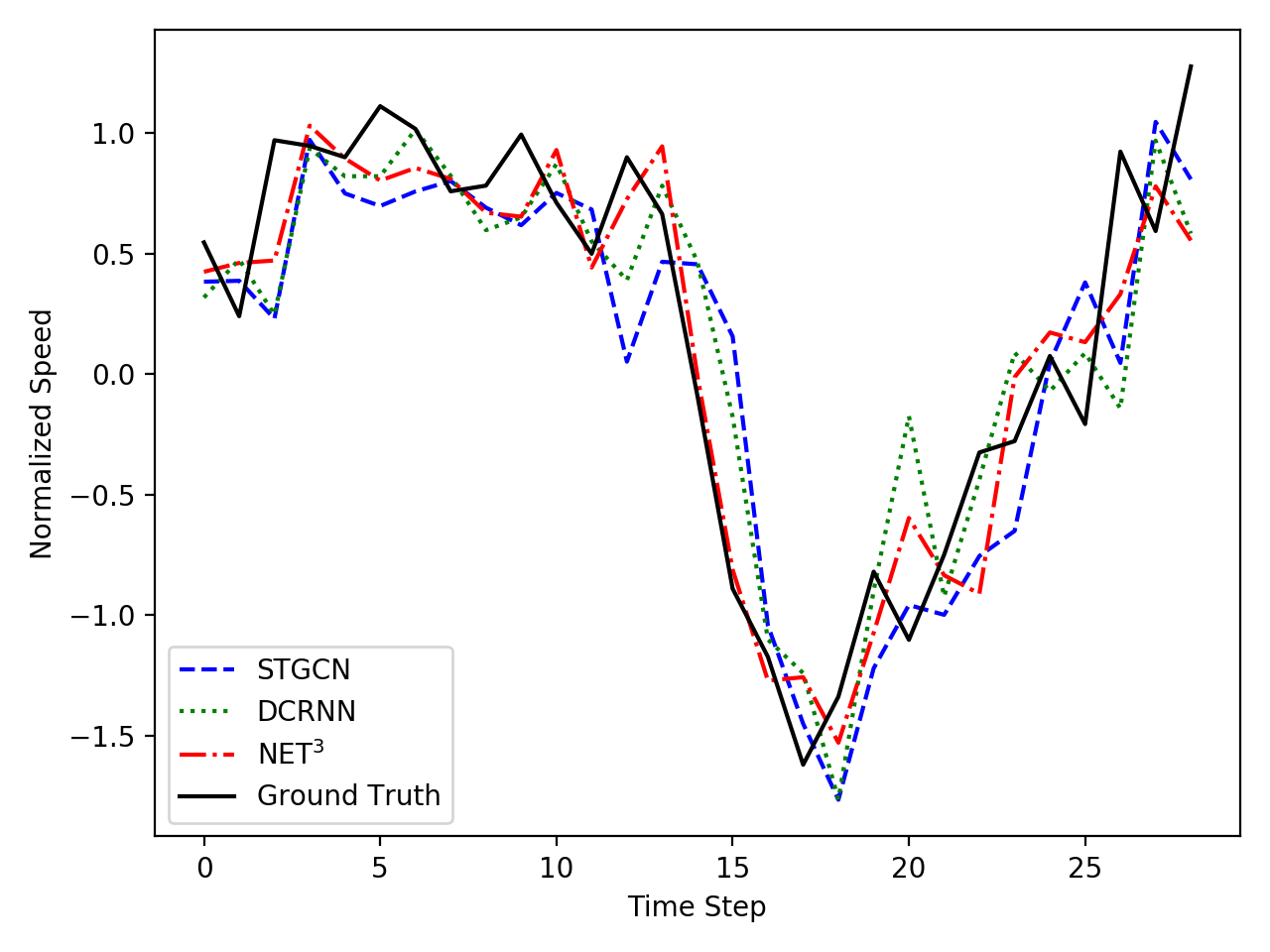}}\,
\subfloat
{\includegraphics[width=.32\linewidth]{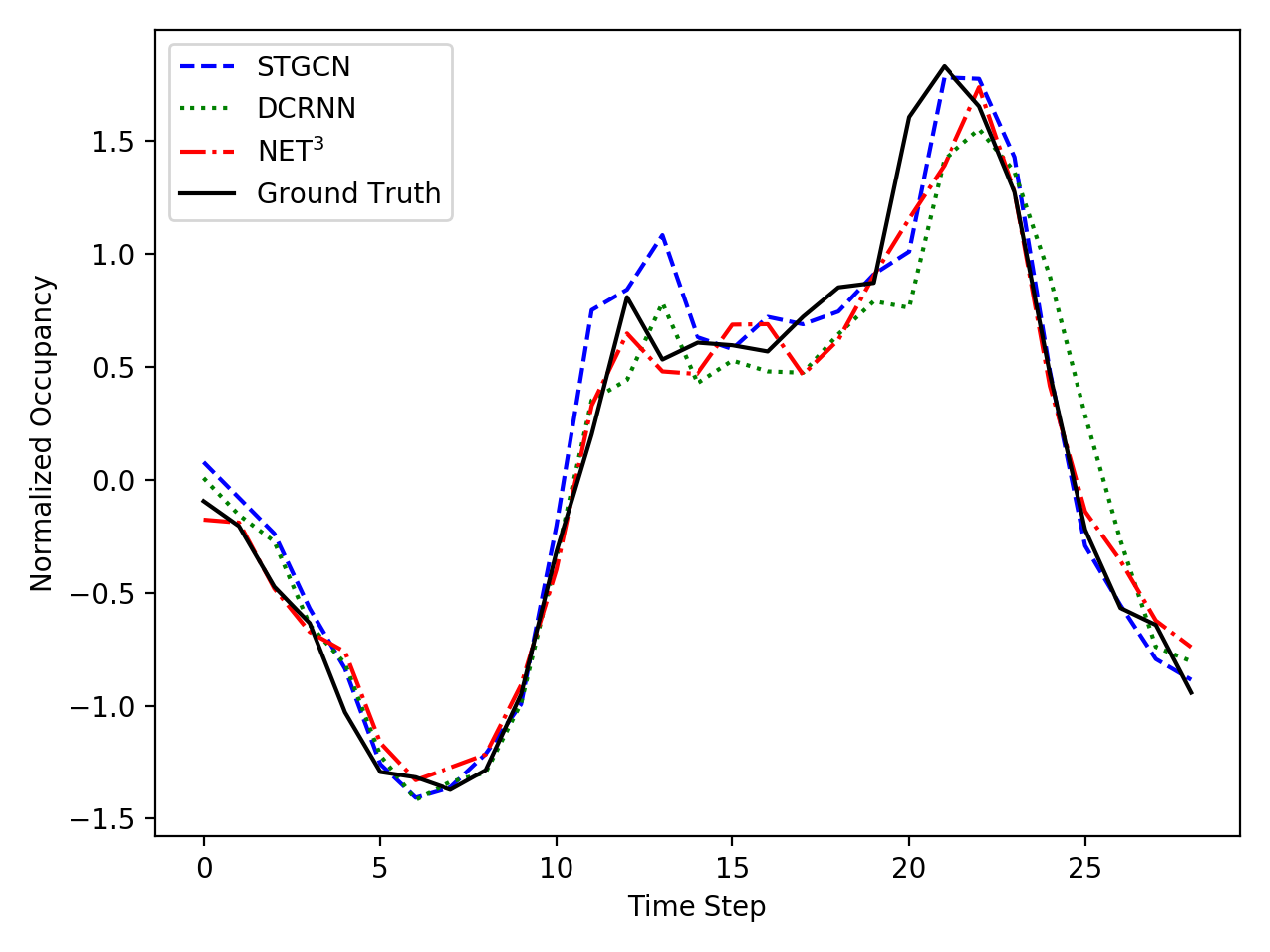}}\,
\subfloat
{\includegraphics[width=.32\linewidth]{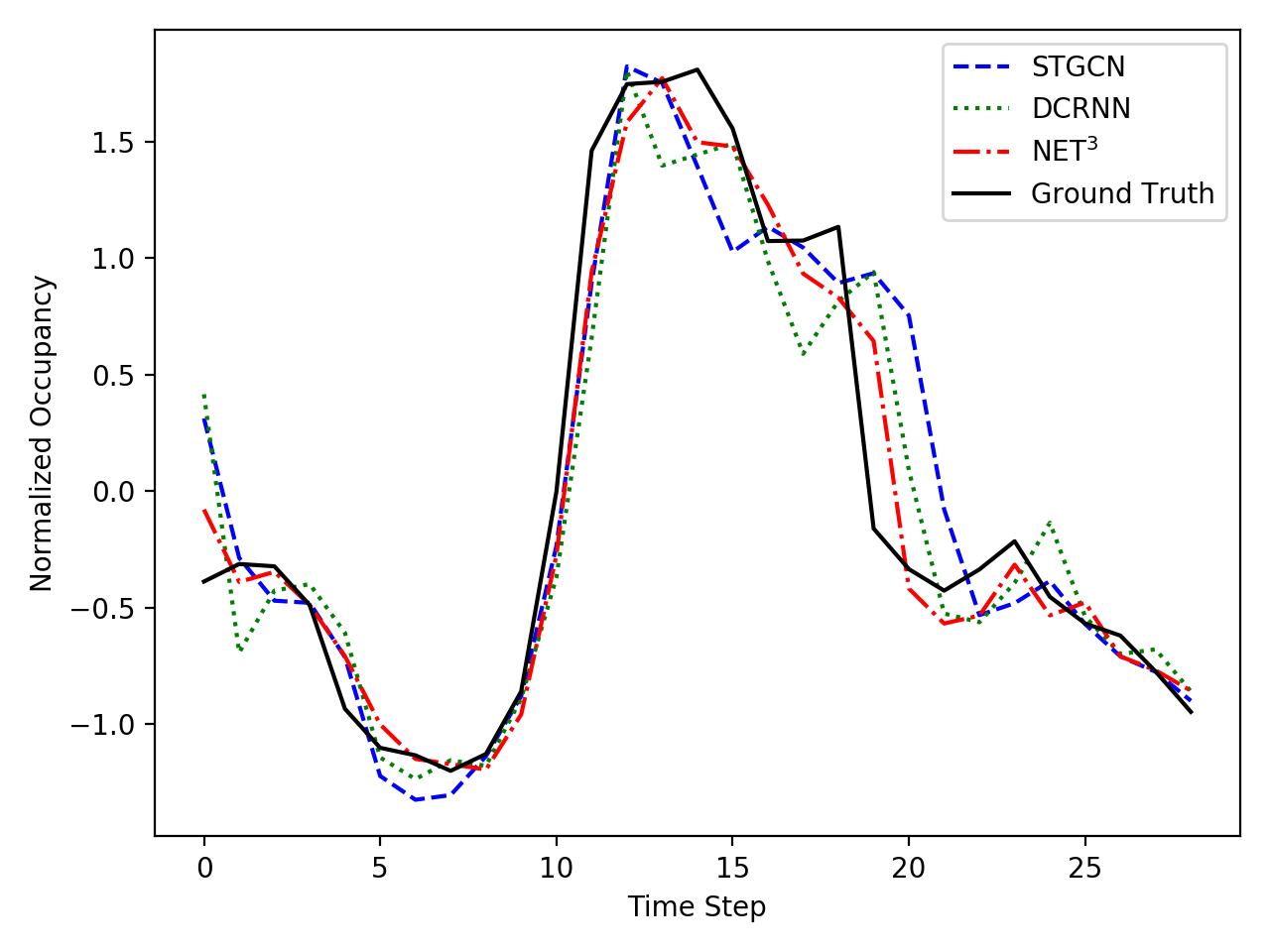}}\,
\subfloat
{\includegraphics[width=.32\linewidth]{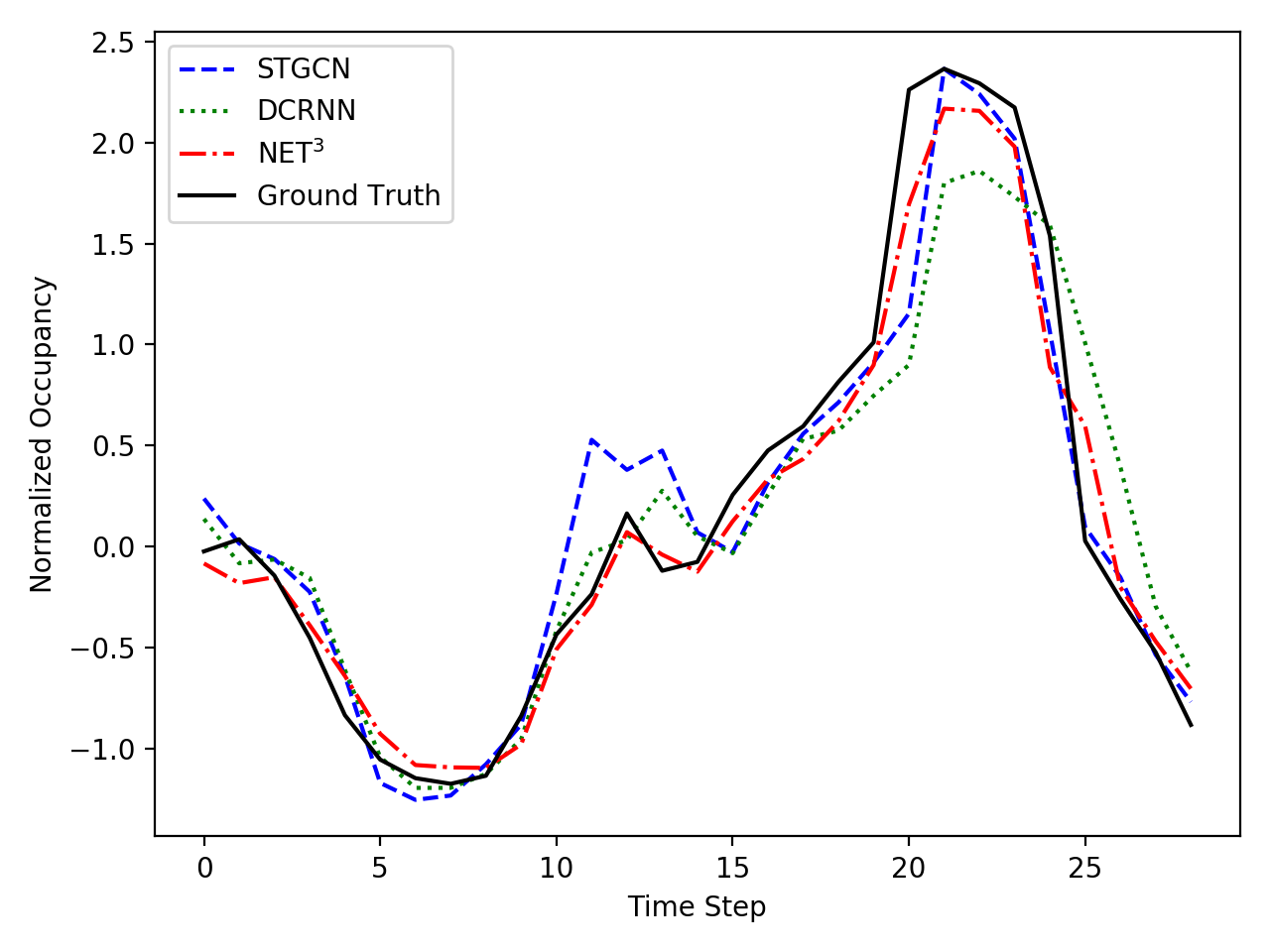}}
\caption{Visualization of future value prediction on the \textit{Traffic} dataset. Upper part presents the results for normalized speed. 
Lower part presents the results for normalized occupancy.}\label{fig:exp_visual}
\end{figure*}

\subsubsection{Implementation Details}
For all the datasets and tasks, we use one layer TGCN, one layer TLSTM, and one layer MLP with the linear activation.
The hidden dimension is fixed as $8$.
We fix $\rho=0.8$, $0.8$, $0.2$, $0.1$ and $0.9$ for TLSTM on \textit{Motes}, \textit{Soil}, \textit{Revenue}, \textit{Traffic}, and \textit{20CR} datasets respectively.
The window size is set as $\omega=5$ and $\tau=1$, and Adam optimizer \cite{kingma2014adam} with a learning rate of $0.01$ is adopted.
Coefficients $\mu_1$ and $\mu_2$ are fixed as $10^{-3}$.

\subsection{Effectiveness Results}\label{sec:exp_effectiveness}
In this section, we present the effectiveness experimental results for missing value recovery, future value prediction, synergy analysis and sensitivity experiments.
\subsubsection{Missing Value Recovery}
For all the datasets, we randomly select 10\% to 50\% of the data points as test sets, and we use the mean and standard deviation of each time series in the training sets to normalize each time series.
The evaluation results on \textit{Motes}, \textit{Soil}, \textit{Revenue} and \textit{Traffic} are shown in Figure \ref{fig:exp_motes_missing}-\ref{fig:exp_traffic_missing}, and the results for \textit{20CR} are presented in \ref{fig:exp_20cr_missing}.
The proposed full model \nettt (TGCN+TLSTM) outperforms all of the baseline methods on almost all of the settings.
Among the baselines methods, those equipped with GCNs generally have a better performance than LSTM.
When comparing TGCN with iTGCN, we observe that TGCN performs better than iTGCN on most of the settings.
This is due to TGCN's capability in capturing various synergy among graphs.
We can also observe that TLSTM (TGCN+TLSTM) achieves lower RMSE than both mLSTM (TGCN+mLSTM) and LSTM (TGCN+LSTM), demonstrating the effectiveness of capturing the implicit relations.

\subsubsection{Future Value Prediction}
We use the last 2\% to 10\% time steps as test sets for the \textit{Motes}, \textit{Traffic}, \textit{Soil} and \textit{20CR} datasets, and we use the last 1\% to 5\% time steps as test sets for the \textit{Revenue} dataset.
% For the Motes, Traffic and Soil datasets, this range covers the data of the last 30 to 144 minutes, 1 to 6 days and 1 week to 1 month respectively.
% For the Revenue dataset, it represents the last 1 to 4 quarters.
Similar to the missing value recovery task, 
The datasets are normalized by mean and standard deviation of the training sets.
The evaluation results are shown in Figure \ref{fig:exp_motes_future}-\ref{fig:exp_traffic_future}
and Figure \ref{fig:exp_20cr_future}.
The proposed \nettt\ outperforms the baseline methods on all of the five datasets.
Different from the missing value recovery task, the classic methods perform much worse than deep learning methods on the future value prediction, which might result from the fact that these methods are unable to capture the non-linearity in the temporal dynamics.
Similar to the missing value recovery task, generally, TGCN also achives lower RMSE than iTGCN and GCN, and TLSTM performs better than both mLSTM and LSTM.
% The proposed full model (TGCN+TLSTM) also outperforms all of its ablated versions except on the smallest dataset (Motes) with a relatively high test ratio (from $0.03$ to $0.05$), where one of its  ablated versions (LSTM+TGCN) achieves slightly lower RMSE.

We present the visualization of the future value prediction task on the \textit{Traffic} dataset in Figure \ref{fig:exp_visual}.

\begin{figure*}[t!]
\centering
\subfloat[Missing value recovery]
{\includegraphics[width=.32\linewidth]{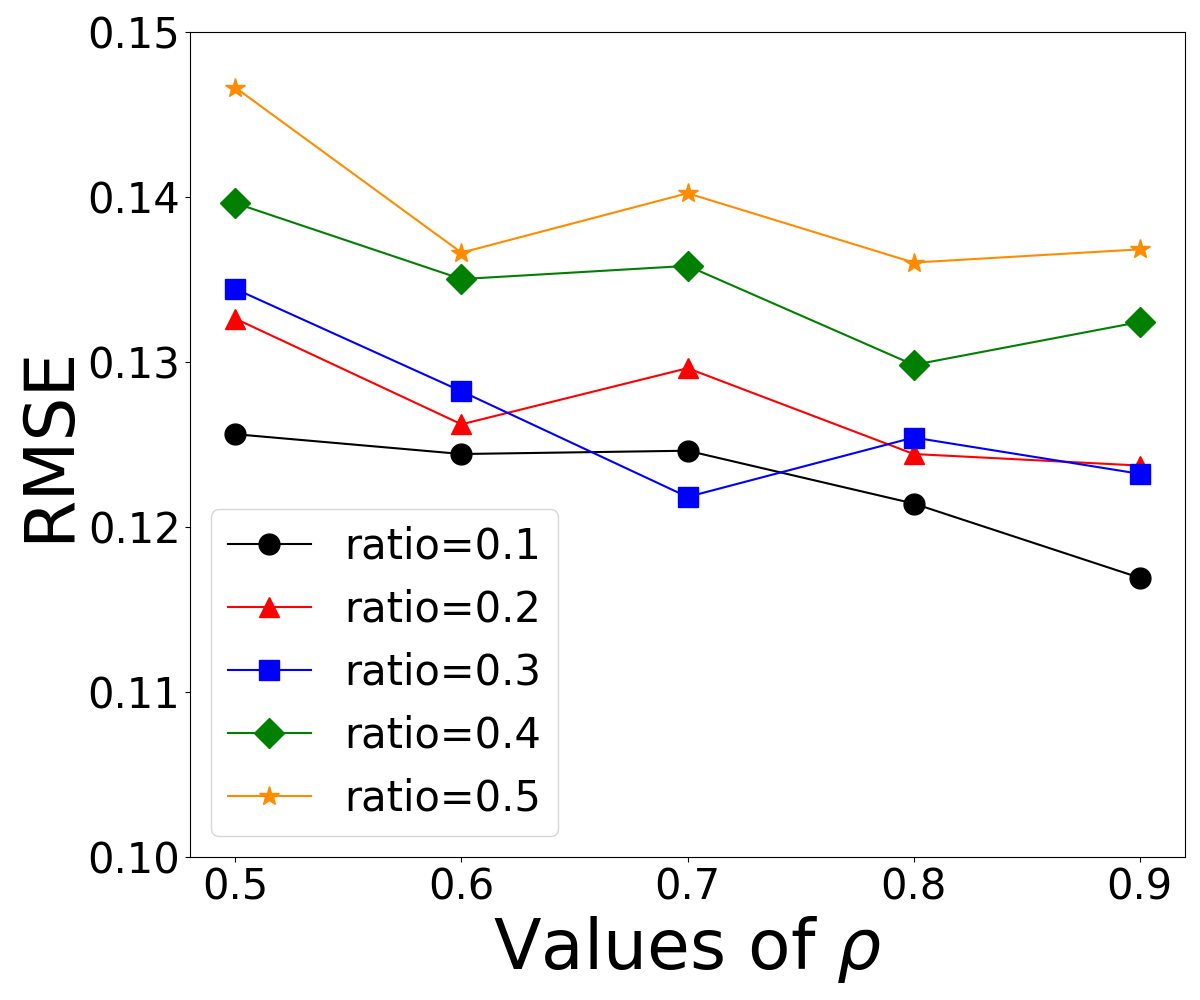}\label{fig:exp_sens_missing}}
\:
\subfloat[Future value prediction]
{\includegraphics[width=.32\linewidth]{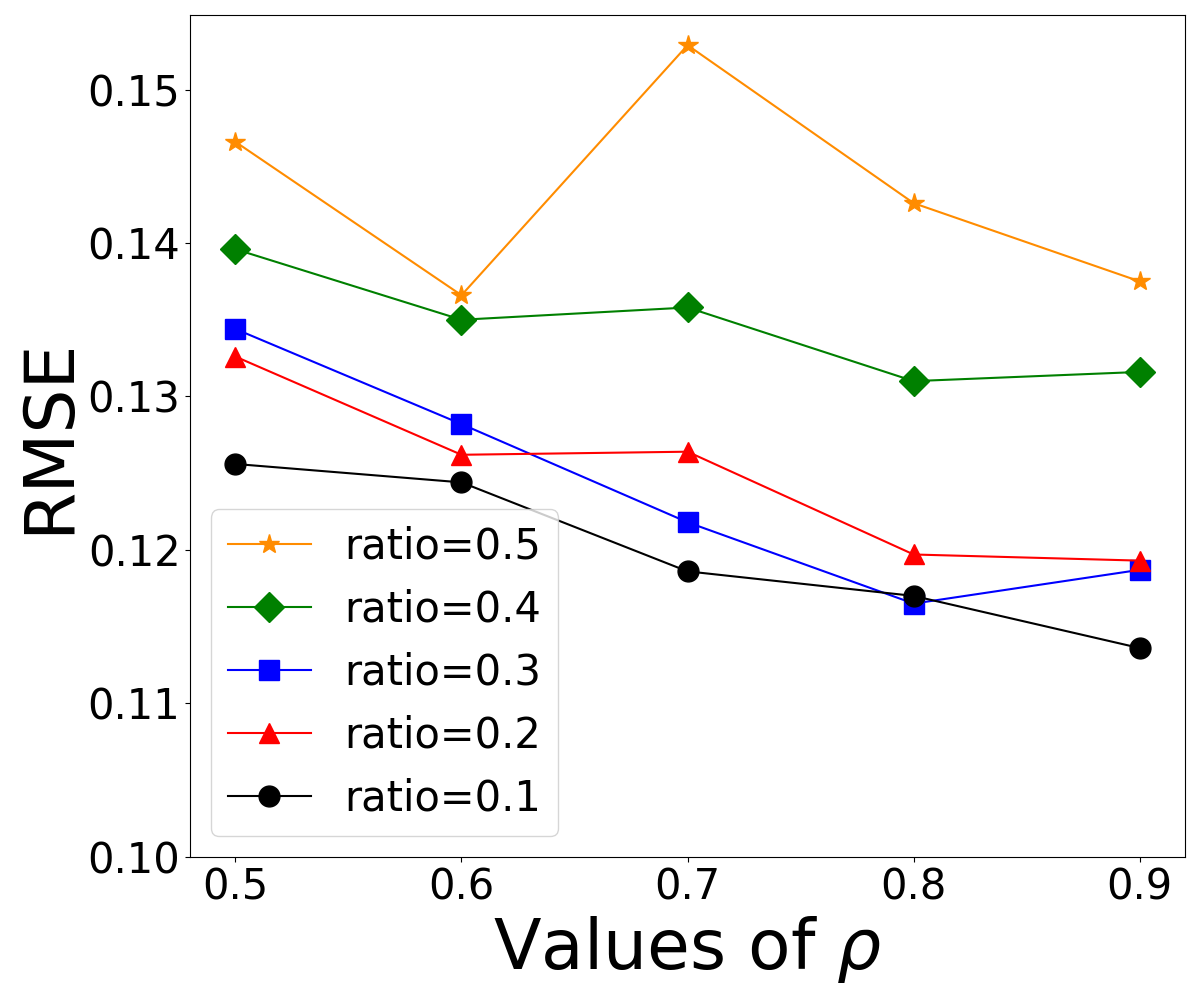}\label{fig:exp_sens_future}}
\:
\subfloat[Number of parameters]
{\includegraphics[width=.32\linewidth]{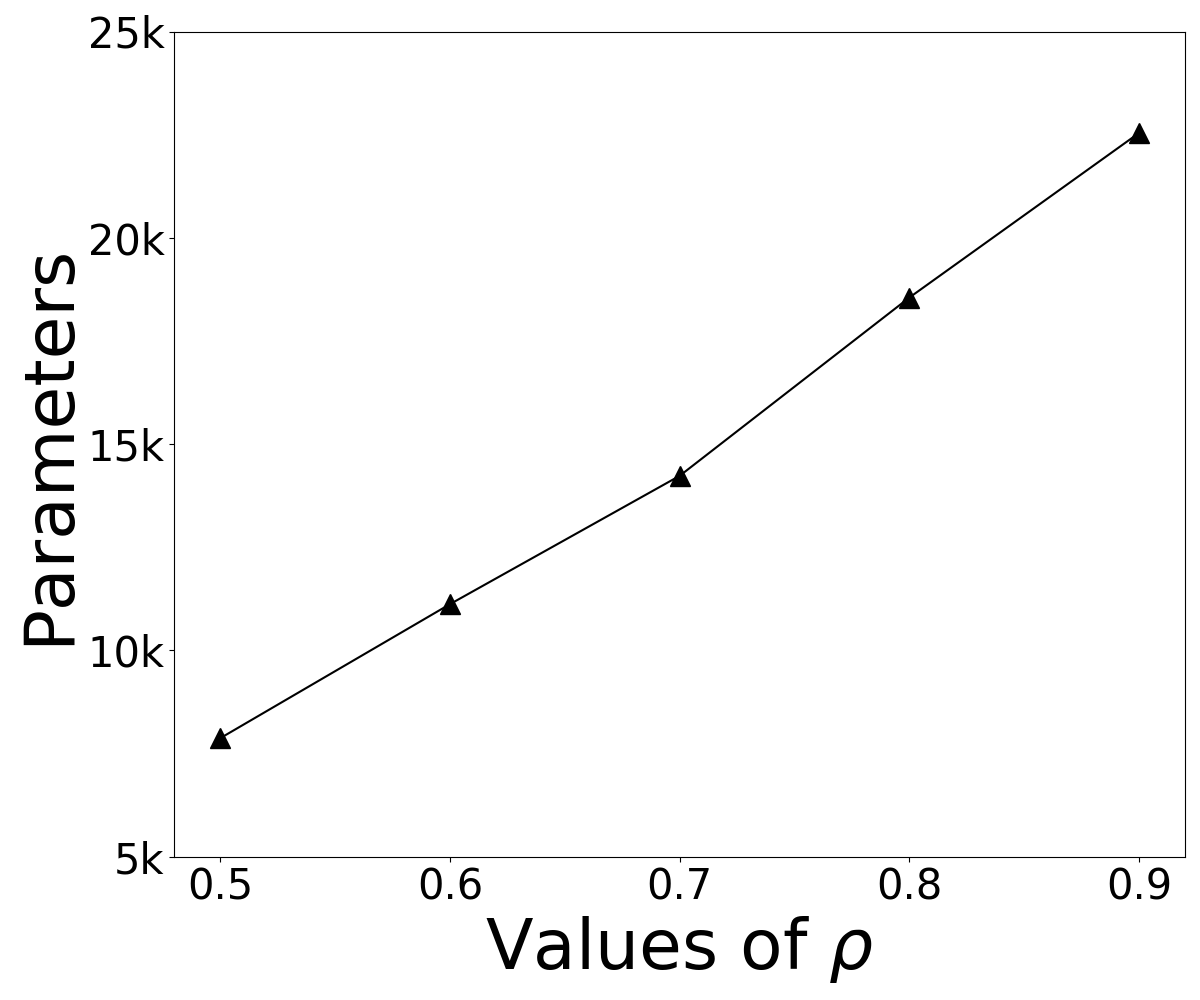}\label{fig:exp_sens_para}}
\caption{Sensitivity experiments of $\rho$ on the Motes dataset.}\label{fig:exp_sensitivity}
\end{figure*}

\subsubsection{Experiments on Synergy}
In this section, we compare the proposed TGCN with iTGCN, GCN$_1$, GCN$_2$, GCN$_3$ and GCN$_4$ (if applicable) on the missing value recovery and future value prediction tasks.
Here, GCN$_1$, GCN$_2$, GCN$_3$ and GCN$_4$ denote the GCN with the adjacency matrix of the 1st, 2nd, 3rd and 4th mode respectively.
iTGCN is an independent version of TGCN (Equation \eqref{eq:itgcn}), which is a simple linear combination of different GCNs (GCN$_1$, GCN$_2$, GCN$_3$ and GCN$_4$).
As shown in Figure \ref{fig:exp_rmse_synergy} and Figure \ref{fig:exp_20cr_missing_synergy}-\ref{fig:exp_20cr_future_synergy}, generally, TGCN outperforms GCNs designed for single modes and the simple combination of them (iTGCN).

\subsubsection{Sensitivity Experiments}
We use different values of $\rho$ for TLSTM on the \textit{Motes} dataset for the missing value recovery and future value prediction tasks and report their RMSE values in Figure \ref{fig:exp_sens_missing} and Figure \ref{fig:exp_sens_future}.
It can be noted that, in general, the greater $\rho$ is, the better results (i.e., smaller RMSE) will be obtained.
We believe the main reason is that a greater $\rho$ indicates that TLSTM captures more interaction between different time series. 
Figure \ref{fig:exp_sens_para} shows that the number of parameters of TLSTM is linear with respect to $\rho$.

\begin{table}[h]
\centering
\caption{In the upper part, $\rho_{max}$ and $\rho_{exp}$ are the upper bounds and the values of $\rho$ used in experiments. The middle and lower parts present the number of parameters in TLSTM and LSTM, and the parameter reduction ratio.}
    \begin{tabular}{c|r|r|r|r|r}
        \hline
        & \textit{Motes} & \textit{Soil} & \textit{Revenue} & \textit{Traffic} & \textit{20CR}\\
        \hline
        $\rho_{max}$ & 2.17 & 2.43 & 0.64 & 0.31 & 57.25\\
        $\rho_{exp}$ & 0.80 & 0.80 & 0.20 & 0.10 & 0.90\\
        \hline
        TLSTM & 18,552 & 10,996 & 87,967 & 180,554 & 16,696\\
        mLSTM & 117,504 & 57,120 & 669,120 & 1,088,000 & 58,752,000\\
        \hline
        Reduce & 84.21\% & 80.75\% & 86.85\% & 83.40\% & 99.97\%\\
        \hline
    \end{tabular}\label{table:rho}
\end{table}

\begin{figure}[h!]
\centering
\subfloat
{\includegraphics[width=.48\linewidth]{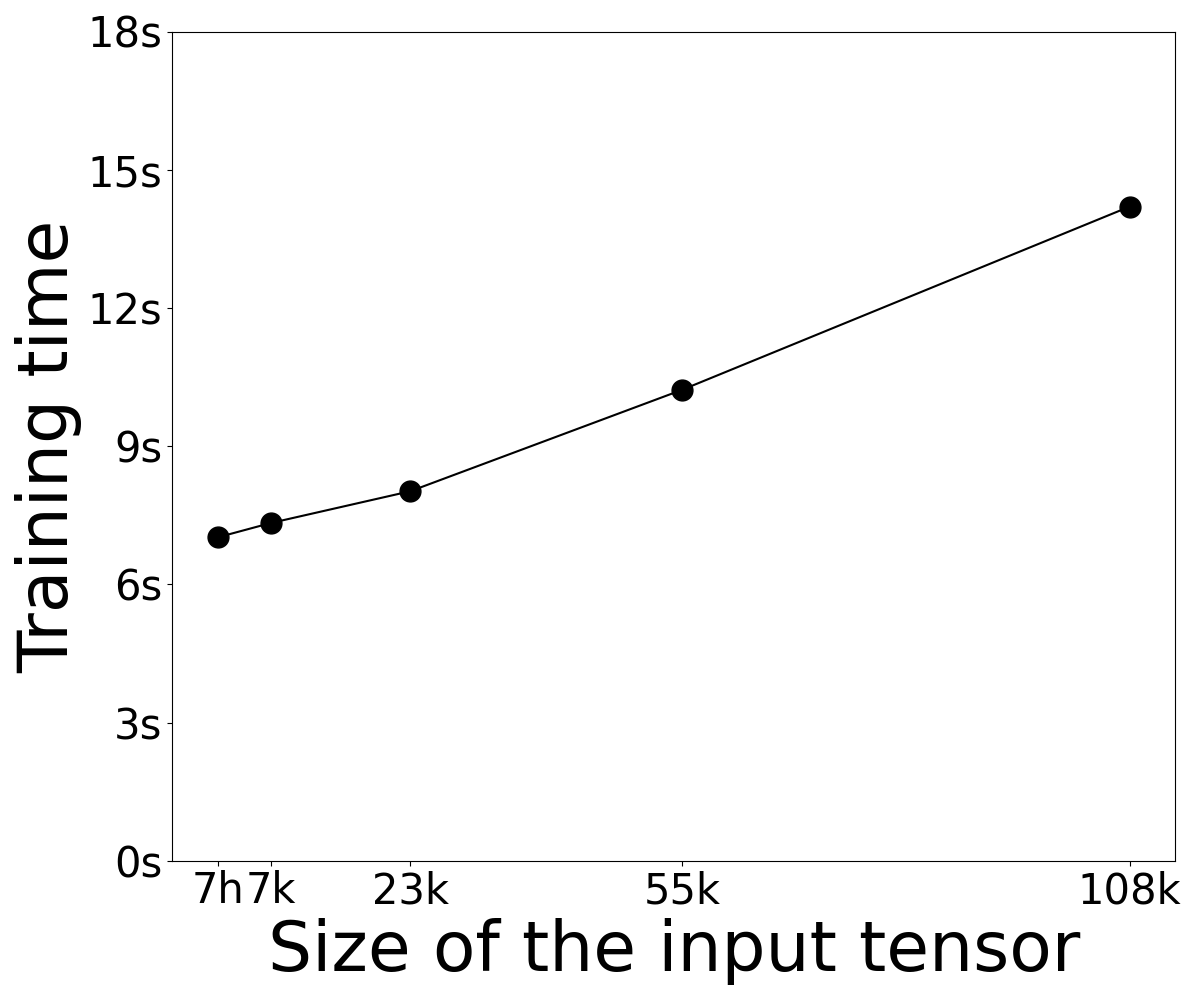}}\,
\subfloat
{\includegraphics[width=.48\linewidth]{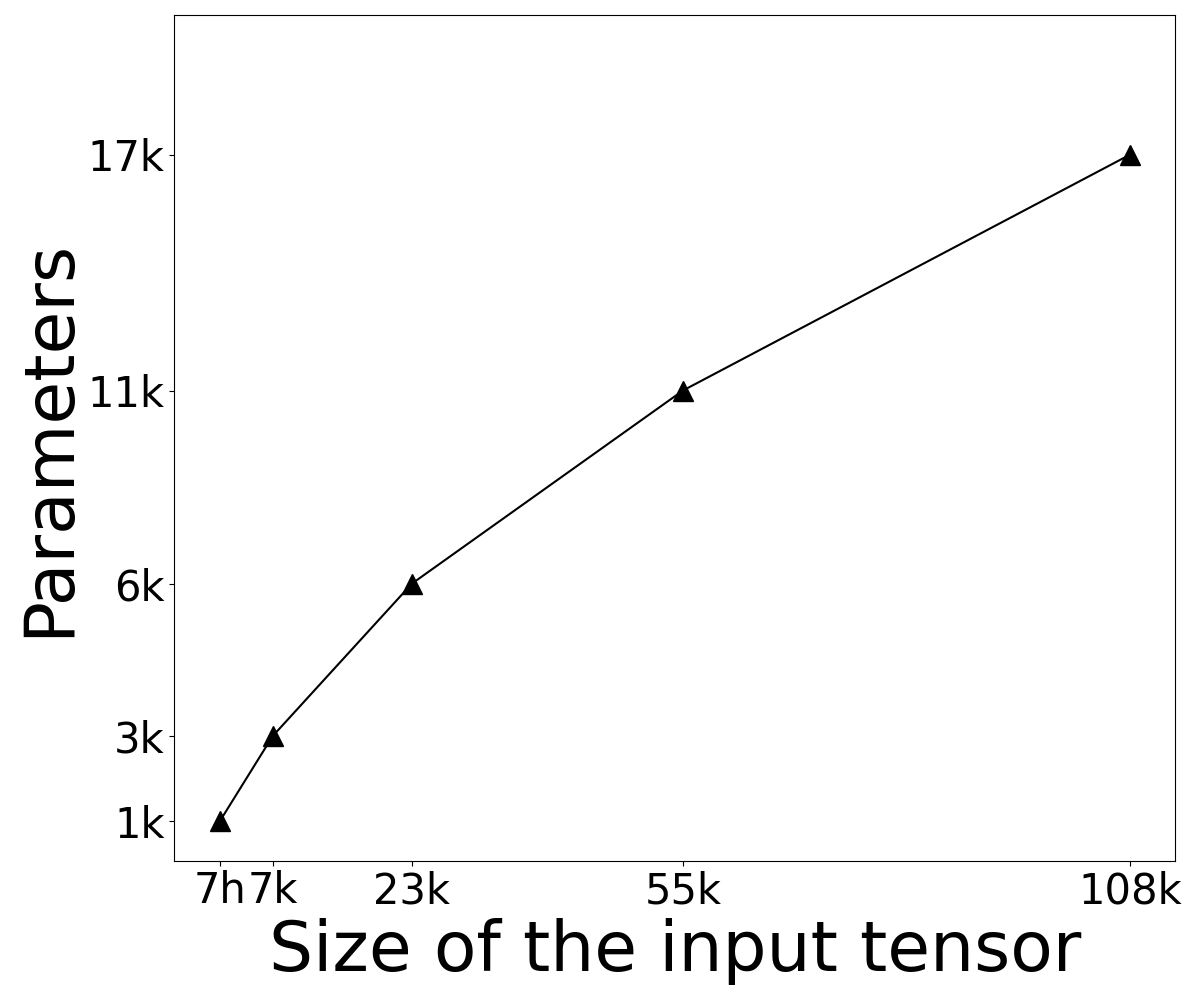}}\,
\caption{Scalability experiments.}\label{fig:exp_scalability}
\end{figure}

\subsection{Efficiency Results}\label{sec:exp_efficiency}
In this section, we present experimental results for memory efficiency and scalability.
\subsubsection{Memory Efficiency}
As shown in Table \ref{table:rho}, the upper bounds ($\rho_{max}$) of $\rho$ for the five datasets are 2.17, 2.43, 0.64, 0.31 and 57.25.
In the experiments, we fix $\rho_{exp}=0.80$, $0.80$, $0.20$, $0.10$ and $0.90$ for the \textit{Motes}, \textit{Soil}, \textit{Revenue}, \textit{Traffic} and \textit{20CR} datasets, respectively.
Given the above values of $\rho_{exp}$, the TLSTM row in Table \ref{table:rho} shows the number of parameters in TLSTM.
The mLSTM row shows the required number of parameters for multiple separate LSTMs for each single time series.
Compared with mLSTM, TLSTM significantly reduces the number of parameters by more than 80\% and yet performs better than mLSTM 
% on all the experiments 
(Figure \ref{fig:exp_rmse} and Figure \ref{fig:exp_20cr_missing}-\ref{fig:exp_20cr_future}).

\subsubsection{Scalability}
We evaluate the scalability of \nettt on the \textit{20CR} dataset in terms of the training time and the number of parameters.
We fix the $\rho=0.9$, and change the size of the input tensor by shrinking the dimension of all the modes by the specified ratios: [0.2, 0.4, 0.6, 0.8, 1.0]. 
Given the ratios, the input sizes (the number of nodes) are therefore 684, 6,912, 23,328, 55,296 and 108,000 respectively.
The averaged training time of one epoch for TLSTM against the size of the input tensor is presented in left part of Figure \ref{fig:exp_scalability},
and the number of parameters of TLSTM against the size of the input tensor is presented in right part of Figure \ref{fig:exp_scalability}.
Note that \textit{h} and \textit{k} on the x-axis represent \textit{hundreds} and \textit{thousands} respectively.
\textit{s} and \textit{k} on the y-axis represent \textit{seconds} and \textit{thousands} respectively.
The figures show that the training time and the number of parameters grow almost linearly with the size of input tensor.

\section{Related Works}\label{sec:related_work}
In this section, we review the related work in terms of (1) co-evolving time series, (2) graph convolutional networks (GCN), and (3) networked time series.

\subsection{Co-evolving Time Series}
Co-evolving time series is ubiquitous and appears in a variety of applications, such as enviornmental monitoring, financial analysis and smart transportation.
Li et al. \cite{li2009dynammo} proposed a linear dynamic system based on Kalman filter and Bayesian networks to model co-evolving time series.
Rogers et al. \cite{rogers2013multilinear} extended \cite{li2009dynammo} and further proposed a Multi-Linear Dynamic System (MLDS), which provides the base of the proposed TRNN.
Yu et al \cite{yu2016temporal} proposed a Temporal Regularized Matrix Factorization (TRMF) for modeling co-evolving time series.
Zhou et al. \cite{7837896} proposed a bi-level model to detect the rare patterns of time series.
Recently, Yu et al. \cite{yu2017deep} used LSTM \cite{hochreiter1997long} for modeling traffic flows.
Liang et al. \cite{liang2018geoman} proposed a multi-level attention network for geo-sensory time series prediction. 
Srivastava et al. \cite{srivastava2018comparative} and Zhou et al. \cite{zhou2017recover} used separate RNNs for weather and air quality monitoring time series.
Yu et al. \cite{yu2017long} proposed a HOT-RNN based on tensor-trains for long-term forecasting. 
Zhou et al. \cite{zhou2020domain} proposed a multi-domality neural attention network for financial time series.
One limitation of this line of research is that it often ignores the relation network between different time series. 

\subsection{Graph Convolutional Networks}
Plenty of real-world data could naturally be represented by a network or graph, such as social networks and sensor networks.
Bruna et al. \cite{bruna2013spectral} defined spectral graph convolution operation in the Fourier domain by analogizing it to one-dimensional convolution.
Henaff et al. \cite{henaff2015deep} used a linear interpolation, and Defferrard et al. \cite{defferrard2016convolutional} adopted Chebyshev polynomials to approximate the spectral graph convolution.
Kipf et al. \cite{kipf2016semi} simplified the Chebyshev approximation and proposed a GCN.
These methods were typically designed for flat graphs.
There are also graph convolutional network methods considering multiple types of relationships.
Monti et al. \cite{monti2017geometric} proposed a multi-graph CNN for matrix completion, which does not apply to tensor graphs. 
Wang et al. \cite{wang2019heterogeneous} proposed HAN which adopted attention mechanism to extract node embedding from different layers of a multiplex network \cite{de2013mathematical, jing2021hdmi, yan2021dynamic}, which is a flat graph with multiple types of relations, but not the \textit{tensor graph} in our paper.
Liu et al. \cite{liu2020tensor} proposed a TensorGCN for text classification. It is worth pointing out that the term \textit{tensor} in \cite{liu2020tensor} was used in a different context, i.e., it actually refers to a multiplex graph. 
For a comprehensive review of the graph neural networks, please refer to \cite{zhang2019graph, zhou2018graph, wu2020comprehensive}.

\subsection{Networked Time Series}
Relation networks have been encoded into traditional machine learning methods such as dynamic linear \cite{li2009dynammo} and multi-linear \cite{rogers2013multilinear} systems for co-evolving time series \cite{cai2015fast, cai2015facets, hairi2020netdyna}.
Recently, Li et al. \cite{li2017diffusion} incorporated spatial dependency of co-evolving traffic flows by the diffusion convolution. 
Yu et al. \cite{yu2017spatio} used GCN to incorporate spatial relations and CNN for capturing temporal dynamics.
Yan et al. \cite{yan2018spatial}, introduced a spatial-temporal GCN for skeleton recognition.
Li et al. \cite{li2019predicting} leveraged RGCN \cite{schlichtkrull2018modeling} to model spatial dependency and LSTM \cite{hochreiter1997long} for temporal dynamics.
These methods only focus on the relation graphs of a single mode, and ignore relations on other modes e.g. the correlation between the speed and occupancy of the traffic. In addition, these methods rely on the same function for capturing temporal dynamics of all time series.

It is worth pointing out that the proposed \nettt\ unifies and supersedes both co-evolving time series and networked time series as a more general data model. For example, if the adjacency matrix $\mathbf{A}_m (m=1,...,M)$ for each mode is set as an identity matrix, the proposed \nettt\ degenerates to co-evolving (tensor) time series (e.g., \cite{li2009dynammo}); networked time series in~\cite{cai2015fast} can be viewed as a special case of \nettt\ whose tensor $\mathcal X$ only has a single mode.

\section{Conclusion}\label{sec:conclusion}
In this paper, we introduce a novel \nettt\ for jointly modeling of tensor time series with its relation networks.
In order to effectively model the tensor with its relation networks at each time step, we generalize the graph convolution from flat 
%\hh{flat? plain graphs usually means graphs without attributes. pls be consistent with the def in sec2}
graphs to tensor graphs and propose a novel TGCN which not only captures the synergy among graphs but also has a succinct form.
To balance the commonality and specificity of the co-evolving time series, we propose a novel TRNN, which helps reduce noise in the data and the number of parameters in the model.
Experiments on a variety of real-world datasets demonstrate the efficacy and the applicability of \nettt.

\begin{acks}
This work is supported by National Science Foundation under grant No. 1947135, %hh-career-new
by Agriculture and Food Research Initiative (AFRI) grant no. 2020-67021-32799/project accession no.1024178 from the USDA National Institute of Food and Agriculture, 
% NSF-NIFA AIFarm@UIUC
and IBM-ILLINOIS Center for Cognitive Computing Systems Research (C3SR) - a research collaboration as part of the IBM AI Horizons Network. 
The content of the information in this document does not necessarily reflect the position or the policy of the Government, and no official endorsement should be inferred.  The U.S. Government is authorized to reproduce and distribute reprints for Government purposes notwithstanding any copyright notation here on.

\end{acks}


\bibliographystyle{ACM-Reference-Format}
\bibliography{acmart.bib}


\begin{thebibliography}{45}

%%% the Web address.

\ifx \showCODEN    \undefined \def \showCODEN     #1{\unskip}     \fi
\ifx \showDOI      \undefined \def \showDOI       #1{#1}\fi
\ifx \showISBNx    \undefined \def \showISBNx     #1{\unskip}     \fi
\ifx \showISBNxiii \undefined \def \showISBNxiii  #1{\unskip}     \fi
\ifx \showISSN     \undefined \def \showISSN      #1{\unskip}     \fi
\ifx \showLCCN     \undefined \def \showLCCN      #1{\unskip}     \fi
\ifx \shownote     \undefined \def \shownote      #1{#1}          \fi
\ifx \showarticletitle \undefined \def \showarticletitle #1{#1}   \fi
\ifx \showURL      \undefined \def \showURL       {\relax}        \fi
% invisible to TeX
\providecommand\bibfield[2]{#2}
\providecommand\bibinfo[2]{#2}
\providecommand\natexlab[1]{#1}
\providecommand\showeprint[2][]{arXiv:#2}

\bibitem[\protect\citeauthoryear{Akoglu, Tong, and Koutra}{Akoglu
  et~al\mbox{.}}{2015}]%
        {akoglu2015graph}
\bibfield{author}{\bibinfo{person}{Leman Akoglu}, \bibinfo{person}{Hanghang
  Tong}, {and} \bibinfo{person}{Danai Koutra}.}
  \bibinfo{year}{2015}\natexlab{}.
\newblock \showarticletitle{Graph based anomaly detection and description: a
  survey}.
\newblock \bibinfo{journal}{\emph{Data mining and knowledge discovery}}
  \bibinfo{volume}{29}, \bibinfo{number}{3} (\bibinfo{year}{2015}),
  \bibinfo{pages}{626--688}.
\newblock


\bibitem[\protect\citeauthoryear{Banzon, Smith, Chin, Liu, and Hankins}{Banzon
  et~al\mbox{.}}{2016}]%
        {banzon2016long}
\bibfield{author}{\bibinfo{person}{Viva Banzon}, \bibinfo{person}{Thomas~M
  Smith}, \bibinfo{person}{Toshio~Mike Chin}, \bibinfo{person}{Chunying Liu},
  {and} \bibinfo{person}{William Hankins}.} \bibinfo{year}{2016}\natexlab{}.
\newblock \showarticletitle{A long-term record of blended satellite and in situ
  sea-surface temperature for climate monitoring, modeling and environmental
  studies}.
\newblock \bibinfo{journal}{\emph{Earth System Science Data}}
  \bibinfo{volume}{8}, \bibinfo{number}{1} (\bibinfo{year}{2016}),
  \bibinfo{pages}{165--176}.
\newblock


\bibitem[\protect\citeauthoryear{Barnett}{Barnett}{1941}]%
        {barnett1941brief}
\bibfield{author}{\bibinfo{person}{Martin~K Barnett}.}
  \bibinfo{year}{1941}\natexlab{}.
\newblock \showarticletitle{A brief history of thermometry}.
\newblock \bibinfo{journal}{\emph{Journal of Chemical Education}}
  \bibinfo{volume}{18}, \bibinfo{number}{8} (\bibinfo{year}{1941}),
  \bibinfo{pages}{358}.
\newblock


\bibitem[\protect\citeauthoryear{Bodik, Hong, Guestrin, Madden, Paskin, and
  Thibau}{Bodik et~al\mbox{.}}{2004}]%
        {motes_dataset}
\bibfield{author}{\bibinfo{person}{Peter Bodik}, \bibinfo{person}{Wei Hong},
  \bibinfo{person}{Carlos Guestrin}, \bibinfo{person}{Sam Madden},
  \bibinfo{person}{Mark Paskin}, {and} \bibinfo{person}{Romain Thibau}.}
  \bibinfo{year}{2004}\natexlab{}.
\newblock \bibinfo{booktitle}{\emph{Motes Dataset}}.
\newblock
\urldef\tempurl%
\url{http://db.csail.mit.edu/labdata/labdata.html}
\showURL{%
\tempurl}


\bibitem[\protect\citeauthoryear{Bruna, Zaremba, Szlam, and LeCun}{Bruna
  et~al\mbox{.}}{2013}]%
        {bruna2013spectral}
\bibfield{author}{\bibinfo{person}{Joan Bruna}, \bibinfo{person}{Wojciech
  Zaremba}, \bibinfo{person}{Arthur Szlam}, {and} \bibinfo{person}{Yann
  LeCun}.} \bibinfo{year}{2013}\natexlab{}.
\newblock \showarticletitle{Spectral networks and locally connected networks on
  graphs}.
\newblock \bibinfo{journal}{\emph{arXiv preprint arXiv:1312.6203}}
  (\bibinfo{year}{2013}).
\newblock


\bibitem[\protect\citeauthoryear{Cai, Tong, Fan, and Ji}{Cai
  et~al\mbox{.}}{2015a}]%
        {cai2015fast}
\bibfield{author}{\bibinfo{person}{Yongjie Cai}, \bibinfo{person}{Hanghang
  Tong}, \bibinfo{person}{Wei Fan}, {and} \bibinfo{person}{Ping Ji}.}
  \bibinfo{year}{2015}\natexlab{a}.
\newblock \showarticletitle{Fast mining of a network of coevolving time
  series}. In \bibinfo{booktitle}{\emph{Proceedings of the 2015 SIAM
  International Conference on Data Mining}}. SIAM, \bibinfo{pages}{298--306}.
\newblock


\bibitem[\protect\citeauthoryear{Cai, Tong, Fan, Ji, and He}{Cai
  et~al\mbox{.}}{2015b}]%
        {cai2015facets}
\bibfield{author}{\bibinfo{person}{Yongjie Cai}, \bibinfo{person}{Hanghang
  Tong}, \bibinfo{person}{Wei Fan}, \bibinfo{person}{Ping Ji}, {and}
  \bibinfo{person}{Qing He}.} \bibinfo{year}{2015}\natexlab{b}.
\newblock \showarticletitle{Facets: Fast comprehensive mining of coevolving
  high-order time series}. In \bibinfo{booktitle}{\emph{KDD}}.
\newblock


\bibitem[\protect\citeauthoryear{Chakrabarti and Faloutsos}{Chakrabarti and
  Faloutsos}{2006}]%
        {chakrabarti2006graph}
\bibfield{author}{\bibinfo{person}{Deepayan Chakrabarti} {and}
  \bibinfo{person}{Christos Faloutsos}.} \bibinfo{year}{2006}\natexlab{}.
\newblock \showarticletitle{Graph mining: Laws, generators, and algorithms}.
\newblock \bibinfo{journal}{\emph{ACM computing surveys (CSUR)}}
  \bibinfo{volume}{38}, \bibinfo{number}{1} (\bibinfo{year}{2006}),
  \bibinfo{pages}{2--es}.
\newblock


\bibitem[\protect\citeauthoryear{Compo, Whitaker, Sardeshmukh, Matsui, Allan,
  Yin, Gleason, Vose, Rutledge, Bessemoulin, et~al\mbox{.}}{Compo
  et~al\mbox{.}}{2011}]%
        {compo2011twentieth}
\bibfield{author}{\bibinfo{person}{Gilbert~P Compo}, \bibinfo{person}{Jeffrey~S
  Whitaker}, \bibinfo{person}{Prashant~D Sardeshmukh}, \bibinfo{person}{Nobuki
  Matsui}, \bibinfo{person}{Robert~J Allan}, \bibinfo{person}{Xungang Yin},
  \bibinfo{person}{Byron~E Gleason}, \bibinfo{person}{Russell~S Vose},
  \bibinfo{person}{Glenn Rutledge}, \bibinfo{person}{Pierre Bessemoulin},
  {et~al\mbox{.}}} \bibinfo{year}{2011}\natexlab{}.
\newblock \showarticletitle{The twentieth century reanalysis project}.
\newblock \bibinfo{journal}{\emph{Quarterly Journal of the Royal Meteorological
  Society}} \bibinfo{volume}{137}, \bibinfo{number}{654}
  (\bibinfo{year}{2011}), \bibinfo{pages}{1--28}.
\newblock


\bibitem[\protect\citeauthoryear{De~Domenico, Sol{\'e}-Ribalta, Cozzo,
  Kivel{\"a}, Moreno, Porter, G{\'o}mez, and Arenas}{De~Domenico
  et~al\mbox{.}}{2013}]%
        {de2013mathematical}
\bibfield{author}{\bibinfo{person}{Manlio De~Domenico}, \bibinfo{person}{Albert
  Sol{\'e}-Ribalta}, \bibinfo{person}{Emanuele Cozzo}, \bibinfo{person}{Mikko
  Kivel{\"a}}, \bibinfo{person}{Yamir Moreno}, \bibinfo{person}{Mason~A
  Porter}, \bibinfo{person}{Sergio G{\'o}mez}, {and} \bibinfo{person}{Alex
  Arenas}.} \bibinfo{year}{2013}\natexlab{}.
\newblock \showarticletitle{Mathematical formulation of multilayer networks}.
\newblock \bibinfo{journal}{\emph{Physical Review X}} \bibinfo{volume}{3},
  \bibinfo{number}{4} (\bibinfo{year}{2013}), \bibinfo{pages}{041022}.
\newblock


\bibitem[\protect\citeauthoryear{Defferrard, Bresson, and
  Vandergheynst}{Defferrard et~al\mbox{.}}{2016}]%
        {defferrard2016convolutional}
\bibfield{author}{\bibinfo{person}{Micha{\"e}l Defferrard},
  \bibinfo{person}{Xavier Bresson}, {and} \bibinfo{person}{Pierre
  Vandergheynst}.} \bibinfo{year}{2016}\natexlab{}.
\newblock \showarticletitle{Convolutional neural networks on graphs with fast
  localized spectral filtering}.
\newblock \bibinfo{journal}{\emph{arXiv preprint arXiv:1606.09375}}
  (\bibinfo{year}{2016}).
\newblock


\bibitem[\protect\citeauthoryear{Gasch, Brown, Brooks, Yourek, Poggio, Cobos,
  and Campbell}{Gasch et~al\mbox{.}}{2017}]%
        {gasch2017pragmatic}
\bibfield{author}{\bibinfo{person}{Caley~K Gasch}, \bibinfo{person}{David~J
  Brown}, \bibinfo{person}{Erin~S Brooks}, \bibinfo{person}{Matt Yourek},
  \bibinfo{person}{Matteo Poggio}, \bibinfo{person}{Douglas~R Cobos}, {and}
  \bibinfo{person}{Colin~S Campbell}.} \bibinfo{year}{2017}\natexlab{}.
\newblock \showarticletitle{A pragmatic, automated approach for retroactive
  calibration of soil moisture sensors using a two-step, soil-specific
  correction}.
\newblock \bibinfo{journal}{\emph{Computers and Electronics in Agriculture}}
  \bibinfo{volume}{137} (\bibinfo{year}{2017}), \bibinfo{pages}{29--40}.
\newblock


\bibitem[\protect\citeauthoryear{Hairi, Tong, and Ying}{Hairi
  et~al\mbox{.}}{2020}]%
        {hairi2020netdyna}
\bibfield{author}{\bibinfo{person}{N Hairi}, \bibinfo{person}{Hanghang Tong},
  {and} \bibinfo{person}{Lei Ying}.} \bibinfo{year}{2020}\natexlab{}.
\newblock \showarticletitle{NetDyna: Mining Networked Coevolving Time Series
  with Missing Values}. In \bibinfo{booktitle}{\emph{IEEE International
  Conference on Big Data}}.
\newblock


\bibitem[\protect\citeauthoryear{Henaff, Bruna, and LeCun}{Henaff
  et~al\mbox{.}}{2015}]%
        {henaff2015deep}
\bibfield{author}{\bibinfo{person}{Mikael Henaff}, \bibinfo{person}{Joan
  Bruna}, {and} \bibinfo{person}{Yann LeCun}.} \bibinfo{year}{2015}\natexlab{}.
\newblock \showarticletitle{Deep convolutional networks on graph-structured
  data}.
\newblock \bibinfo{journal}{\emph{arXiv preprint arXiv:1506.05163}}
  (\bibinfo{year}{2015}).
\newblock


\bibitem[\protect\citeauthoryear{Hochreiter and Schmidhuber}{Hochreiter and
  Schmidhuber}{1997}]%
        {hochreiter1997long}
\bibfield{author}{\bibinfo{person}{Sepp Hochreiter} {and}
  \bibinfo{person}{J{\"u}rgen Schmidhuber}.} \bibinfo{year}{1997}\natexlab{}.
\newblock \showarticletitle{Long short-term memory}.
\newblock \bibinfo{journal}{\emph{Neural computation}} \bibinfo{volume}{9},
  \bibinfo{number}{8} (\bibinfo{year}{1997}), \bibinfo{pages}{1735--1780}.
\newblock


\bibitem[\protect\citeauthoryear{Jing, Park, and Tong}{Jing
  et~al\mbox{.}}{2021}]%
        {jing2021hdmi}
\bibfield{author}{\bibinfo{person}{Baoyu Jing}, \bibinfo{person}{Chanyoung
  Park}, {and} \bibinfo{person}{Hanghang Tong}.}
  \bibinfo{year}{2021}\natexlab{}.
\newblock \showarticletitle{HDMI: High-order Deep Multiplex Infomax}.
\newblock \bibinfo{journal}{\emph{arXiv preprint arXiv:2102.07810}}
  (\bibinfo{year}{2021}).
\newblock


\bibitem[\protect\citeauthoryear{Kingma and Ba}{Kingma and Ba}{2014}]%
        {kingma2014adam}
\bibfield{author}{\bibinfo{person}{Diederik~P Kingma} {and}
  \bibinfo{person}{Jimmy Ba}.} \bibinfo{year}{2014}\natexlab{}.
\newblock \showarticletitle{Adam: A method for stochastic optimization}.
\newblock \bibinfo{journal}{\emph{arXiv preprint arXiv:1412.6980}}
  (\bibinfo{year}{2014}).
\newblock


\bibitem[\protect\citeauthoryear{Kipf and Welling}{Kipf and Welling}{2016}]%
        {kipf2016semi}
\bibfield{author}{\bibinfo{person}{Thomas~N Kipf} {and} \bibinfo{person}{Max
  Welling}.} \bibinfo{year}{2016}\natexlab{}.
\newblock \showarticletitle{Semi-supervised classification with graph
  convolutional networks}.
\newblock \bibinfo{journal}{\emph{arXiv preprint arXiv:1609.02907}}
  (\bibinfo{year}{2016}).
\newblock


\bibitem[\protect\citeauthoryear{Kolda and Bader}{Kolda and Bader}{2009}]%
        {kolda2009tensor}
\bibfield{author}{\bibinfo{person}{Tamara~G Kolda} {and}
  \bibinfo{person}{Brett~W Bader}.} \bibinfo{year}{2009}\natexlab{}.
\newblock \showarticletitle{Tensor decompositions and applications}.
\newblock \bibinfo{journal}{\emph{SIAM review}} \bibinfo{volume}{51},
  \bibinfo{number}{3} (\bibinfo{year}{2009}), \bibinfo{pages}{455--500}.
\newblock


\bibitem[\protect\citeauthoryear{Lee, Ma, and Wang}{Lee et~al\mbox{.}}{2015}]%
        {lee2015search}
\bibfield{author}{\bibinfo{person}{Charles~MC Lee}, \bibinfo{person}{Paul Ma},
  {and} \bibinfo{person}{Charles~CY Wang}.} \bibinfo{year}{2015}\natexlab{}.
\newblock \showarticletitle{Search-based peer firms: Aggregating investor
  perceptions through internet co-searches}.
\newblock \bibinfo{journal}{\emph{Journal of Financial Economics}}
  \bibinfo{volume}{116}, \bibinfo{number}{2} (\bibinfo{year}{2015}),
  \bibinfo{pages}{410--431}.
\newblock


\bibitem[\protect\citeauthoryear{Li, Han, Cheng, Su, Wang, Zhang, and Pan}{Li
  et~al\mbox{.}}{2019}]%
        {li2019predicting}
\bibfield{author}{\bibinfo{person}{Jia Li}, \bibinfo{person}{Zhichao Han},
  \bibinfo{person}{Hong Cheng}, \bibinfo{person}{Jiao Su},
  \bibinfo{person}{Pengyun Wang}, \bibinfo{person}{Jianfeng Zhang}, {and}
  \bibinfo{person}{Lujia Pan}.} \bibinfo{year}{2019}\natexlab{}.
\newblock \showarticletitle{Predicting path failure in time-evolving graphs}.
  In \bibinfo{booktitle}{\emph{Proceedings of the 25th ACM SIGKDD International
  Conference on Knowledge Discovery \& Data Mining}}.
  \bibinfo{pages}{1279--1289}.
\newblock


\bibitem[\protect\citeauthoryear{Li, McCann, Pollard, and Faloutsos}{Li
  et~al\mbox{.}}{2009}]%
        {li2009dynammo}
\bibfield{author}{\bibinfo{person}{Lei Li}, \bibinfo{person}{James McCann},
  \bibinfo{person}{Nancy~S Pollard}, {and} \bibinfo{person}{Christos
  Faloutsos}.} \bibinfo{year}{2009}\natexlab{}.
\newblock \showarticletitle{Dynammo: Mining and summarization of coevolving
  sequences with missing values}. In \bibinfo{booktitle}{\emph{Proceedings of
  the 15th ACM SIGKDD international conference on Knowledge discovery and data
  mining}}. \bibinfo{pages}{507--516}.
\newblock


\bibitem[\protect\citeauthoryear{Li, Yu, Shahabi, and Liu}{Li
  et~al\mbox{.}}{2017}]%
        {li2017diffusion}
\bibfield{author}{\bibinfo{person}{Yaguang Li}, \bibinfo{person}{Rose Yu},
  \bibinfo{person}{Cyrus Shahabi}, {and} \bibinfo{person}{Yan Liu}.}
  \bibinfo{year}{2017}\natexlab{}.
\newblock \showarticletitle{Diffusion convolutional recurrent neural network:
  Data-driven traffic forecasting}.
\newblock \bibinfo{journal}{\emph{arXiv preprint arXiv:1707.01926}}
  (\bibinfo{year}{2017}).
\newblock


\bibitem[\protect\citeauthoryear{Liang, Ke, Zhang, Yi, and Zheng}{Liang
  et~al\mbox{.}}{2018}]%
        {liang2018geoman}
\bibfield{author}{\bibinfo{person}{Yuxuan Liang}, \bibinfo{person}{Songyu Ke},
  \bibinfo{person}{Junbo Zhang}, \bibinfo{person}{Xiuwen Yi}, {and}
  \bibinfo{person}{Yu Zheng}.} \bibinfo{year}{2018}\natexlab{}.
\newblock \showarticletitle{Geoman: Multi-level attention networks for
  geo-sensory time series prediction.}. In \bibinfo{booktitle}{\emph{IJCAI}}.
  \bibinfo{pages}{3428--3434}.
\newblock


\bibitem[\protect\citeauthoryear{Liu, You, Zhang, Wu, and Lv}{Liu
  et~al\mbox{.}}{2020}]%
        {liu2020tensor}
\bibfield{author}{\bibinfo{person}{Xien Liu}, \bibinfo{person}{Xinxin You},
  \bibinfo{person}{Xiao Zhang}, \bibinfo{person}{Ji Wu}, {and}
  \bibinfo{person}{Ping Lv}.} \bibinfo{year}{2020}\natexlab{}.
\newblock \showarticletitle{Tensor graph convolutional networks for text
  classification}. In \bibinfo{booktitle}{\emph{Proceedings of the AAAI
  Conference on Artificial Intelligence}}, Vol.~\bibinfo{volume}{34}.
  \bibinfo{pages}{8409--8416}.
\newblock


\bibitem[\protect\citeauthoryear{Monti, Bronstein, and Bresson}{Monti
  et~al\mbox{.}}{2017}]%
        {monti2017geometric}
\bibfield{author}{\bibinfo{person}{Federico Monti}, \bibinfo{person}{Michael~M
  Bronstein}, {and} \bibinfo{person}{Xavier Bresson}.}
  \bibinfo{year}{2017}\natexlab{}.
\newblock \showarticletitle{Geometric matrix completion with recurrent
  multi-graph neural networks}.
\newblock \bibinfo{journal}{\emph{arXiv preprint arXiv:1704.06803}}
  (\bibinfo{year}{2017}).
\newblock


\bibitem[\protect\citeauthoryear{Rogers, Li, and Russell}{Rogers
  et~al\mbox{.}}{2013}]%
        {rogers2013multilinear}
\bibfield{author}{\bibinfo{person}{Mark Rogers}, \bibinfo{person}{Lei Li},
  {and} \bibinfo{person}{Stuart~J Russell}.} \bibinfo{year}{2013}\natexlab{}.
\newblock \showarticletitle{Multilinear dynamical systems for tensor time
  series}.
\newblock \bibinfo{journal}{\emph{Advances in Neural Information Processing
  Systems}}  \bibinfo{volume}{26} (\bibinfo{year}{2013}),
  \bibinfo{pages}{2634--2642}.
\newblock


\bibitem[\protect\citeauthoryear{Schlichtkrull, Kipf, Bloem, Van Den~Berg,
  Titov, and Welling}{Schlichtkrull et~al\mbox{.}}{2018}]%
        {schlichtkrull2018modeling}
\bibfield{author}{\bibinfo{person}{Michael Schlichtkrull},
  \bibinfo{person}{Thomas~N Kipf}, \bibinfo{person}{Peter Bloem},
  \bibinfo{person}{Rianne Van Den~Berg}, \bibinfo{person}{Ivan Titov}, {and}
  \bibinfo{person}{Max Welling}.} \bibinfo{year}{2018}\natexlab{}.
\newblock \showarticletitle{Modeling relational data with graph convolutional
  networks}. In \bibinfo{booktitle}{\emph{European semantic web conference}}.
  Springer, \bibinfo{pages}{593--607}.
\newblock


\bibitem[\protect\citeauthoryear{Slivinski, Compo, Whitaker, Sardeshmukh,
  Giese, McColl, Allan, Yin, Vose, Titchner, et~al\mbox{.}}{Slivinski
  et~al\mbox{.}}{2019}]%
        {slivinski2019towards}
\bibfield{author}{\bibinfo{person}{Laura~C Slivinski},
  \bibinfo{person}{Gilbert~P Compo}, \bibinfo{person}{Jeffrey~S Whitaker},
  \bibinfo{person}{Prashant~D Sardeshmukh}, \bibinfo{person}{Benjamin~S Giese},
  \bibinfo{person}{Chesley McColl}, \bibinfo{person}{Rob Allan},
  \bibinfo{person}{Xungang Yin}, \bibinfo{person}{Russell Vose},
  \bibinfo{person}{Holly Titchner}, {et~al\mbox{.}}}
  \bibinfo{year}{2019}\natexlab{}.
\newblock \showarticletitle{Towards a more reliable historical reanalysis:
  Improvements for version 3 of the Twentieth Century Reanalysis system}.
\newblock \bibinfo{journal}{\emph{Quarterly Journal of the Royal Meteorological
  Society}} \bibinfo{volume}{145}, \bibinfo{number}{724}
  (\bibinfo{year}{2019}), \bibinfo{pages}{2876--2908}.
\newblock


\bibitem[\protect\citeauthoryear{Srivastava and Lessmann}{Srivastava and
  Lessmann}{2018}]%
        {srivastava2018comparative}
\bibfield{author}{\bibinfo{person}{Shikhar Srivastava} {and}
  \bibinfo{person}{Stefan Lessmann}.} \bibinfo{year}{2018}\natexlab{}.
\newblock \showarticletitle{A comparative study of LSTM neural networks in
  forecasting day-ahead global horizontal irradiance with satellite data}.
\newblock \bibinfo{journal}{\emph{Solar Energy}}  \bibinfo{volume}{162}
  (\bibinfo{year}{2018}), \bibinfo{pages}{232--247}.
\newblock


\bibitem[\protect\citeauthoryear{Sutskever}{Sutskever}{2013}]%
        {sutskever2013training}
\bibfield{author}{\bibinfo{person}{Ilya Sutskever}.}
  \bibinfo{year}{2013}\natexlab{}.
\newblock \bibinfo{booktitle}{\emph{Training recurrent neural networks}}.
\newblock \bibinfo{publisher}{University of Toronto Toronto, Canada}.
\newblock


\bibitem[\protect\citeauthoryear{Tsay}{Tsay}{2014}]%
        {tsay2014financial}
\bibfield{author}{\bibinfo{person}{Ruey~S Tsay}.}
  \bibinfo{year}{2014}\natexlab{}.
\newblock \showarticletitle{Financial time series}.
\newblock \bibinfo{journal}{\emph{Wiley StatsRef: Statistics Reference Online}}
  (\bibinfo{year}{2014}), \bibinfo{pages}{1--23}.
\newblock


\bibitem[\protect\citeauthoryear{Wang, Ji, Shi, Wang, Ye, Cui, and Yu}{Wang
  et~al\mbox{.}}{2019}]%
        {wang2019heterogeneous}
\bibfield{author}{\bibinfo{person}{Xiao Wang}, \bibinfo{person}{Houye Ji},
  \bibinfo{person}{Chuan Shi}, \bibinfo{person}{Bai Wang},
  \bibinfo{person}{Yanfang Ye}, \bibinfo{person}{Peng Cui}, {and}
  \bibinfo{person}{Philip~S Yu}.} \bibinfo{year}{2019}\natexlab{}.
\newblock \showarticletitle{Heterogeneous graph attention network}. In
  \bibinfo{booktitle}{\emph{The World Wide Web Conference}}.
  \bibinfo{pages}{2022--2032}.
\newblock


\bibitem[\protect\citeauthoryear{Wu, Pan, Chen, Long, Zhang, and Philip}{Wu
  et~al\mbox{.}}{2020}]%
        {wu2020comprehensive}
\bibfield{author}{\bibinfo{person}{Zonghan Wu}, \bibinfo{person}{Shirui Pan},
  \bibinfo{person}{Fengwen Chen}, \bibinfo{person}{Guodong Long},
  \bibinfo{person}{Chengqi Zhang}, {and} \bibinfo{person}{S~Yu Philip}.}
  \bibinfo{year}{2020}\natexlab{}.
\newblock \showarticletitle{A comprehensive survey on graph neural networks}.
\newblock \bibinfo{journal}{\emph{IEEE transactions on neural networks and
  learning systems}} (\bibinfo{year}{2020}).
\newblock


\bibitem[\protect\citeauthoryear{Yan, Xiong, and Lin}{Yan
  et~al\mbox{.}}{2018}]%
        {yan2018spatial}
\bibfield{author}{\bibinfo{person}{Sijie Yan}, \bibinfo{person}{Yuanjun Xiong},
  {and} \bibinfo{person}{Dahua Lin}.} \bibinfo{year}{2018}\natexlab{}.
\newblock \showarticletitle{Spatial temporal graph convolutional networks for
  skeleton-based action recognition}. In \bibinfo{booktitle}{\emph{Proceedings
  of the AAAI conference on artificial intelligence}},
  Vol.~\bibinfo{volume}{32}.
\newblock


\bibitem[\protect\citeauthoryear{Yan, Liu, Ban, Jing, and Tong}{Yan
  et~al\mbox{.}}{2021}]%
        {yan2021dynamic}
\bibfield{author}{\bibinfo{person}{Yuchen Yan}, \bibinfo{person}{Lihui Liu},
  \bibinfo{person}{Yikun Ban}, \bibinfo{person}{Baoyu Jing}, {and}
  \bibinfo{person}{Hanghang Tong}.} \bibinfo{year}{2021}\natexlab{}.
\newblock \showarticletitle{Dynamic Knowledge Alignment}. In
  \bibinfo{booktitle}{\emph{AAAI}}.
\newblock


\bibitem[\protect\citeauthoryear{Yu, Yin, and Zhu}{Yu et~al\mbox{.}}{2017b}]%
        {yu2017spatio}
\bibfield{author}{\bibinfo{person}{Bing Yu}, \bibinfo{person}{Haoteng Yin},
  {and} \bibinfo{person}{Zhanxing Zhu}.} \bibinfo{year}{2017}\natexlab{b}.
\newblock \showarticletitle{Spatio-temporal graph convolutional networks: A
  deep learning framework for traffic forecasting}.
\newblock \bibinfo{journal}{\emph{arXiv preprint arXiv:1709.04875}}
  (\bibinfo{year}{2017}).
\newblock


\bibitem[\protect\citeauthoryear{Yu, Rao, and Dhillon}{Yu
  et~al\mbox{.}}{2016}]%
        {yu2016temporal}
\bibfield{author}{\bibinfo{person}{Hsiang-Fu Yu}, \bibinfo{person}{Nikhil Rao},
  {and} \bibinfo{person}{Inderjit~S Dhillon}.} \bibinfo{year}{2016}\natexlab{}.
\newblock \showarticletitle{Temporal Regularized Matrix Factorization for
  High-dimensional Time Series Prediction.}. In
  \bibinfo{booktitle}{\emph{NIPS}}. \bibinfo{pages}{847--855}.
\newblock


\bibitem[\protect\citeauthoryear{Yu, Li, Shahabi, Demiryurek, and Liu}{Yu
  et~al\mbox{.}}{2017a}]%
        {yu2017deep}
\bibfield{author}{\bibinfo{person}{Rose Yu}, \bibinfo{person}{Yaguang Li},
  \bibinfo{person}{Cyrus Shahabi}, \bibinfo{person}{Ugur Demiryurek}, {and}
  \bibinfo{person}{Yan Liu}.} \bibinfo{year}{2017}\natexlab{a}.
\newblock \showarticletitle{Deep learning: A generic approach for extreme
  condition traffic forecasting}. In \bibinfo{booktitle}{\emph{Proceedings of
  the 2017 SIAM international Conference on Data Mining}}. SIAM,
  \bibinfo{pages}{777--785}.
\newblock


\bibitem[\protect\citeauthoryear{Yu, Zheng, Anandkumar, and Yue}{Yu
  et~al\mbox{.}}{2017c}]%
        {yu2017long}
\bibfield{author}{\bibinfo{person}{Rose Yu}, \bibinfo{person}{Stephan Zheng},
  \bibinfo{person}{Anima Anandkumar}, {and} \bibinfo{person}{Yisong Yue}.}
  \bibinfo{year}{2017}\natexlab{c}.
\newblock \showarticletitle{Long-term forecasting using tensor-train rnns}.
\newblock \bibinfo{journal}{\emph{Arxiv}} (\bibinfo{year}{2017}).
\newblock


\bibitem[\protect\citeauthoryear{Zhang, Tong, Xu, and Maciejewski}{Zhang
  et~al\mbox{.}}{2019}]%
        {zhang2019graph}
\bibfield{author}{\bibinfo{person}{Si Zhang}, \bibinfo{person}{Hanghang Tong},
  \bibinfo{person}{Jiejun Xu}, {and} \bibinfo{person}{Ross Maciejewski}.}
  \bibinfo{year}{2019}\natexlab{}.
\newblock \showarticletitle{Graph convolutional networks: a comprehensive
  review}.
\newblock \bibinfo{journal}{\emph{Computational Social Networks}}
  (\bibinfo{year}{2019}).
\newblock


\bibitem[\protect\citeauthoryear{{Zhou}, {He}, {Cao}, and {Seo}}{{Zhou}
  et~al\mbox{.}}{2016}]%
        {7837896}
\bibfield{author}{\bibinfo{person}{D. {Zhou}}, \bibinfo{person}{J. {He}},
  \bibinfo{person}{Y. {Cao}}, {and} \bibinfo{person}{J. {Seo}}.}
  \bibinfo{year}{2016}\natexlab{}.
\newblock \showarticletitle{Bi-Level Rare Temporal Pattern Detection}. In
  \bibinfo{booktitle}{\emph{2016 IEEE 16th International Conference on Data
  Mining (ICDM)}}. \bibinfo{pages}{719--728}.
\newblock
\urldef\tempurl%
\url{https://doi.org/10.1109/ICDM.2016.0083}
\showDOI{\tempurl}


\bibitem[\protect\citeauthoryear{Zhou, Zheng, Zhu, Li, and He}{Zhou
  et~al\mbox{.}}{2020}]%
        {zhou2020domain}
\bibfield{author}{\bibinfo{person}{Dawei Zhou}, \bibinfo{person}{Lecheng
  Zheng}, \bibinfo{person}{Yada Zhu}, \bibinfo{person}{Jianbo Li}, {and}
  \bibinfo{person}{Jingrui He}.} \bibinfo{year}{2020}\natexlab{}.
\newblock \showarticletitle{Domain adaptive multi-modality neural attention
  network for financial forecasting}. In \bibinfo{booktitle}{\emph{Proceedings
  of The Web Conference 2020}}. \bibinfo{pages}{2230--2240}.
\newblock


\bibitem[\protect\citeauthoryear{Zhou, Cui, Zhang, Yang, Liu, Wang, Li, and
  Sun}{Zhou et~al\mbox{.}}{2018}]%
        {zhou2018graph}
\bibfield{author}{\bibinfo{person}{Jie Zhou}, \bibinfo{person}{Ganqu Cui},
  \bibinfo{person}{Zhengyan Zhang}, \bibinfo{person}{Cheng Yang},
  \bibinfo{person}{Zhiyuan Liu}, \bibinfo{person}{Lifeng Wang},
  \bibinfo{person}{Changcheng Li}, {and} \bibinfo{person}{Maosong Sun}.}
  \bibinfo{year}{2018}\natexlab{}.
\newblock \showarticletitle{Graph neural networks: A review of methods and
  applications}.
\newblock \bibinfo{journal}{\emph{arXiv preprint arXiv:1812.08434}}
  (\bibinfo{year}{2018}).
\newblock


\bibitem[\protect\citeauthoryear{Zhou and Huang}{Zhou and Huang}{2017}]%
        {zhou2017recover}
\bibfield{author}{\bibinfo{person}{Jingguang Zhou} {and} \bibinfo{person}{Zili
  Huang}.} \bibinfo{year}{2017}\natexlab{}.
\newblock \showarticletitle{Recover missing sensor data with iterative imputing
  network}.
\newblock \bibinfo{journal}{\emph{arXiv preprint arXiv:1711.07878}}
  (\bibinfo{year}{2017}).
\newblock


\end{thebibliography}

% \appendix

\end{document}